\documentclass[preprint,3p,12pt]{elsarticle}
\usepackage[utf8]{inputenc}
\usepackage{times}
\usepackage{tocloft}
\usepackage{graphicx}
\graphicspath{ {images/} }
\usepackage{algorithm}
\usepackage{xcolor}
\usepackage{algpseudocode}
\usepackage{pifont}
\usepackage{pgfplots}
\usepackage{amsmath,amsthm,amssymb}
\usepackage{tikz}
\usetikzlibrary{shapes,arrows,automata,matrix,positioning}
\usepackage[makeroom]{cancel}
\usepackage{caption}
\usepackage{subcaption}
\usepackage{adjustbox,lipsum}
\usepackage{makecell}
\usepackage{float}
\usepackage{amsthm}
\usepackage{thmtools,thm-restate}
\usepackage{hyperref}
\hypersetup{
colorlinks,
citecolor=black,
filecolor=black,
linkcolor=black,
urlcolor=black
}
\newcommand{\quotes}[1]{``#1''}
\renewcommand*{\appendixname}{}

\newtheorem{theorem}{Theorem}
\newtheorem{lemma}{Lemma}

\newtheorem{corollary}{Corollary}
\newtheorem{definition}{Definition}
\newtheorem{proposition}{Proposition}

\newtheorem{example}{Example}

\newcommand{\VCdim}{\mathrm{VCD}}
\newcommand{\TD}{\mathrm{TD}}
\newcommand{\RTD}{\mathrm{RTD}}

\newcommand{\CPT}{\mathrm{CPT}}

\newcommand{\LB}{\mathrm{LB}}
\newcommand{\VQ}{\mathrm{VQ}}

\title{The Complexity of Learning of Acyclic Conditional Preference Networks}

\author[sa]{Eisa Alanazi\corref{cor1}}
\ead{eaanazi@uqu.edu.sa}
\author[reg]{Malek Mouhoub}
\ead{mouhoub@cs.uregina.ca}
\author[reg]{Sandra Zilles}
\ead{zilles@cs.uregina.ca}

\cortext[cor1]{Corresponding author}

\address[sa]{Department of Computer Science, Umm Al-Qura University, Makkah, Saudi Arabia}
\address[reg]{Department of Computer Science,
University of Regina, Regina, SK, Canada S4S 0A2}

\begin{document}
	\begin{abstract}
		Learning of user preferences, as represented by, for example, Conditional Preference Networks (CP-nets), has become a core issue in AI research. Recent studies investigate learning of CP-nets from randomly chosen examples or from membership and equivalence queries. To assess the optimality of learning algorithms as well as to better understand the combinatorial structure of classes of CP-nets, it is helpful to calculate certain learning-theoretic information complexity parameters. This article focuses on the frequently studied case of learning from so-called swap examples, which express preferences among objects that differ in only one attribute. It presents bounds on or exact values of some well-studied information complexity parameters, namely the VC dimension, the teaching dimension, and the recursive teaching dimension, for classes of acyclic CP-nets. We further provide algorithms that learn tree-structured and general acyclic CP-nets from membership queries. Using our results on complexity parameters, we prove that our algorithms, as well as another query learning algorithm for acyclic CP-nets presented in the literature, are near-optimal.
	\end{abstract}

\maketitle
\section{Introduction}\label{sec:introduction}
Preference learning has become a major branch of AI, with applications in decision support systems in general and in e-commerce in particular~\cite{Fuernkranz2011}. For instance, recommender systems based on collaborative filtering make predictions on a single user's preferences by exploiting information about large groups of users. Another example are intelligent tutoring systems, which learn a student's preferences in order to deliver personalized content to the student.

To design and analyze algorithms for learning preferences of a single user, one needs an abstract model for representing user preferences. Some approaches model preferences quantitatively, thus allowing for expressing the relative magnitude of preferences between object pairs, while others are purely qualitative, expressing partial orders or rankings over objects~\cite{keeney1993decisions,DOMSHLAK20111037}.


Most application domains are of multi-attribute form, meaning that the set of possible alternatives (i.e., objects, or outcomes) is defined on a set of attributes and every alternative corresponds to an assignment of values to the attributes. Such combinatorial domains require compact models to capture the preference information in a structured manner. In recent years, various models have been suggested, such as Generalized Additive Decomposable (GAI-net) utility functions~\cite{AllenSG17}, Lexicographic Preference Trees~\cite{Booth:2010:LCL:1860967.1861021}, and Conditional Preference Networks (CP-nets)~\cite{1-DBLP:journals/jair/BoutilierBDHP04}. 

CP-nets provide a compact qualitative preference representation for multi-attribute domains where the preference of one attribute may depend on the values of other attributes. The study of their learnability \cite{Koriche2010685,DBLP:conf/ijcai/DimopoulosMA09,Lang:2009:CLS:1661445.1661580,10.1109/TKDE.2012.231,www,AllenSG17} is an ongoing topic in research on preference elicitation. 

For example, Koriche and Zanuttini \cite{Koriche2010685} investigated query learning of $k$-bounded acyclic CP-nets (i.e., with a bound $k$ on the number of attributes on which the preferences for any attribute may depend) and provided efficient algorithms using both membership and equivalence queries, cf.\ \cite{Angluin:1988:QCL:639961.639995}. CP-nets have also been studied in models of passive learning from examples, both for batch learning \cite{DBLP:conf/ijcai/DimopoulosMA09,Lang:2009:CLS:1661445.1661580,10.1109/TKDE.2012.231,AllenSG17} and for online learning \cite{www,LaberniaZMYA17}. 

The focus of our work is on the design of methods for learning CP-nets though interaction with the user, and on an analysis of the complexity of such learning problems. In particular, we study the model of learning from membership queries, in which users are asked for information on their preference between two objects. To the best of our knowledge, algorithms for learning CP-nets from membership queries only have not been studied in the literature yet. We argue below (in Section~\ref{sec:perfect}) why algorithms using membership queries alone are of importance to research on preference learning. In a nutshell, membership queries seem to be more easily deployable in preference learning than equivalence queries and, from a theoretical point of view, it is interesting to see how powerful they are in comparison to equivalence queries. The latter alone are known to be insufficient for efficient learning of acyclic CP-nets~\cite{Koriche2010685}. Therefore, one major part of this article deals with learning CP-nets from membership queries only. Note that the importance of learning from queries in general, and from membership queries in particular, is widely recognized in the context of preference elicitation~\cite{BlumJSZ04,BoutilierRV09,LahaieP04,ZinkevichBS03}.


In every formal model of learning, a fundamental question in assessing learning algorithms is how many queries or examples would be needed by the best possible learning algorithm in the given model. For several models, lower bounds can be derived from the Vapnik Chervonenkis dimension ($\VCdim$) \cite{VC71}. This central parameter is one of several that, in addition to yielding bounds on the performance of learning algorithms, provide deep insights into the combinatorial structure of the studied concept class. Such insights can in turn help to design new learning algorithms. 

A classical result states that $\log_2(4/3)d$ is a lower bound on the number of equivalence and membership queries required for learning a concept class whose $\VCdim$ equals $d$~\cite{AuerL99}. Likewise, it is known that a parameter called the teaching dimension~\cite{GK95} is a lower bound on the number of membership queries required for learning~\cite{queriesRevisited}. Therefore, another major part of this article deals with calculating exact values or non-trivial bounds on a number of learning-theoretic complexity parameters, such as the $\VCdim$ and the teaching dimension. All these complexity parameters are calculated under the assumption that information about user preferences is provided for so-called swaps, exclusively. A swap is a pair of objects that differ in the value of only a single attribute. Learning CP-nets over swap examples is an often studied scenario~\cite{Koriche2010685,DBLP:conf/ijcai/DimopoulosMA09,labernia2016query,LaberniaYMA18,LaberniaZMYA17}, which we adopt here for various reasons detailed in Section~\ref{sec:representation}. 

Our main contributions are the following:

(a) We provide the first study that exactly calculates the $\VCdim$ for the class of unbounded acyclic CP-nets, and give a lower bound for any bound $k$. So far, the only existing studies present a lower bound \cite{Koriche2010685}, which we prove incorrect for large values of $k$, and asymptotic complexities \cite{qqq}. The latter show that $\VCdim\in\Theta(2^n)$ for $k=n-1$ and $\VCdim\in\tilde{\Theta}(n2^k)$ when $k\in o(n)$, in agreement with our result that $\VCdim=2^n-1$ for $k=n-1$, and is at least $\VCdim\ge 2^k-1+(n-k)2^k$ for general values of $k$. It should be noted that both previous studies assume that CP-nets can be incomplete, i.e., for some attributes, preference relations may not be fully specified. In our study, we first investigate the (not uncommon) assumption that CP-nets are complete, but then we extend each of our results to the more general case that includes incomplete CP-nets, as well. Further, some of our results are more general than existing ones in that they cover also the case of CP-nets with multi-valued attributes (as opposed to binary attributes.)

(b) We further provide exact values (or, in some cases, non-trivial bounds) for two other important information complexity parameters, namely the teaching dimension \cite{GK95}, and the recursive teaching dimension \cite{ZLHZ11}.

(c) Appendix\ref{sec:struct} gives an in-depth study of structural properties of the class of all complete acyclic CP-nets that are of importance to learning-theoretic studies.

(d) We present a new algorithm that learns tree-structured CP-nets (i.e., the case $k=1$) from membership queries only and use our results on the teaching dimension to show that our algorithm is close to optimal. We further extend our algorithm to deal with the general case of $k$-bounded acyclic CP-nets with bound $k\ge 1$.

 (e) In most real-world scenarios, one would expect some degree of noise in the responses to membership queries, or that sometimes no response at all is obtained. To address this issue, we demonstrate how, under certain assumptions on the noise and the missing responses, our algorithm for learning tree CP-nets can be adapted to handle incomplete or incorrect answers to membership queries.

(f) We re-assess the degree of optimality of Koriche and Zanuttini's algorithm for learning bounded acyclic CP-nets, using our result on the $\VCdim$. 

This article extends a previous conference paper~\cite{AMZ16}. Theorem 4 in this conference paper included an incorrect claim about the so-called self-directed learning complexity of classes of acyclic CP-nets; the incorrect statement has been removed in this extended version.

\section{Related Work}\label{sec:related}

This section sets the present paper into the context of the existing literature.

\subsection{Preference Elicitation with Membership Queries}\label{ssec:pe}

Preference elicitation is the interactive process of gathering information about the preferences of a user of a system, mostly through queries, and is applied, e.g., in recommender systems and combinatorial auctions. The goal of preference elicitation varies from learning a full preference model \cite{ChajewskaKO01,ErculianiDTP18} to learning enough information to make a (near-)optimal recommendation to the user \cite{BlumJSZ04,BoutilierRV09,DeryKRS16}. Learning a full preference model is of practical interest, be it for situations when the user's most preferred items are not available or for the ease of adapting user models when assuming that user preferences change over time. To reduce the user's burden, one major objective in preference elicitation is to keep the number of queries small. This has motivated learning-theoretic studies on the efficient use of queries in preference elicitation  \cite{BlumJSZ04,BoutilierRV09,LahaieP04,ZinkevichBS03}.


One popular type of query used in preference elicitation is the so-called value query, which requests the user to assign a numerical value to an item. Since such queries are in essence equivalent to membership queries~\cite{LahaieP04,BlumJSZ04}, the notion of learning from membership queries is well-studied in preference elicitation. For example, Boutilier at al.~\cite{BoutilierRV09} studied preference elicitation in the presence of user-defined features. They cast the problem of learning defined features as a concept learning problem, which they solved using membership queries. Zinkevich et al.~\cite{ZinkevichBS03} focused on learning preference functions that correspond to read-once formulas over certain kinds of gates and used a classical algorithm for learning read-once formulas~\cite{AngluinHK93} in order to elicit user preferences via value/membership queries alone. Finally, Lahaie and Parkes \cite{LahaieP04} showed that any exact learning algorithm with membership and equivalence queries can be converted to a preference elicitation algorithm with value and demand queries. 

Our work assumes that preferences are represented compactly via CP-nets and our goal is to recover the full preference relation exactly. Since the structure of the network is not known in advance, this problem is non-trivial. We need to devise algorithms that learn both the structure and preference function for every variable in the network. Our work provides efficient (and provably close to optimal) algorithms to recover preferences exactly via membership queries alone and is thus in line with some of the approaches discussed above~\cite{BoutilierRV09,ZinkevichBS03}. In order to prove that our algorithms are close to optimal, we make use of a learning-theoretic parameter called the teaching dimension \cite{GK95}, which we calculate for various classes of CP-nets in Sections \ref{sec:complete} and \ref{sec:incomplete}.

\subsection{Learning CP-Nets}


The problem of learning CP-nets has recently gained a substantial amount of attention \cite{DBLP:conf/ijcai/DimopoulosMA09,DBLP:conf/ijcai/KoricheZ09,Lang:2009:CLS:1661445.1661580,Koriche2010685,qqq,10.1109/TKDE.2012.231,www,Liu20137,DBLP:conf/ijcai/MichaelP13,DBLP:conf/stairs/BigotMZ14,AllenSG17}. 

Both in active and in passive learning, a sub-problem to be solved by many natural learning algorithms is the so-called \emph{consistency problem}. This decision problem is defined as follows. A problem instance consists of a CP-net $N$ and a set $S$ of user preferences between objects, in the form of ``object $o$ is preferred over object $o'$'' or ``object $o$ is not preferred over object $o'$''. The question to be answered is whether or not $N$ is consistent with $S$, i.e., whether the partial order $\succ$ over objects that is induced by $N$ satisfies $o\succ o'$ if $S$ states that $o$ is preferred over $o'$, and satisfies $o\not\succ o'$ if $S$ states that $o$ is not preferred over $o'$. The consistency problem was shown to be NP-hard even if $N$ is restricted to be an acyclic $k$-bounded CP-net for some fixed $k\geq 2$ and even when, for any object pair $(o,o')$ under consideration, the outcomes $o$ and $o'$ differ in the values of at most two attributes \cite{DBLP:conf/ijcai/DimopoulosMA09}. Based on this result, Dimopoulos et al.~\cite{DBLP:conf/ijcai/DimopoulosMA09} showed that complete acyclic CP-nets with bounded indegree are not efficiently PAC-learnable, i.e., learnable in polynomial time in the PAC model. The authors, however, then showed that such CP-nets are efficiently PAC-learnable from examples that are drawn exclusively from the set of so-called transparent entailments. Specifically, this implied that complete acyclic $k$-bounded CP-nets are efficiently PAC-learnable from swap examples. Michael and Papageorgiou \cite{DBLP:conf/ijcai/MichaelP13} then provided a comprehensive experimental view on the performance of the algorithm proposed in \cite{DBLP:conf/ijcai/DimopoulosMA09}. Their work also proposed an efficient method for checking whether a given entailment is transparent or not. These studies focus on learning approximations of target CP-nets passively and from randomly chosen data. By comparison, all algorithms we propose below learn target CP-nets exactly and they actively pose queries in order to collect training data, following Angluin's model of learning from membership queries~\cite{Angluin:1988:QCL:639961.639995}. 

Lang and Mengin \cite{Lang:2009:CLS:1661445.1661580} considered the complexity of learning binary separable CP-nets, in various learning settings.\footnote{A CP-net is separable if it is $0$-bounded, i.e., the preferences over the domain of any attribute are not conditioned on the values of other attributes.} The literature also includes results on learning CP-nets from noisy examples, namely, via statistical hypothesis testing \cite{Liu20137}, using evolutionary algorithms and metaheuristics~\cite{HaqqaniL17,AllenSG17}, or by learning the induced graph directly, which takes time exponential in the number of attributes~\cite{10.1109/TKDE.2012.231}. 
These results cannot be compared to the ones presented in the present paper, as (i) they focus on approximating instead of exactly learning the target CP-net, and (ii) the noise models on which they build differ substantially from the settings we consider. In our first setting, there is no noise in the data whatsoever. The second setting we study is one in which the membership oracle may corrupt a certain number of query responses, but there is no randomness to the process. Instead, one analyzes learning under an adversarial assumption on the oracle's choice of which answers to corrupt, and then investigates whether exact learning is still possible~\cite{Angluin1997,DBLP:journals/tcs/BennetB07,DBLP:journals/jcss/BishtBK08}. To the best of our knowledge, this setting has not been studied in the context of CP-nets so far.

As for active learning, Guerin et al.~\cite{www} proposed a heuristic online algorithm that is not limited to swap comparisons. The algorithm assumes the user to provide explicit answers of the form ``object $o$ is preferred over object $o'$'', ``object $o'$ is preferred over object $o$'', or ``neither of the two objects is preferred over the other'' to any query $(o,o')$. Labernia et al.~\cite{LaberniaZMYA17} proposed another online learning algorithm based on swap observations where the latter can be noisy. It is assumed that the target CP-net represents the global preference for a group of users and the noise is due to variations of a user's preference compared to the global one. The authors formally proved that their algorithm produces a close approximation to the target CP-net and analyzed the algorithm empirically under random noise. Again, as in all the related literature discussed above, the most striking difference to our setting is that these works focus on approximating the target CP-net rather than learning it exactly.

To the best of our knowledge, the only studies of learning CP-nets in Angluin's query model, where the target concept is identified exactly, are one by Koriche and Zanuttini \cite{Koriche2010685} and one by Labernia et al.~\cite{labernia2016query}. Koriche and Zanuttini assumed perfect oracles and investigated the problem of learning complete and incomplete bounded CP-nets from membership and equivalence queries over the swap instance space. They showed that complete acyclic CP-nets are not learnable from equivalence queries alone but are \emph{attribute-efficiently} learnable from membership and equivalence queries. Attribute-efficiency means that the number of queries required is upper-bounded by a function that is polynomial in the size of the input, but only logarithmic in the number of attributes. In the case of tree CP-nets, their results hold true even when the equivalence queries may return non-swap examples. The setting considered in their work is more general than ours and exhibits the power of membership queries when it comes to learning CP-nets. Labernia et al.~\cite{labernia2016query} investigated the problem of learning an average CP-net from multiple users using equivalence queries alone. However, neither study addresses the problem of learning complete acyclic CP-nets from membership queries alone, whether corrupted or uncorrupted. 
Given that learning from membership queries plays an important role in preference elicitation (see Section~\ref{ssec:pe}), our algorithms thus address a gap in the literature. We provide a detailed comparison of our methods to those by Koriche and Zanuttini in Section~\ref{sec:perfect}. In a nutshell, our membership query algorithm improves on theirs in that it does not require equivalence queries, but has the downside of not being attribute-efficient. The latter is not an artefact of our algorithm---we argue in Section~\ref{sec:perfect} why attribute-efficient learning of CP-nets with membership queries alone is not possible.

\subsection{Complexity Parameters in Computational Learning Theory}

The Vapnik Chervonenkis Dimension, also called VC dimension~\cite{VC71}, is one of the best studied complexity parameters in the computational learning theory literature. Upper and/or lower sample complexity bounds that are linear in the VC dimension are known for various popular models of concept learning, namely for PAC learning, which is a basic model of learning from randomly chosen examples~\cite{Hanneke16}, for exact learning from equivalence and membership queries, which is a model of learning from active queries~\cite{AuerL99}, and, in some special cases even for learning from teachers~\cite{DFSZ14,SimonZ15}.

Because of these bounds, knowledge of the VC dimension value of a concept class $\mathcal{C}$ can help in assessing learning algorithms for $\mathcal{C}$. For example, if the number of queries consumed by an algorithm learning $\mathcal{C}$ exceeds a known lower bound on the query complexity by a constant factor, we know that the algorithm is within a constant factor of optimal. For this reason, the VC dimension of classes of CP-nets have been studied in the literature. Koriche and Zanuttini \cite{Koriche2010685}, for the purpose of analyzing their algorithms for learning from equivalence and membership queries, established a lower bound on the VC dimension of the class of complete and incomplete $k$-bounded binary CP-nets. Chevaleyre et al.~\cite{qqq} gave asymptotic estimates on the VC dimension of classes of CP-nets. They showed that the VC dimension of acyclic binary CP-nets is $\Theta(2^n)$ for arbitrary CP-nets and $\tilde{\Theta}(n2^k)$ for $k$-bounded CP-nets. Here $n$ is the number of attributes in a CP-net. The results by Chevaleyre et al.\ are in agreement with our results, stating that $\VCdim$ is $2^n-1$ for arbitrary acyclic CP-nets and at least $(n-k)2^k+2^k-1$ for $k$-bounded ones. 

Our work improves on both of these contributions. Firstly, we correct a mistake in the lower bound published by Koriche and Zanuttini. Secondly, compared to asymptotic studies by Chevaleyre et al., we calculate exact values or explicit lower bounds on the VC dimension. Thirdly, we calculate the VC dimension also for the case that the attributes in a CP-net have more than two values.

To the best of our knowledge, there are no other studies that address learning-theoretic complexity parameters of CP-nets. The present paper computes two more parameters, namely the teaching dimension~\cite{GK95} and the recursive teaching dimension~\cite{ZLHZ11}, both of which refer to the complexity of machine teaching. Recently, models of machine teaching have received increased attention in the machine learning community~\cite{Zhu15,ZhuSZR18}, since they try to capture the idea of helpfully selected training data as would be expected in many human-centric applications. In our study, teaching complexity parameters are of relevance for two reasons.

First, the teaching dimension is a lower bound on the number of membership queries required for learning~\cite{GK95}, und thus a tool for evaluating the efficiency of our learning algorithms relative to the theoretic optimum.

Second, due to the increased interest in machine teaching, the machine learning community is looking for bounds on the efficiency of teaching, in terms of other well-studied parameters. The recursive teaching dimension is the first teaching complexity parameter that was shown to be closely related to the VC dimension. It is known to be at most quadratic in the VC dimension~\cite{HuWLW17}, and under certain structural properties even equal to the VC dimension~\cite{DFSZ14}. However, it remains open whether or not it is upper-bounded by a function linear in the VC dimension~\cite{SimonZ15}. With the class of all unbounded acyclic CP-nets, we provide the first example of an ``interesting'' concept class for which the VC dimension and the recursive teaching dimension are equal, provably without satisfying any of the known structural properties that would imply such equality. Thus, our study of the recursive teaching dimension of classes of CP-nets may be of help to ongoing learning-theoretic studies of teaching complexity in general.

\section{Background}
\label{sec:background}

This section introduces the terminology and notation used subsequently, and motivates the formal settings studied in the rest of the paper.

\subsection{Conditional Preference Networks (CP-nets)}

We largely follow the notation introduced by Boutilier et al.~\cite{1-DBLP:journals/jair/BoutilierBDHP04} in their seminal work on CP-nets; the reader is referred to Table~\ref{table:notation} for a list of the most important notation used throughout our manuscript.

\begin{table}
\centering
\begin{small}
\begin{tabular}{|l|l|}
\hline
notation&meaning\\\hline\hline
$n$&number of variables\\\hline
$m$ &size of the domain of each variable\\\hline
$k$&upper bound on the number of parents of a variable in a CP-net\\\hline
$V$&set of $n$ distinct variables, $V=\{v_1,v_2,\dots,v_n\}$\\\hline
$v_i$&variable in $V$, for $1\le i\le n$\\\hline
$D_{v_i}$&domain $\{v^i_1,v^i_2,\dots,v^i_m\}$ of variable $v_i\in V$\\\hline
$\mathcal{O}_{U}$ for $U\subseteq V$&set of vectors (outcomes over $U$) that assign each $v_i\in U$ a value in $D_{v_i}$\\\hline
$\mathcal{O}$&set of all outcomes over the full variable set $V$; equal to $\mathcal{O}_V$\\\hline
$o$&outcome over the full variable set $V$, i.e., an element of $\mathcal{O}$\\\hline
$o\succ o'$&outcome $o\in\mathcal{O}$ is strictly preferred over outcome $o'\in\mathcal{O}$\\\hline
$o[U]$&projection of outcome $o\in\mathcal{O}$ onto a set $U\subseteq V$; $o[v_i]$ is short for $o[\{v_i\}]$\\\hline
$x (=(o,o'))$&swap pair of outcomes\\\hline
$V(x)$&swapped variable of $x$\\\hline
$x.1$ and $x.2$&$o$ and $o'$ respectively, where $x=(o,o')$ is a swap\\\hline
$x[\Gamma]$&projection of $x.1$ (and also of $x.2$) onto a set $\Gamma\subseteq V\setminus\{V(x)\}$\\\hline
$Pa(v_i)$&set of the parent variables of $v_i$; note that $Pa(v_i)\subseteq V\backslash\{v_i\}$\\\hline
$\succ^{v_i}_{u}$&conditional preference relation of ${v_i}$ in the context of $u$, where $u\in \mathcal{O}_{Pa(v_i)}$\\\hline
$\CPT(v_i)$&conditional preference table of $v_i$\\\hline
$size(\CPT(v_i))$&size (number of preference statements) of $\CPT(v_i)$; note $size(\CPT(v_i)) \le m^{|Pa(v_i)|}$\\\hline
$E$&set of edges in a CP-net, where $(v_i,v_j)\in E$ iff $v_i\in Pa(v_j)$\\\hline
$\mathcal{C}$&a concept class\\\hline
$c$&a concept in a concept class\\\hline
$\mathcal{X}$&instance space over which a concept class is defined\\\hline
$\mathcal{X}_{swap}$&instance space of swap examples (without redundancies)\\\hline
$\overline{\mathcal{X}}_{swap}$&instance space of swap examples (with redundancies)\\\hline
$c(x)$&label that concept $c$ assigns to instance $x$\\\hline
$\VCdim(\mathcal{C})$&VC dimension of concept class $\mathcal{C}$\\\hline
$\TD(\mathcal{C})$&teaching dimension of concept class $\mathcal{C}$\\\hline
$\TD(c,\mathcal{C})$&teaching dimension of concept $c$ with respect to concept class $\mathcal{C}$\\\hline
$\RTD(\mathcal{C})$&recursive teaching dimension of concept class $\mathcal{C}$\\\hline
$\mathcal{C}^k_{ac}$&class of all complete acyclic $k$-bounded CP-nets, over $\mathcal{X}_{swap}$\\\hline
$\overline{\mathcal{C}}^k_{ac}$&class of all complete or incomplete acyclic $k$-bounded CP-nets, over $\overline{\mathcal{X}}_{swap}$\\\hline
$e_{max}$&$(n-k)k+\binom{k}{2}\ (\le nk)$; maximum number of edges in a CP-net in $\mathcal{C}^k_{ac}$ or $\overline{\mathcal{C}}^k_{ac}$\\\hline
$\mathcal{M}_k$&$(n-k)m^k+\frac{m^k-1}{m-1}$; maximum number of statements in a CP-net in $\mathcal{C}^k_{ac}$ or $\overline{\mathcal{C}}^k_{ac}$\\\hline
$\mathcal{U}_k$&smallest possible size of an $(m,n-1,k)$-universal set\\\hline
LIM&strategy to combat a limited oracle\\\hline
MAL&strategy to combat a malicious oracle\\\hline
$F^1(x)$&set of all swap instances differing from $x$ in exactly one non-swapped variable\\\hline
\end{tabular}
\end{small}
\caption{Summary of notation.}\label{table:notation}
\end{table}

Let $V=\{v_1,v_2,\dots,v_n\}$ be a set of attributes or variables. Each variable $v_i\in V$ has a set of possible values (its domain) $D_{v_i}=\{v^i_1,v^i_2,\dots,v^i_m\}$. We assume that every domain $D_{v_i}$ is of a fixed size $m\geq 2$, independent of $i$. An assignment $x$ to a set of variables $U\subseteq V$ is a mapping for every variable $v_i\in U$ to a value from $D_{v_i}$. We denote the set of all assignments of $U\subseteq V$ by $\mathcal{O}_{U}$ and remove the subscript when $U=V$. A preference is an irreflexive, transitive binary relation $\succ$. For any $o,o'\in \mathcal{O}$, we write $o\succ o'$ (resp. $o\nsucc o'$) to denote the fact that $o$ is strictly preferred (resp. not preferred) to $o'$, where $o$ and $o'$ are incomparable w.r.t.~$\succ$ if both $o\nsucc o'$ and $o'\nsucc o$ holds. We use $o[U]$ to denote the projection of $o$ onto $U\subset V$ and write $o[v_i]$ instead of $o[\{v_i\}]$.

The CP-net model captures complex qualitative preference statements in a graphical way. Informally, a CP-net is a set of statements of the form $\gamma:v^i_{\sigma(1)}\succ\ldots\succ v^i_{\sigma(m)}$ which states that the preference over $v_i$ with $D_{v_i}=\{v^i_1,v^i_2,\dots,v^i_m\}$ is conditioned upon the assignment of $\Gamma\subseteq V\backslash\{v_i\}$, where $\sigma$ is some permutation over $\{1,\ldots,m\}$. In particular, when $\Gamma$ has the value $\gamma\in \mathcal{O}_\Gamma$ and $s<t$, $v^i_{\sigma(s)}$ is preferred to $v^i_{\sigma(t)}$ as a value of $v_i$ \emph{ceteris paribus} (all other things being equal). That is, for any two outcomes $o,o'\in \mathcal{O}$ where $o[v_i]=v^i_{\sigma(s)}$ and $o'[v_i]=v^i_{\sigma(t)}$ the preference holds when i) $o[\Gamma]=o'[\Gamma]=\gamma$ and ii) $o[Z]=o'[Z]$ for $Z=V\backslash (\Gamma\cup\{v_i\})$. In such case, we say $o$ is preferred to $o'$ ceteris paribus. Clearly, there could be exponentially many pairs of outcomes ($o$,$o'$) that are affected by one such statement. 


CP-nets provide a compact representation of preferences over $\mathcal{O}$ by providing such statements for every variable. For every $v_i\in V$, the decision maker\footnote{This can be any entity in charge of constructing the preference network, i.e., a computer agent, a person, a group of people, etc.} chooses a set $Pa(v_i)\subseteq V\backslash\{v_i\}$ of parent variables that influence the preference order of ${v_i}$. For any $\gamma\in \mathcal{O}_{Pa(v_i)}$, the decision maker may choose to specify a total order $\succ^{v_i}_{\gamma}$ over $D_{v_i}$. We refer to $\succ^{v_i}_{\gamma}$ as the conditional preference statement of ${v_i}$ in the context of $\gamma$. A Conditional Preference Table for $v_i$, $\CPT(v_i)$, is a set of conditional preference statements $\{\succ^{v_i}_{\gamma_1},\ldots,\succ^{v_i}_{\gamma_z}\}$.

\begin{definition}[CP-net \cite{1-DBLP:journals/jair/BoutilierBDHP04}]
Given, $V$, $Pa(v)$, and $\CPT(v)$ for $v\in V$, a CP-net is a directed graph $(V,E)$, where, for any $v_i,v_j\in V$, $(v_i,v_j)\in E$ iff $v_i\in Pa(v_j)$. 
\label{def:cpnets}
\end{definition}

A CP-net is acyclic if it contains no cycles. We call a CP-net $k$-bounded, for some $k\le n-1$, if each vertex has indegree at most $k$, i.e., each variable has a parent set of size at most $k$. CP-nets that are $0$-bounded are also called separable; those that are $1$-bounded are called tree CP-nets. When speaking about the class of ``unbounded'' acyclic CP-nets, we refer to the case when no upper bound is given on the indegree of nodes in a CP-net, other than the trivial bound $k=n-1$. 

\begin{definition}\label{def:complete}
$\CPT(v_i)$ is said to be complete, if, for every element $\gamma\in\mathcal{O}_{Pa(v_i)}$, the preference relation $\succ^{v_i}_\gamma$ is defined, i.e., $\CPT(v_i)$ contains a statement that imposes a total order on $D_{v_i}$ for every context $\gamma$ over the parent variables. By contrast, $\CPT(v_i)$ is incomplete, if there exists some $\gamma\in\mathcal{O}_{Pa(v_i)}$ for which the preference relation $\succ^{v_i}_\gamma$ is empty. Analogously, a CP-net $N$ is said to be complete if every $\CPT$ it poses is complete; otherwise it is incomplete.
\end{definition}

Note that we do not allow strictly partial orders as preference statements; a preference relation in a $\CPT$ must either be empty or impose a total order. In the case of binary CP-nets, which is the focus of the majority of the literature on CP-nets, this restriction is irrelevant, since every order on a domain of two elements is either empty or total. In the non-binary case though, the requirement that every $\CPT$ statement be either empty or a total order is a proper restriction. 

It would be possible to also study non-binary CP-nets that are incomplete in the sense that some $\CPT$ statements impose proper partial orders, but this extension is not discussed below.

Lastly, we assume CP-nets are defined in their minimal form, i.e., there is no \emph{dummy} parent in any $\CPT$ that actually does not affect the preference relation.


\begin{example}
Figure \ref{fig:example1network} shows a complete acyclic CP-net over $V=\{A,B,C\}$ with $D_A=\{a,\bar{a}\}$, $D_B=\{b,\bar{b}\}$, $D_C=\{c,\bar{c}\}$. Each variable is annotated with its $\CPT$. For variable $A$, the user prefers $a$ to $\bar{a}$ unconditionally. For $C$, the preference depends on the values of $A$ and $B$, i.e., $Pa(C)=\{A,B\}$. For instance, in the context of $a\bar{b}$, $\bar{c}$ is preferred over $c$. Removing any of the four statements in $\CPT(C)$ would result in an incomplete CP-net.
\label{example1}
\end{example}
Two outcomes $o,\hat{o}\in \mathcal{O}$ are swap outcomes (`swaps' for short) if they differ in the value of exactly one variable $v_i$; then $v_i$ is called the swapped variable \cite{1-DBLP:journals/jair/BoutilierBDHP04}. 

The size of a preference table for a variable $v_i$, denoted by $size(\CPT(v_i))$, is the number of preference statements it holds which is $m^{|Pa(v_i)|}$ if $\CPT(v_i)$ is complete. The size of a CP-net $N$ is defined as the sum of its tables' sizes.\footnote{It might seem more adequate to define the size of a $\CPT$ to be $(m-1)$ times the number of its preference statements, as each preference statement consists of $m-1$ pairwise preferences. In the binary case, i.e., when $m=2$, this makes no difference. As this technical detail does not affect our results, we ignore it and define the size of a $\CPT$ and of a CP-net simply by the overall number of its statements.}

\begin{example}
In Figure \ref{fig:example1}, $abc,\bar{a}bc$ are swaps over the swapped variable $A$. The size of the CP-net is $1+2+4=7$.
\end{example}

We will frequently use the notation $\mathcal{M}_k=\max\{size(N)\mid N\mbox{ is a }k\mbox{-bounded acyclic CP-net}\}$, which refers to the maximum number of statements in any $k$-bounded acyclic CP-net over $n$ variables, each of domain size $m$. Note that a CP-net has to be complete in order to attain this maximum size. It can be verified that $\mathcal{M}_k=(n-k)m^k+\frac{m^k-1}{m-1}$.

\begin{lemma}\label{lem:Mk}
The maximum possible size $\mathcal{M}_k$ of a $k$-bounded acyclic CP-net over $n$ variables of domain size $m$ is given by $\mathcal{M}_k=(n-k)m^k+\frac{m^k-1}{m-1}$.
\end{lemma}

\begin{proof}
We first make the following claim: any $k$-bounded acyclic CP-net of largest possible size has (i) exactly $1$ variable of indegree $r$, for any $r\in\{0,\ldots,k-1\}$ and (ii)~exactly $n-k$ variables of indegree $k$.

For $k=0$, i.e., for separable CP-nets, there is no $r\in\{0,\ldots,k-1\}$, so the claim states the existence of exactly $n$ vertices of indegree $0$, which is obviously correct. Consider any $k$-bounded acyclic CP-net $N$ of largest possible size, where $k\ge 1$. Since $N$ is acyclic, it has a topological sort. W.l.o.g., suppose $(v_1,\ldots,v_n)$ is the sequence of variables in $N$ as they occur in a topological sort. Clearly, $v_1$ must have indegree $0$. If $v_2$ were also of indegree $0$, then $N$ would not be of maximal size since one could add $v_1$ as a parent of $v_2$ without violating the indegree bound $k$. The resulting CP-net would be of larger size than $N$, since the size of $\CPT(v_2)$ would grow by a factor of $m$ without changing the sizes of other CPTs. Hence $v_2$ has indegree $1$ in $N$. With the same argument, one can prove that $v_i$ has indegree $i-1$ in $N$, for $1\le i\le k+1$. For the variables $v_{k+2},\ldots,v_n$, one can apply the same argument but has to cap their indegrees at $k$ because $N$ is $k$-bounded. Hence $v_{k+1},\ldots,v_n$ all have indegree $k$. This establishes the claim.

It remains to count the maximal number of statements in a CP-net $N$ of this specific structure. The maximal number of statements for a given CP-net graph is obviously obtained when the CP-net is complete, i.e., when the CPT for any variable $v$ has $m^{|Pa(v)|}$ rules. Summing up, we obtain $\sum_{i=0}^{k-1}m^i=\frac{m^k-1}{m-1}$ statements for the first $k$ variables in the topological sort, plus $(n-k)m^k$ statements for the remaining $n-k$ variables.
\end{proof}

From this lemma, we also know that the maximum possible number of edges in a $k$-bounded acyclic CP-net is $(n-k)k+\sum_{i=0}^{k-1}i=(n-k)k+\binom{k}{2}$. We will use the notation $e_{max}$ to refer to this quantity. 

\begin{definition}\label{def:emax}
For given $n\ge 1$ and $k\in\{0,\ldots,n-1\}$, let $e_{max}=(n-k)k+\binom{k}{2}$ denote the maximum possible number of edges in a $k$-bounded acyclic CP-net over $n$ variables.
\end{definition}

Note that $e_{max}\le nk$.

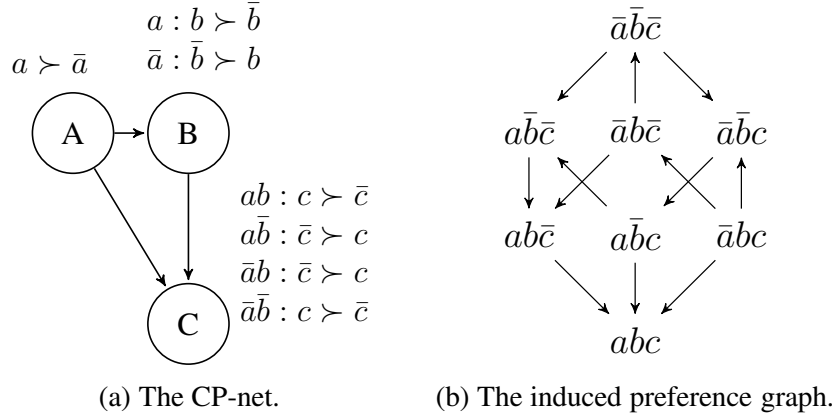
\begin{figure}
\centering
\begin{subfigure}[b]{0.35\textwidth}
\centering
\resizebox{.9\textwidth}{!}{
\begin{tikzpicture}[->,>=stealth',shorten >=1pt,auto,node distance=1.5cm,semithick,scale=1]
\node[state,] (A){A};
\node[state,] (B)[right of=A]{B};
\node[state] (C) at (1.5,-2.5){C};

\node (bCPT) at(1.7,1.2){\makecell[l]{$a:b\succ \bar{b}$\\$\bar{a}:\bar{b}\succ b$}};
\node (aCPT) at(-0.3,.9){$a\succ \bar{a}$};
\node (cCPT) at(3.0,-1.6){\makecell[l]{$ ab: c\succ \bar{c}$\\$a\bar{b}: \bar{c}\succ c$\\
		$\bar{a}b: \bar{c}\succ c$\\ $\bar{a}\bar{b}: c\succ \bar{c}$\\}};

\path (A) edge[->] (C);
\path (A) edge[->] (B);
\path (B) edge[->] (C);
\end{tikzpicture}
}
\caption{The CP-net.}
\label{fig:example1network}
\end{subfigure}
\begin{subfigure}[b]{.35\textwidth}
\centering
\resizebox{.7\textwidth}{!}{
\begin{tikzpicture}[->,>=stealth',auto,node distance=1.2cm]
\node (2) {$a\bar{b}\bar{c}$};
\node (3) [right of=2]{$\bar{a}b\bar{c}$};
\node (4) [right of=3] {$\bar{a}\bar{b}c$};
\node (5) [below of=2] {$ab\bar{c}$};
\node (6) [right of=5]{$a\bar{b}c$};
\node (7) [right of=6]{$\bar{a}bc$};
\node (8) [below of=6]{$abc$};
\node (1) [above of=3]{$\bar{a}\bar{b}\bar{c}$};

\path (1) edge (2) edge (4)
(2) edge (5) 
(3) edge (5) edge (1)
(4) edge (6) 
(5) edge (8)
(6) edge (2) edge (8)
(7) edge (3) edge (4) edge (8);

\end{tikzpicture}
}
\caption{The induced preference graph.}
\label{fig:graph}
\end{subfigure}
\caption{An acyclic CP-net (cf.~Def. \ref{def:cpnets}) and its induced preference graph (cf.~Def. \ref{def:inducedGraph}).}
\label{fig:example1}
\end{figure}


The semantics of CP-nets is described in terms of improving flips. Let $\gamma\in \mathcal{O}_{Pa(v_i)}$ be an assignment of the parents for a variable $v_i\in V$. Let $\succ^{v_i}_\gamma=v^i_1\succ \ldots\succ v^i_m$ be the preference order of $v_i$ in the context of $\gamma$. Then, all else being equal, going from $v^i_j$ to $v^i_k$ is an improving flip over $v_i$ whenever $k<j\leq m$. 

\begin{example}
In Figure \ref{fig:example1network}, $(a\bar{b}c,abc)$ is an improving flip with respect to the variable $B$. 
\end{example}

For complete CP-nets, the improving flip notion makes every pair $(o,\hat{o})$ of swap outcomes comparable, i.e., either $o\succ \hat{o}$ or $\hat{o}\succ o$ holds \cite{1-DBLP:journals/jair/BoutilierBDHP04}. The question \quotes{is $o\succ \hat{o}$?} is then a special case of a so-called dominance query and can be answered directly from the preference table of the swapped variable. Let $v_i$ be the swapped variable of a swap $(o,\hat{o})$. Let $\gamma$ be the context of $Pa(v_i)$ in both $o$ and $\hat{o}$. Then, $o\succ \hat{o}$ iff $o[v_i]\succ^{v_i}_\gamma \hat{o}[v_i]$. 
A general dominance query is of the form: given two outcomes $o,\hat{o}\in \mathcal{O}$, is $o\succ \hat{o}$? The answer is yes, iff $o$ is preferred to $\hat{o}$, i.e., there is a sequence $(\lambda_1,\ldots,\lambda_n)$ of improving flips from $\hat{o}$ to $o$, where $\hat{o}=\lambda_1$, $o=\lambda_n$, and $(\lambda_i,\lambda_{i+1})$ is an improving flip for all $i\in\{1,\dots,n-1\}$ \cite{1-DBLP:journals/jair/BoutilierBDHP04}. 

\begin{example}
In Figure \ref{fig:graph}, $abc\succ \bar{a}\bar{b}c$, as witnessed by the sequence $\bar{a}\bar{b}c\rightarrow a\bar{b}c\rightarrow abc$ of improving flips.
\end{example}



\begin{definition}[Induced Preference Graph~\cite{1-DBLP:journals/jair/BoutilierBDHP04}]
The induced preference graph of a CP-net $N$ is a directed graph $G$ where each vertex represents an outcome $o\in \mathcal{O}$. An edge from $\hat{o}$ to ${o}$ exists iff $(o,\hat{o})\in\mathcal{O}\times \mathcal{O}$ is a swap w.r.t.\ some $v_i\in V$ and ${o}[v_i]$ precedes $\hat{o}[v_i]$ in $\succ^{v_i}_{o[Pa(v_i)]}$.
\label{def:inducedGraph}
\end{definition}
 Therefore, a CP-net $N$ defines a partial order $\succ$ over $\mathcal{O}$ that is given by the transitive closure of its induced preference graph. If $o\succ \hat{o}$ we say $N$ entails $(o,\hat{o})$. $N$ is consistent if there is no $o\in \mathcal{O}$ with $o\succ o$, i.e., if its induced preference graph is acyclic. Acyclic CP-nets are guaranteed to be consistent while such guarantee does not exist for cyclic CP-nets; the consistency of the latter depends on the actual values of the $\CPT$s \cite{1-DBLP:journals/jair/BoutilierBDHP04}. Lastly, the complexity of finding the best outcome in an acyclic CP-net has been shown to be linear \cite{1-DBLP:journals/jair/BoutilierBDHP04} while the complexity of answering dominance queries depends on the structure of CP-nets: PSPACE-complete for arbitrary (cyclic and acyclic) consistent CP-nets \cite{3-DBLP:journals/jair/GoldsmithLTW08} and linear in case of trees \cite{DBLP:conf/uai/BigotZFM13}. 

\begin{example}
Figure \ref{fig:cyclicConsistent} shows an example of a cyclic CP-net that is consistent while Figure \ref{fig:cyclicInconsistent} shows an inconsistent one. Note that both share the same $\CPT$s except for $\CPT(C)$. The dotted edges in the induced preference graph of Figure \ref{fig:cyclicInconsistent} represent a cycle.
\end{example}

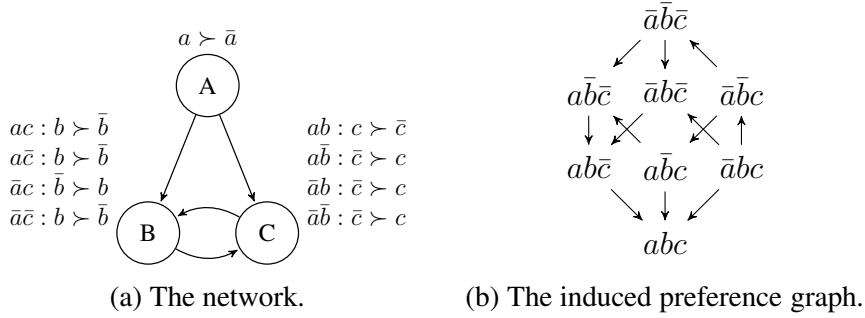
\begin{figure}
\centering
\begin{subfigure}[b]{.34\textwidth}
\centering
\resizebox{\textwidth}{!}{
\begin{tikzpicture}[->,>=stealth',shorten >=1pt,auto,node distance=1.5cm,semithick,scale=1]
\node[state,] (A){A};
\node[state,] (B)at (-1,-2.5){B};
\node[state] (C) at (1,-2.5){C};

\node (aCPT) at(0,.8){$a\succ \bar{a}$};
\node (bCPT) at(-2.5,-1.5){\makecell[l]{$a c: b\succ \bar{b}$\\
		$a \bar{c}: b\succ \bar{b}$\\$\bar{a} c: \bar{b}\succ b$\\$\bar{a} \bar{c}: b\succ \bar{b}$}};
\node (cCPT) at(2.5,-1.5){\makecell[l]{$a b: c\succ \bar{c}$\\
		$a \bar{b}: \bar{c}\succ c$\\$\bar{a} b: \bar{c}\succ c$\\$\bar{a} \bar{b}: \bar{c}\succ c$}};

\path (A) edge[->] (C)
(A) edge[->] (B)
(B) edge[bend right,->] (C)
(C) edge[bend right,->] (B);
\end{tikzpicture}
}
\caption{The network.}
\label{fig:ccnetwork}
\end{subfigure}\quad
\begin{subfigure}[b]{.34\textwidth}
\centering
\resizebox{.55\textwidth}{!}{
\begin{tikzpicture}
[->,>=stealth',auto,node distance=1cm]
\node (2) {$a\bar{b}\bar{c}$};
\node (3) [right of=2]{$\bar{a}b\bar{c}$};
\node (4) [right of=3] {$\bar{a}\bar{b}c$};
\node (5) [below of=2] {$ab\bar{c}$};
\node (6) [right of=5]{$a\bar{b}c$};
\node (7) [right of=6]{$\bar{a}bc$};
\node (8) [below of=6]{$abc$};
\node (1) [above of=3]{$\bar{a}\bar{b}\bar{c}$};

\path (1) edge (2)  edge (3) 
(2) edge (5)
(3) edge (5) 
(4) edge (1) edge (6) 
(5) edge (8)
(6) edge (2) edge (8)
(7) edge (3) edge (4) edge (8);
\end{tikzpicture}
}
\caption{The induced preference graph.}
\label{fig:cyclicConsistentGraph}
\end{subfigure}
\caption{An example of a \emph{consistent} cyclic CP-net.}
\label{fig:cyclicConsistent}
\end{figure}


\begin{figure}
\centering
\begin{subfigure}[b]{.34\textwidth}
\centering
\resizebox{\textwidth}{!}{
\begin{tikzpicture}[->,>=stealth',shorten >=1pt,auto,node distance=1.5cm,semithick,scale=1]
\node[state,] (A){A};
\node[state,] (B)at (-1,-2.5){B};
\node[state] (C) at (1,-2.5){C};

\node (aCPT) at(0,.8){$a\succ \bar{a}$};
\node (bCPT) at(-2.5,-1.5){\makecell[l]{$a c: b\succ \bar{b}$\\
		$a \bar{c}: b\succ \bar{b}$\\$\bar{a} c: \bar{b}\succ b$\\$\bar{a} \bar{c}: b\succ \bar{b}$}};
\node (cCPT) at(2.3,-1.5){\makecell[l]{$b: c\succ \bar{c}$\\
		$\bar{b}: \bar{c}\succ c$}};

\path 
(A) edge[->] (B)
(B) edge[bend right,->] (C)
(C) edge[bend right,->] (B);
\end{tikzpicture}
}
\caption{The network.}
\end{subfigure}\quad
\begin{subfigure}[b]{.34\textwidth}
\centering
\resizebox{.55\textwidth}{!}{
\begin{tikzpicture}
[->,>=stealth',auto,node distance=1cm]
\node (2) {$a\bar{b}\bar{c}$};
\node (3) [right of=2]{$\bar{a}b\bar{c}$};
\node (4) [right of=3] {$\bar{a}\bar{b}c$};
\node (5) [below of=2] {$ab\bar{c}$};
\node (6) [right of=5]{$a\bar{b}c$};
\node (7) [right of=6]{$\bar{a}bc$};
\node (8) [below of=6]{$abc$};
\node (1) [above of=3]{$\bar{a}\bar{b}\bar{c}$};

\path (1) edge (2)  edge[blue!50!black,dotted] (3) 
(2) edge (5)
(3) edge (5) edge[blue!50!black,dotted] (7)
(4) edge[blue!50!black,dotted] (1) edge (6) 
(5) edge (8)
(6) edge (2) edge (8)
(7)  edge[blue!50!black,dotted] (4) edge (8);
\end{tikzpicture}
}
\caption{The induced preference graph.}
\label{fig:cyclicInconsistentGraph}
\end{subfigure}
\caption{An example of an \emph{inconsistent} cyclic CP-net.}
\label{fig:cyclicInconsistent}
\end{figure}
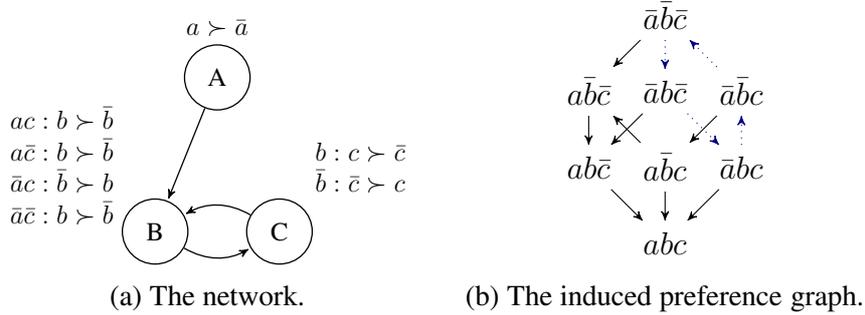

\subsection{Concept Learning}

The first part of our study is concerned with determining---for the case of acyclic CP-nets---the values of information complexity parameters that are typically studied in computational learning theory. By information complexity, we mean the complexity in terms of the amount of information a learning algorithm needs to identify a CP-net. Examples of such complexity notions will be introduced below.
 
A specific complexity notion corresponds to a specific formal model of machine learning. Each such learning model assumes that there is an information source that supplies the learning algorithm with information about a hidden target concept $c^*$. The latter is a member of a concept class, which is simply the class of potential target concepts, and, in the context of this paper, also the class of hypotheses that the learning algorithm can formulate in the attempt to identify the target concept $c^*$. 

Formally, one fixes a finite set $\mathcal{X}$, called instance space, which contains all possible instances (i.e., elements of) an underlying domain. A concept $c$ is then defined as a mapping from $\mathcal{X}$ to $\{0,1\}$. Equivalently, $c$ can be seen as the set $c=\{x\in \mathcal{X} \mid c(x)=1\}$, i.e., a subset of the instance space.  A \emph{concept class} $\mathcal{C}$ is a set of concepts. Within the scope of our study, the information source (sometimes called oracle), supplies the learning algorithm in some way or another with a set of \emph{labeled examples}\/ for the target concept $c^*\in C$. A labeled example for $c^*$ is a pair $(x,b)\in \mathcal{X}\times \{0,1\}$ where  $x\in \mathcal{X}$ and $b=c(x)$. Under the set interpretation of concepts, this means that $b=1$ if and only if the instance $x$ belongs to the concept $c^*$. A concept $c$ is consistent with a set $S\subseteq X\times\{0,1\}$ of labeled examples, if and only if $c(x)=b$ for all $(x,b)\in S$, i.e., if every element of $S$ is an example for $c$.

In practice, a concept is usually \emph{encoded} by a representation $\sigma(c)$ defined based on a representation class $\mathcal{R}$ \cite{Kearns:1994:ICL:200548}. Thus, one usually has some fixed representation class $\mathcal{R}$ in mind, with a one-to-one correspondence between the concept class $C$ and its representation class $R$. We will assume in what follows that the representation class is chosen in a way that minimizes the worst case size of the representation of any concept in $C$. Generally, there may be various interpretations of the term ``size;'' since we will focus on learning CP-nets, we use CP-nets as representations for concepts, and the size of a representation is simply the size of the corresponding CP-net as defined above.  

At the onset of a learning process, both the oracle and the learning algorithm (often called \emph{learner}\/ for short) agree on the representation class $\mathcal{R}$ (and thus also on the concept class $C$), but only the oracle knows the target concept $c^*$. After some period of communication with the oracle, the learner is required to identify the target concept $c^*$ either exactly or approximately.

Many learning models have been proposed to deal with different learning settings \cite{Kearns:1994:ICL:200548,Angluin:1988:QCL:639961.639995,Littlestone88,V84}. These models typically differ in the constraints they impose on the oracle and the learning goal. One also distinguishes between learners that actively query the oracle for specific information content and learners that passively receive a set of examples chosen solely by the information source. One of the best known passive learning models is the \emph{Probably Approximately Correct } (PAC) model \cite{V84}. The PAC model is concerned with finding, with high probability, a close approximation to the target concept $c^*$ from randomly chosen examples. The examples are assumed to be sampled independently from an unknown distribution. On the other end of the spectrum, a model that requires exact identification of $c^*$ is Angluin's model for learning from queries \cite{Angluin:1988:QCL:639961.639995}. In this model, the learner actively poses queries of a certain type to the oracle. 

In this paper, we consider specifically two types of queries introduced by Angluin~\cite{Angluin:1988:QCL:639961.639995}, namely membership queries and equivalence queries. A membership query is specified by an element $x\in\mathcal{X}$ of the instance space, and it represents the question whether or not $c^*$ contains $x$. The oracle supplies the learner with the correct answer, i.e., it provides the label $c^*(x)$ in response to the membership query for $x$. In an equivalence query, the learner specifies a hypothesis $c$. If $c=c^*$, the learning process is completed as the learner has then identified the target concept. If $c\ne c^*$, the learner is provided with a labeled example $(x,c^*(x))$ that witnesses $c\ne c^*$. That means, $c^*(x)\ne c(x)$. Note that $x$ in this case can be any element in the symmetric difference of the sets associated with $c$ and $c^*$. 

A class $\mathcal{C}\subseteq 2^\mathcal{X}$ over some instance space $\mathcal{X}$ is learnable from membership and/or equivalence queries via a representation class $\mathcal{R}$ for $\mathcal{C}$, if there is an algorithm $\mathcal{A}$ such that for every concept $c^*\in \mathcal{C}$, $\mathcal{A}$ asks polynomially many adaptive membership and/or equivalence queries and then outputs a hypothesis $h$ that is equivalent to $c^*$. By adaptivity, we here mean that learning proceeds in rounds; in every round the learner asks a single query and receives an answer from the oracle before deciding on its subsequent query. The number of queries to be polynomial means that it is upper-bounded by a polynomial in $size(c^*)$ where $size(c^*)$ is the size of the minimal representation of $c^*$ w.r.t.~$\mathcal{R}$.

The above definition is concerned only with the information or query complexity, i.e., the number of queries required to exactly identify any target concept. Moreover, $\mathcal{C}$ is said to be efficiently learnable from membership and/or equivalence queries if there exists an algorithm $\mathcal{A}$ that exactly learns $\mathcal{C}$, in the above sense, and runs in time polynomial in $size(c^*)$. Every one of the query strategies we describe in Section \ref{sec:perfect} gives an obvious polynomial time algorithm in this regard, and thus we will not explicitly mention run-time efficiency of learning algorithms henceforth.

The combinatorial structure of a concept class $\mathcal{C}$ has implications on the complexity of learning $\mathcal{C}$, in particular on the sample complexity (sometimes called information complexity), which refers to the number of labeled examples the learner needs in order to identify any target concept in the class under the constraints of a given learning model. One of the most important complexity parameters studied in machine learning is the Vapnik-Chervonenkis dimension (VCD). In what follows, let $\mathcal{C}$ be a concept class over the (finite) instance space $\mathcal{X}$.

\begin{definition}\cite{VC71} A subset $Y\subseteq\mathcal{X}$ is shattered by $\mathcal{C}$ if the projection of $\mathcal{C}$ onto $Y$ has $2^{|Y|}$ concepts. The VC dimension of $\mathcal{C}$, denoted by $\VCdim(\mathcal{C})$, is the size of the largest subset of $\mathcal{X}$ that is shattered by $\mathcal{C}$.
\end{definition}

For example, if $\mathcal{X}$ contains 5 elements and $\mathcal{C}^5_3$ is the class of all subsets of $X$ that have size at most 3, then the VC dimension of  $\mathcal{C}^5_3$ is 3. Clearly, no subset $Y$ of $\mathcal{X}$ of size 4 can be shattered by $\mathcal{C}^5_3$, since no concept in $\mathcal{C}^5_3$ would contains all 4 elements of $Y$. That means, one obtains only 15, not the full 16 possible concepts  over $Y$ when projecting $\mathcal{C}^5_3$ onto $Y$. However, any subset $Y'\subset X$ of size 3 is indeed shattered by $\mathcal{C}^5_3$, as every subset of $Y'$ is also a concept in $\mathcal{C}^5_3$.

The number of randomly chosen examples needed to identify concepts from $\mathcal{C}$ in the PAC-learning model is linear in $\VCdim(\mathcal{C})$~\cite{BEHW89,Hanneke16}. By contrast to learning from random examples, in teaching models, the learner is provided with well-chosen labeled examples.

\begin{definition}\cite{GK95,SM91} A teaching set for a concept $c^*\in\mathcal{C}$ with respect to $\mathcal{C}$ is a set $S=\{(x_1,\ell_1),\ldots,(x_z,\ell_z)\}$ of labeled examples such that $c^*$ is the only concept $c\in\mathcal{C}$ that satisfies $c(x_i)=\ell_i$ for all $i\in\{1,\ldots,z\}$. The teaching dimension of $c$ with respect to $\mathcal{C}$, denoted by $\TD(c,\mathcal{C})$, is the size of the smallest teaching set for $c$ with respect to $\mathcal{C}$. The teaching dimension of $\mathcal{C}$, denoted by $\TD(\mathcal{C})$, is given by $\TD(\mathcal{C})=\max\{\TD(c,\mathcal{C})\mid c\in\mathcal{C}\}$.
\end{definition}

Consider again the class $\mathcal{C}^5_3$ of all subsets of size at most 3 over a 5-element instance space. Any concept $c$ containing 3 instances has a teaching set of size 3 in this class: the three positively labeled examples referring to the elements contained in $c$ uniquely determine $c$. However, concepts with fewer than 3 elements do not have teaching sets smaller than 5, since any set consisting of 2 positive and 2 negative examples agrees with at least two different concepts in $\mathcal{C}^5_3$, and so does every set of 1 positive and 3 negative examples and every set of 4 negative examples.

$\TD_{min}(\mathcal{C})=\min\{\TD(c,\mathcal{C})\mid c\in\mathcal{C}\}$ denotes the smallest TD of any $c\in\mathcal{C}$. In the class $\mathcal{C}^5_3$, the value for $\TD$ is 5, while the value for $\TD_{min}$ is 3.

A well-studied variation of teaching is called recursive teaching. Its complexity parameter, the recursive teaching dimension, is defined by recursively removing from $\mathcal{C}$ all the concepts with the smallest TD and then taking the maximum over the smallest TDs encountered in that process. For the corresponding definition of teachers, see~\cite{ZLHZ11}.

\begin{definition}\cite{ZLHZ11} Let $\mathcal{C}_0=\mathcal{C}$ and, for all $i$ such that $\mathcal{C}_i\ne\emptyset$, define $\mathcal{C}_{i+1}=\mathcal{C}_i\setminus\{c\in\mathcal{C}_i\mid\TD(c,\mathcal{C}_i)=\TD_{min}(\mathcal{C}_i)\}$. The recursive teaching dimension of $\mathcal{C}$, denoted by $\RTD(\mathcal{C})$, is defined by $\RTD(\mathcal{C})=\max\{\TD_{min}(\mathcal{C}_i)\mid i\ge 0\}$.
\label{def:RTD}
\end{definition}

As an example, consider $\mathcal{C}'=\{c_1,c_2,\dots,c_t\}$ to be the class of singletons defined over the instance space $\mathcal{X}=\{x_1,x_2,\dots,x_t\}$ where $c_i=\{x_i\}$ and let $\mathcal{C}=\mathcal{C}'\cup\{c_0\}$, where $c_0$ is the empty concept, i.e.,  $c_0(x)=0$ for all $x\in \mathcal{X}$. Table \ref{table:singletonsClass} displays this class along with $\TD(c,\mathcal{C})$ for every $c\in \mathcal{C}$. Since distinguishing the concept $c_0$ from all other concepts in $\mathcal{C}$ requires $t$ labeled examples, one obtains $\TD(\mathcal{C})=t$. However, $\RTD(\mathcal{C})=1$ as witnessed by $\mathcal{C}_0=\mathcal{C}'$ (each concept in $\mathcal{C}'$ can be taught with a single example) and $\mathcal{C}_1=\{c_0\}$ (the remaining concept $c_0$ has a teaching dimension of $0$ with respect to the class containing only $c_0$). Note that also $\VCdim(\mathcal{C})=1$, since there is no set of two examples that is shattered by $\mathcal{C}$.

Similarly, one can verify that $\RTD(\mathcal{C}^5_3)=3$.

\begin{table}
	\centering
	\caption{The class $\mathcal{C}$ of all singletons and the empty concept over a set of $t$ instances, along with the teaching dimension value of each individual concept.}
	\begin{tabular}{|c||c|c|c|c|c|c||c|}
		\hline
	$\mathcal{C}$& $x_1$ & $x_2$ & $x_3$ & $x_4$ &$\dots$ & $x_t$ & $\TD$\\\hline
	\hline
	$c_0$ 			  &		0	 &		0	 &    0     &   0 &   $\dots$      &    0   & $t$ \\
	\hline
	$c_1$ 			  &		1	 &		0	 &    0     &   0 &$\dots$      &    0     & 1\\
	\hline
	$c_2$ 			  &		0	 &		1	 &    0     &   0 &$\dots$      &    0     & 1\\
	\hline
	$c_3$ 			  &		0	 &		0	 &    1     &   0 &$\dots$      &    0     & 1\\
	\hline
	$\vdots$ 	     &		$\vdots$	 &		$\vdots$	 &   $\vdots$& $\vdots$    &      $\vdots$      &    $\vdots$     & $\vdots$ \\
	\hline
	$c_t$ 			  &		0	 &		0	 &    0     &      0& $\dots$      &    1    & 1 \\
	\hline
	\end{tabular}
	\label{table:singletonsClass}
\end{table}

As opposed to the TD, the RTD exhibits interesting relationships to the VCD. For example, if $\mathcal{C}$ is a maximum class, i.e., its size $|\mathcal{C}|$ meets Sauer's upper bound $\binom{|\mathcal{X}|}{0}+\binom{|\mathcal{X}|}{1}+\ldots+\binom{|\mathcal{X}|}{\VCdim(\mathcal{C})}$ \cite{Sau72}, and in addition $\mathcal{C}$ can be ``corner-peeled''\footnote{Corner peeling is a sample compression procedure introduced by Rubinstein and Rubinstein~\cite{RR12}; the actual algorithm or its purpose are not of immediate relevance to our paper.}, then $\mathcal{C}$  fulfills $\RTD(\mathcal{C})=\VCdim(\mathcal{C})$~\cite{DFSZ14}. The same equality holds if $\mathcal{C}$ is intersection-closed or has VCD 1~\cite{DFSZ14}. In general, the RTD is upper-bounded by a function quadratic in the VCD~\cite{HuWLW17}.

\section{Representing CP-Nets as Concepts}
\label{sec:representation}

Assuming a user's preferences are captured in a target CP-net $N^*$, an interesting learning problem is to identify $N^*$ from a set of observations representing the user's preferences, i.e., labeled examples, of the form $o\succ o'$ or $o \nsucc o'$ where $\succ$ is the relation induced by $N^*$~\cite{DBLP:conf/ijcai/DimopoulosMA09,Koriche2010685}. In order to study the complexity of learning CP-nets, we model a class of CP-nets as a concept class over a fixed instance space. 

The first issue to address is how to define the instance space. A natural approach would be to consider any pair $(o,o')$ of outcomes as an instance. Such instance would be contained in the concept corresponding to $N^*$ if and only if $N^*$ entails $o\succ o'$. In our study however, we restrict the instance space to the set of all swaps.

On the one hand, note that our results, due to the restriction to swaps, do not apply to scenarios in which preferences are elicited over arbitrary outcome pairs, as is likely the case in many real-world applications. 

On the other hand, for various reasons, the restriction to swaps is still of both theoretical and practical interest and therefore justified. First, CP-net semantics are completely determined by the preference relation over swaps, so that no information on the preference order is lost by restricting to swaps. In particular, the set of all swaps is the most compact instance space for representing the class of all CP-nets or the class of all acyclic CP-nets. Second, many studies in the literature address learning CP-nets from information on swaps, see~\cite{Koriche2010685,DBLP:conf/ijcai/DimopoulosMA09,labernia2016query,LaberniaYMA18,LaberniaZMYA17}, so that our results can be compared to existing ones on the swap instance space. Third, in the learning models that we consider (teaching and learning from membership queries) learning becomes harder when restricting the instance space. For example, a learning algorithm may potentially succeed faster when it is allowed to enquire about preferences over any pair of outcomes rather than just swaps. Since our study restricts the information presented to the learner to preferences over swaps, our complexity results thus serve as upper bounds on the complexity of learning in more relaxed settings. Fourth, for the design of learning methods, it is often desirable that the learner can check whether its hypothesis (in our case a CP-net) is consistent with the labeled examples obtained. This can be done in linear time for swap examples but is NP-hard for non-swap examples \cite{1-DBLP:journals/jair/BoutilierBDHP04}. Fifth, there are potential application scenarios in which the information presented to a learner may be in the form of preferences over swaps. This is due to the intuition that in many cases preferences over swaps would be much easier to elicit than preferences over two arbitrary outcomes. For example, a user may be overwhelmed with the question whether to prefer one laptop over another if each of them has a nice feature that the other does not have. It is likely easier for the user to express a preference over two laptops that are identical except in a single feature. 

One may argue that the VC dimension should be computed over arbitrary instances rather than just swap instances, since it captures how difficult a concept class is to learn when the choice of instances is out of the learner's (or teacher's) control. However, for the following two reasons, computing the VC dimension over swap instances is of importance to our study:
\begin{itemize}
\item We use our calculations on the VC dimension in order to assess the optimality of one of Koriche and Zanuttini's algorithms \cite{Koriche2010685} for learning acyclic CP-nets with nodes of bounded indegree from equivalence and membership queries. It is well-known that $\log_2(4/3)\VCdim(\mathcal{C})$ is a lower bound on the number of membership and equivalence queries required for learning a concept class $\mathcal{C}$~\cite{AuerL99}. Since Koriche and Zanuttini's algorithm that we assess is designed over the swap instance space, an optimality assessment using the VC dimension necessarily requires that the VC dimension be computed over the swap instance space as well. 
\item Although the VC dimension is best known for characterizing the sample complexity of learning from randomly chosen examples, namely in the model of PAC learning, existing results exhibit a broader scope of applicability of the VC dimension. Recently it was shown that the number of examples needed for learning from benevolent teachers can be upper-bounded by a function quadratic in the VC dimension~\cite{HuWLW17}. When studying the number of swap examples required for teaching CP-nets, thus again also the VC dimension over swap examples becomes interesting.
\end{itemize}

We therefore consider the set $\{(o,o')\in\mathcal{O}\times\mathcal{O}\mid (o,o')$ is a swap$\}$ as an instance space. The size of this instance space is $nm^n(m-1)$: every variable has $m^{n-1}$ different assignments of the other variables and fixing each assignment of these we have $m(m-1)$ instances. For complete acyclic CP-nets, however, half of these instances are redundant as if $c((o,o'))=0$ then we know for certain that $c((o',o))=1$, and vice versa. By contrast, in the case of incomplete CP-nets, $c((o,o'))=0$ does not necessarily mean $c((o',o))=1$ as there could be no relation between the two outcomes, i.e., $o$ and $o'$ are incomparable, corresponding to both $c((o,o'))=0$ and $c((o',o))=0$. 

Consequently, the choice of instance space in our study will be as follows:
\begin{itemize}
\item Whenever we study classes of CP-nets that contain incomplete CP-nets, we use the instance space $\overline{\mathcal{X}}_{swap}=\{(o,o')\in\mathcal{O}\times\mathcal{O}\mid (o,o')$ is a swap$\}$.
\item Whenever we study classes of only complete CP-nets, we use an instance space $\mathcal{X}_{swap}\subset \overline{\mathcal{X}}_{swap}$ that includes, for any two swap outcomes $o,o'$, exactly one of the two pairs $(o,o'),(o',o)$.\footnote{All our results are independent on the mechanism choosing which of two pairs $(o,o'),(o',o)$ to include in $\mathcal{X}_{swap}$. We assume that the selection is prescribed in some arbitrary fashion.} Note that $|\mathcal{X}_{swap}|=nm^{n-1}\binom{m}{2}=\frac{m^nn(m-1)}{2}$.  When we say that a learning algorithm, either passively or through active queries, is given information on the label of a swap $(o,o')$ under the target concept, we implicitly refer to information on either $(o,o')$ or $(o',o)$, depending on which of the two is actually contained in $\mathcal{X}_{swap}$.
\end{itemize}

We sometimes refer to $\mathcal{X}_{swap}$ as the set of all swaps without ``redundancies'', since for the case of complete CP-nets half the instances in $\overline{\mathcal{X}}_{swap}$ are redundant. Of course, for incomplete CP-nets they are not redundant.


For $x=(o,o')\in \overline{\mathcal{X}}_{swap}$, let $V(x)$ denote the swapped variable of $x$. We refer to the first and second outcomes of an example $x$ as $x.1$ and $x.2$, respectively. We use $x[\Gamma]$ to denote the assignments (in both $x.1$ and $x.2$) of $\Gamma\subseteq V\backslash\{V(x)\}$. Note that $x[\Gamma]$ is guaranteed to be the same in $x.1$ and $x.2$, otherwise $x$ will not form a swap instance.

Now if $N$ is any CP-net and induces the (partial) order $\succ$ over outcomes, then $N$ corresponds to a concept $c_N$ over $\overline{\mathcal{X}}_{swap}$ (over $\mathcal{X}_{swap}$, respectively), where $c_N$ is defined as follows, for any $x\in \overline{\mathcal{X}}_{swap}$ (any $x\in \mathcal{X}_{swap}$, respectively.)

\begin{equation*}
c(x)=
\begin{cases}
1 & \text{if}\ x.1\succ x.2 \\
0 & \text{otherwise}
\end{cases}
\end{equation*}

In such case, we say that $c$ is represented by $N$. Since no two distinct CP-nets induce exactly the same set of swap entailments, a concept over the instance space $\overline{\mathcal{X}}_{swap}$ cannot be represented by more than one CP-net, and a concept over the instance space $\mathcal{X}_{swap}$ cannot be represented by more than one complete CP-net. Therefore, in the context of a specific instance space, we identify a CP-net $N$ with the concept $c_N$ it represents and use these two notions interchangeably. 

Consequently, we say that a concept $c$ contains a swap pair $x$ iff the CP-net representing $c$ entails $(x.1,x.2)$. By $size(c)$, we refer to the size of the CP-net that represents $c$. 

Table \ref{tab:swapInstanceExample} shows two concepts $c_1$ and $c_2$ that correspond to the complete CP-nets shown in Figures \ref{fig:example1} and \ref{fig:cyclicConsistent}, respectively, along with one choice of $\mathcal{X}_{swap}$. It is important to restate the fact that $c(x)$ is actually a dominance relation between $x.1$ and $x.2$, i.e., $c(x)$ is mapped to $1$ (resp. to $0$) if $x.1\succ x.2$ (resp. $x.2\succ x.1$) holds. Thus, we sometimes talk about the value of $c(x)$ in terms of the relation between $x.1$ and $x.2$ ($x.1\succ x.2$ or $x.2\succ x.1$). 

\begin{table}

\begin{adjustbox}{max width=\textwidth}
\begin{tabular}{ l || c | c | c | c || c | c | c | c || c | c | c | c |}
$\mathcal{X}_{swap}$&
$(abc,\bar{a}bc)$ &
$(ab\bar{c},\bar{a}b\bar{c})$ &
$(a\bar{b}c,\bar{a}\bar{b}c)$ & 
$(a\bar{b}\bar{c},\bar{a}\bar{b}\bar{c})$ &

$(abc,a\bar{b}c)$ &
$(ab\bar{c},a\bar{b}\bar{c})$ &
$(\bar{a}bc,\bar{a}\bar{b}c)$ &
$(\bar{a}b\bar{c},\bar{a}\bar{b}\bar{c})$ &

$(abc,ab\bar{c})$ & 
$(a\bar{b}c,a\bar{b}\bar{c})$ & 
$(\bar{a}bc,\bar{a}b\bar{c})$ & 
$(\bar{a}\bar{b}c,\bar{a}\bar{b}\bar{c})$\\
\hline
$c_1$ & 1& 1 &	1 &	1 &	1 &	1 &	0 &	0 &	1	& 0 & 0 & 1 \\
\hline 
$c_2$ & 1& 1 &	1 &	1 &	1 &	1 &	0 &	1 &	1	& 0 & 0 & 0 \\
\hline 
\end{tabular}
\end{adjustbox}
\caption{ The concepts $c_1$ and $c_2$ represent the CP-nets in Figures \ref{fig:example1} and \ref{fig:cyclicConsistent}, respectively, over $\mathcal{X}_{swap}$.}
\label{tab:swapInstanceExample}
\end{table}

In the remainder of this article, we fix $n\ge 1$, $m\ge 2$, $k\in\{0,\ldots,n-1\}$, and consider the following two concept classes:

\begin{itemize}
\item The class $\mathcal{C}_{ac}^k$ of all \emph{complete\/} acyclic $k$-bounded CP-nets over $n$ variables of domain size $m$. This class is represented over the instance space $\mathcal{X}_{swap}$.
\item The class $\overline{\mathcal{C}}_{ac}^k$ of \emph{all complete and all incomplete\/} acyclic $k$-bounded CP-nets over $n$ variables of domain size $m$. This class is represented over the instance space $\overline{\mathcal{X}}_{swap}$.
\end{itemize}

\section{The Complexity of Learning Complete Acyclic CP-Nets}\label{sec:complete}

In this section, we will study the information complexity parameters introduced above, for the class $\mathcal{C}^k_{ac}$ of all complete acyclic $k$-bounded CP-nets over $n$ variables of domain size $m$. In Section~\ref{sec:incomplete}, we will extend our results on the VC dimension and the teaching dimension also to the class $\overline{\mathcal{C}}^k_{ac}$ of all \emph{complete and incomplete}\/  acyclic $k$-bounded CP-nets. It turns out though that studying the complete case first is easier. 

Table~\ref{resultsSummary} summarizes our complexity results for complete acyclic $k$-bounded CP-nets. The two extreme cases are unbounded acyclic CP-nets ($k=n-1$) and separable CP-nets ($k=0$). 

To define the value $\mathcal{U}_k$ used in this table, we will first introduce the notion of $(z,k)$-universal set, which is typically used in combinatorics, cf.~\cite{adaptiveLearning,Jukna:2010:ECA:1965203}.

\begin{definition}\label{def:universal2}
Let $S\subseteq\{0,1\}^z$ be a set of binary vectors of length $z$ and let $k\le z$. The set $S$ is called $(z,k)$-universal if, for every set $Z=\{i_1,\ldots,i_k\}\subseteq\{1,\ldots z\}$ with $|Z|=k$, the projection
\[\{(s_{i_1},\ldots,s_{i_k})\mid (s_1,\ldots,s_z)\in S\}\]
of $S$ to the components in $Z$ is of size $2^k$, i.e., it contains all binary vectors of length $k$.
\end{definition}

In other words, a $(z,k)$-universal set $S$ is a set of concepts over an instance space $\mathcal{X}$ of size $z$, such that every subset of $\mathcal{X}$ of size $k$ is shattered by $S$. 

\begin{example}
Consider $z=10$ and $k=3$, and let $S$ be the set of all binary vectors of length 10 that contain no more than 3 components equal to 1. This set is $(10,3)$-universal, since projecting it to any set of three components yields all eight binary vectors of length 3. It is not $(10,4)$-universal since projections onto four components do not exhibit the binary vector consisting of four 1s.
\end{example}

The role of $(z,k)$-universal sets in our study will become evident later on when we investigate the teaching dimension of classes of CP-nets. Here we present a generalization of $(z,k)$-universal sets to the non-binary case in order to define the quantity $\mathcal{U}_k$ used in Table~\ref{resultsSummary}.

\begin{definition}\label{def:universal}
Let $S\subseteq\{1,\ldots, m\}^z$ be a set of vectors of length $z$ with components in $\{1,\ldots,m\}$. Let $k\le z$. The set $S$ is called $(m,z,k)$-universal if, for every set $Z=\{i_1,\ldots,i_k\}\subseteq\{1,\ldots z\}$ with $|Z|=k$, the projection
\[\{(s_{i_1},\ldots,s_{i_k})\mid (s_1,\ldots,s_z)\in S\}\]
of $S$ to the components in $Z$ is of size $m^k$, i.e., it contains all vectors of length $k$ with components in $\{1,\ldots,m\}$.
\end{definition}

Thus the classical $(z,k)$-universal sets correspond to $(2,z,k)$-universal sets in our notation. To date, no formula is known that expresses the exact size of a \emph{smallest possible\/} $(z,k)$-universal set (or of a smallest possible $(m,z,k)$-universal set, respectively) in dependence of $z$ and $k$ (in dependence of $m$, $z$, and $k$, respectively,) though some useful bounds on this quantity have been established~\cite{Jukna:2010:ECA:1965203}. The quantity is of importance for our results on the teaching dimension.

\begin{definition}\label{def:Uk}
Let $n\ge 1$, $m\ge 2$, $k\le n-1$. When studying CP-nets over $n$ variables, each of which has a domain of size $m$, we denote by $\mathcal{U}_k$ the smallest possible size of an $(m,n-1,k)$-universal set.
\end{definition}

The following simple observations on universal sets will be useful for our studies.

\begin{lemma}\label{lem:universal} The smallest sizes for $(m,n-1,k)$-universal sets, for $k\in\{0,1,n\}$ are as follows.
\begin{enumerate}
\item $\mathcal{U}_0=1$.
\item $\mathcal{U}_1=m$.
\item $\mathcal{U}_{n-1}=m^{n-1}$.
\end{enumerate}
\end{lemma}
\begin{proof}
The equality $\mathcal{U}_0=1$ is immediate by definition. $\mathcal{U}_1=m$ is obvious since, independently of $n$, the $m$ distinct vectors $(1,\ldots,1)$, $(2,\ldots,2)$, \dots, $(m,\ldots,m)$ form an $(m,n-1,1)$-universal set, and no smaller $(m,n-1,1)$-universal set can exist. Finally, to realize all possible assignments to the full universe of size $n-1$, exactly $m^{n-1}$ vectors are required, yielding $\mathcal{U}_{n-1}=m^{n-1}$.
\end{proof}

One observation from Table~\ref{resultsSummary} is that $\VCdim$ equals $\RTD$ for all values of $m$ in $\mathcal{C}_{ac}^{n-1}$. Further, if $m=2$ (the best-studied case in the literature), then $\TD(\mathcal{C}_{ac}^{n-1})$ equals the instance space size $n2^{n-1}$. A close inspection of the case $m=2$, as discussed in Appendix\ref{sec:struct}, shows that $\mathcal{X}_{swap}$ has only $n$ instances that are relevant for $\mathcal{C}_{ac}^0$, and $\mathcal{C}_{ac}^0$ corresponds to the class of all concepts over these $n$ instances. Thus the values of $\VCdim$, $\TD$, and $\RTD$ 
are trivially equal to $n$ in this special case. 
The remainder of this section is dedicated to proving the statements from Table~\ref{resultsSummary}. 

\begin{table*}
\centering
\caption{Summary of complexity results for classes of complete CP-nets. $\mathcal{M}_k=(n-k)m^k+\frac{m^k-1}{m-1}$, as detailed in Lemma~\ref{lem:Mk}; $e_{max}=(n-k)k+\binom{k}{2}\le nk$; the value $\mathcal{U}_k$ is defined in Definition~\ref{def:Uk}.}
\begin{footnotesize}
\begin{tabular}{|c|c|c|c|}
\hline class & VCD & TD & RTD \\
\hline $\mathcal{C}_{ac}^k$& $\geq (m-1)\mathcal{M}_k$& $n(m-1)\mathcal{U}_k\le \TD \le e_{max}+n(m-1)\mathcal{U}_k$ & $(m-1)\mathcal{M}_k$\\
\hline $\mathcal{C}_{ac}^{n-1}$& $m^n-1$& $n(m-1)m^{n-1}$ & $m^{n}-1$\\
\hline $\mathcal{C}_{ac}^{0}$&$(m-1)n$ &$(m-1)n$ &$(m-1)n$\\
\hline
\end{tabular}
\end{footnotesize}
\label{resultsSummary}
\end{table*}

\subsection{VC Dimension}

We begin by studying the VC dimension, with the following main result.




\begin{restatable}{theorem}{thmVCDac}
For fixed $n\ge 1$, $m\ge 2$, and any given $k\le n-1$, the following statements hold for the VC dimension of the class $\mathcal{C}_{ac}^k$ of all complete acyclic $k$-bounded CP-nets over $n$ variables with $m$ values each, over the instance space $\mathcal{X}_{swap}$.
\begin{enumerate}
\item $\VCdim(\mathcal{C}_{ac}^{n-1})=m^n-1$.
\item $\VCdim(\mathcal{C}_{ac}^0)=(m-1)n$.
\item $\VCdim(\mathcal{C}_{ac}^k)\geq(m-1)\mathcal{M}_k = (m-1)(n-k)m^k+m^k-1$. 
\end{enumerate}
\label{thm:VCDac}
\end{restatable}

The proof of Theorem~\ref{thm:VCDac} relies on decomposing $\mathcal{C}_{ac}^k$ as a product of concept classes over subsets of $\mathcal{X}_{swap}$.\footnote{This kind of decomposition is a standard technique; the reader may look at \cite[Lemma 16]{DFSZ14} for an example of its usage in the computational learning theory literature.}

\begin{definition}
	Let $\mathcal{C}_i\subseteq 2^{\mathcal{X}_i}$ and $\mathcal{C}_j\subseteq 2^{\mathcal{X}_j}$ be concept classes with $\mathcal{X}_i\cap \mathcal{X}_j=\emptyset$. The concept class $\mathcal{C}_{i}\times \mathcal{C}_j\subseteq 2^{\mathcal{X}_i\cup \mathcal{X}_j}$ is defined by $\mathcal{C}_{i}\times \mathcal{C}_j=\{c_i\cup c_j\mid c_i\in \mathcal{C}_i \mbox{ and } c_j\in \mathcal{C}_j\}$. For concept classes $\mathcal{C}_1$, \dots, $\mathcal{C}_r$, we define $\prod_{i=1}^r \mathcal{C}_i = \mathcal{C}_1 \times \cdots \times \mathcal{C}_r = (\cdots((\mathcal{C}_1\times
	\mathcal{C}_2)\times \mathcal{C}_3)\times\cdots\times \mathcal{C}_r)$.
\end{definition}

It is a well-known obvious fact that $\VCdim(\prod\limits_{i=1}^{t}\mathcal{C}_i)=\sum\limits_{i=1}^{t}\VCdim(\mathcal{C}_i)$, see, e.g., \cite[Lemma 16]{DFSZ14}.

For any $v_i\in V$ and any $\Gamma\subseteq V\setminus\{v_i\}$, we define $\mathcal{C}^\Gamma_{\CPT(v_i)}$ to be the concept class consisting of all preference relations corresponding to some $\CPT(v_i)$ where $Pa(v_i)=\Gamma$ and $|\Gamma|\leq k$; here the instance space $\mathcal{X}$ is the set of all swap pairs $x\in\mathcal{X}_{swap}$ with $V(x)=v_i$. Now, if we fix the context of $v_i$ by fixing an assignment $\gamma\in \mathcal{O}_\Gamma$ of all variables in $\Gamma$, we obtain a concept class $\mathcal{C}^\Gamma_{\succ^{v_i}_\gamma}$, which corresponds to the set of all preference statements concerning the variable $v_i$ conditioned on the context $\gamma$. Its instance space is the set of all swaps $x$ with $V(x)=v_i$ and $x[\Gamma]=\gamma$. 

Recall that $V=\{v_1,\ldots,v_n\}$. By $S_n$ we denote the class of all permutations of $\{1,\ldots,n\}$. Using this notation, we first establish two lemmas before proving Theorem~\ref{thm:VCDac}.

\begin{lemma}\label{lem:decompositionac}
	Under the same premises as in Theorem~\ref{thm:VCDac}, we obtain 
	\[\mathcal{C}_{ac}^k=\bigcup\limits_{\sigma\in S_n}\prod\limits_{i=1}^{n}\bigcup\limits_{\Gamma\subseteq \{v_{\sigma(1)},\ldots,v_{\sigma(i-1)}\}, |\Gamma|\le k}\prod\limits_{\gamma\in\mathcal{O}_\Gamma}\mathcal{C}^\Gamma_{\succ^{v_{\sigma(i)}}_\gamma}\,.\]
\end{lemma}

\begin{proof}
	By definition, for $v\in V$ and $\Gamma\subseteq V\setminus\{v\}$, the class $\mathcal{C}^\Gamma_{\CPT(v)}$ equals $\prod_{\gamma\in\mathcal{O}_\Gamma}\mathcal{C}^\Gamma_{\succ^{v}_\gamma}$.\footnote{Any concept representing a preference table for $v$ with $Pa(v)=\Gamma$ corresponds to a union of concepts each of which represents a preference statement over $D_{v}$ conditioned on some context $\gamma\in\mathcal{O}_\Gamma$.}
	
	Any concept corresponds to choosing a set $\Gamma_v$ of parent variables of size at most $k$ for each variable $v$, which means $\mathcal{C}_{ac}^k\subseteq\prod_{i=1}^{n}\bigcup_{\Gamma\subseteq V\setminus\{v_i\}, |\Gamma|\le k}\mathcal{C}^\Gamma_{\CPT(v_i)}$. By acyclicity, $v_j\in Pa(v_i)$ implies $v_i\notin Pa(v_j)$, so that for each concept $c\in\mathcal{C}_{ac}^k$ some $\sigma\in S_n$ fulfills $$c\in\prod_{i=1}^{n}\bigcup_{\Gamma\subseteq \{v_{\sigma(1)},\ldots,v_{\sigma(i-1)}\}, |\Gamma|\le k}\mathcal{C}^\Gamma_{\CPT(v_{\sigma(i)})}\,.$$
	
	\noindent Thus, $\mathcal{C}_{ac}^k\subseteq\bigcup_{\sigma\in S_n}\prod_{i=1}^{n}\bigcup_{\Gamma\subseteq \{v_{\sigma(1)},\ldots,v_{\sigma(i-1)}\}, |\Gamma|\le k}\prod_{\gamma\in\mathcal{O}_\Gamma}\mathcal{C}^\Gamma_{\succ^{v_{\sigma(i)}}_\gamma}$. 
	
	Similarly, one can argue that every concept in the class on the right hand side represents an acyclic CP-net with parent sets of size at most $k$. With $\mathcal{C}^\Gamma_{\CPT(v_i)}=\prod_{\gamma\in\mathcal{O}_\Gamma}\mathcal{C}^\Gamma_{\succ^{v_i}_\gamma}$, the statement of the lemma follows.
\end{proof}

\begin{lemma}\label{lem:VCDfixedcontext}
	Using the notation introduced before Lemma~\ref{lem:decompositionac}, we obtain $\VCdim(\mathcal{C}^\Gamma_{\succ^{v_i}_\gamma}) =m-1$ for any $v_i\in V$, $\Gamma\subseteq V\setminus\{v_i\}$, and $\gamma\in\mathcal{O}_\Gamma$.
\end{lemma}

\begin{proof}
	Let $v_i\in V$, $\Gamma\subseteq V\setminus\{v_i\}$, and $\gamma\in\mathcal{O}_\Gamma$. By definition, $\mathcal{C}^\Gamma_{\succ^{v_i}_\gamma}$ is the class of all total orders over the domain $D_{v_i}$ of $v_i$. We show that $\mathcal{C}^\Gamma_{\succ^{v_i}_\gamma}$ shatters some set of size $m-1$, but no set of size $m$. 
	
	To show $\VCdim(\mathcal{C}^\Gamma_{\succ^{v_i}_\gamma}) \ge m-1$, choose any set of $m-1$ swaps over $\Gamma\cup\{v_i\}$ with fixed context $\gamma$, in which the pairs of swapped values in $v_i$ are $(v^i_1,v^i_2)$,\dots, $(v^i_{m-1},v^i_m)$. 
	
	Fix any set $S\subseteq\mathcal{X}_{swap}$ of $m$ swaps over $\Gamma\cup\{v_i\}$ with fixed context $\gamma$. To show that $S$ is not shattered, consider the undirected graph $G$ with vertex set $D_{v_i}$ in which an edge between $v^i_r$ and $v^i_s$ exists iff $S$ contains a swap pair flipping $v^i_r$ to $v^i_s$ or vice versa. $G$ has $m$ vertices and $m$ edges and thus contains a cycle. The directed versions of $G$ correspond to the labelings of $S$; therefore some labeling $\ell$ of $S$ corresponds to a cyclic directed version of $G$, which does not induce a total order over $D_{v_i}$. Hence the labeling $\ell$ is not realized by $\mathcal{C}^\Gamma_{\succ^{v_i}_\gamma}$, so that $S$ is not shattered by $\mathcal{C}^\Gamma_{\succ^{v_i}_\gamma}$.
\end{proof}

These technical observations will help to establish a lower bound on the VC dimension of classes of bounded acyclic CP-nets. Our upper bound relies on the following straightforward generalization of an observation made by Booth et al.~\cite{Booth:2010:LCL:1860967.1861021} (their Proposition 3), which states that any concept class corresponding to a set of transitive and irreflexive relations (such as a class of acyclic CP-nets) over $\{0,1\}^n$ has a VC dimension no larger than $2^n-1$. We generalize this statement to relations over $\{0,\ldots,m\}^n$ and formulate it in CP-net terminology. Note that our statement applies to any class of consistent complete CP-nets, whether acyclic or not.

\begin{lemma}[based on \cite{Booth:2010:LCL:1860967.1861021}]\label{lem:booth}
Let $n\ge 1$, $m\ge 2$. Let $X\subseteq\mathcal{X}_{swap}$ with $|X|=m^n$ be any set of $m^n$ swaps over $n$ variables of domain size $m$. Let $\mathcal{C}$ be any class of consistent complete CP-nets over the same domain $\mathcal{X}_{swap}$. Then $\mathcal{C}$ does not shatter $X$. In particular, $\VCdim(\mathcal{C})\le m^n-1$.
\end{lemma}
\begin{proof}
Suppose $\mathcal{C}$ shatters $X$. Then each labeling of $X$ is consistent with some CP-net in $\mathcal{C}$. 

Observe first that the induced preference graph of any CP-net in $\mathcal{C}$ has exactly $m^n$ vertices. Each labeling of a swap corresponds to a directed edge in the induced preference graph of a CP-net.\footnote{For example, in Figure~\ref{fig:cyclicConsistent}, the edge from $a\bar{b}\bar{c}$ to $ab\bar{c}$ corresponds to the labeling $0$ for the swap $(a\bar{b}\bar{c},ab\bar{c})$, and, likewise, to the labeling $1$ for the swap $(ab\bar{c},a\bar{b}\bar{c})$.} Therefore, each labeling of all elements in $X$ corresponds to a set of $m^n$ directed edges in the induced preference graph. When ignoring the direction of the edges, the edge sets resulting from these labelings are all identical, and we obtain a set of $m^n$ undirected edges over a set of exactly $m^n$ vertices. Consequently, these edges form at least one cycle. Obviously, there exists a labeling $\ell$ of the elements in $X$ that will turn this undirected cycle into a directed cycle. Since $\mathcal{C}$ shatters $X$, it contains a concept $c$ that is consistent with the labeling $\ell$, i.e., whose induced preference graph contains the directed edges corresponding to the labeling $\ell$. In particular, this concept $c$ has a cycle in its induced preference graph. This contradicts the fact that $c$ is a consistent CP-net.

Therefore, $\mathcal{C}$ does not shatter $X$. Since $X$ was chosen arbitrarily, $\mathcal{C}$ shatters no set of size $m^n$, so that $\VCdim(\mathcal{C})\le m^n-1$.
\end{proof}

We are now ready to give a proof of Theorem~\ref{thm:VCDac}, which we restate here for convenience.

\thmVCDac*

\begin{proof}	Lemma~\ref{lem:decompositionac} states that 
	$$
	\mathcal{C}_{ac}^k=\bigcup\limits_{\sigma\in S_n}\prod\limits_{i=1}^{n}\bigcup\limits_{\Gamma\subseteq \{v_{\sigma(1)},\ldots,v_{\sigma(i-1)}\}, |\Gamma|\le k}\prod\limits_{\gamma\in\mathcal{O}_\Gamma}\mathcal{C}^\Gamma_{\succ^{v_{\sigma(i)}}_\gamma}\,, $$

 which yields the bound
  $$\VCdim(\mathcal{C}_{ac}^k)\ge\max\limits_{\sigma\in S_n}\sum\limits_{i=1}^{n}\max\limits_{\Gamma\subseteq \{v_{\sigma(1)},\ldots,v_{\sigma(i-1)}\},|\Gamma|\le k}\sum\limits_{\gamma\in\mathcal{O}_\Gamma}\VCdim(\mathcal{C}^\Gamma_{\succ^{v_{\sigma(i)}}_\gamma})\,.$$ 
	By Lemma~\ref{lem:VCDfixedcontext}, we have $\VCdim(\mathcal{C}^\Gamma_{\succ^{v_{\sigma(i)}}_\gamma}) =m-1$, independent of $\Gamma$ and $\gamma$, so that one obtains, for any $\sigma\in S_n$,
	\begin{eqnarray*}
		\VCdim(\mathcal{C}_{ac}^k)&\ge& (m-1)\sum_{i=1}^n\max_{\Gamma\subseteq\{v_{\sigma(1)},\ldots,v_{\sigma(i-1)}\}, |\Gamma|\le k}|\mathcal{O}_\Gamma|\\
		&=&(m-1)\sum_{i=1}^n\max_{\Gamma\subseteq\{v_{\sigma(1)},\ldots,v_{\sigma(i-1)}\},|\Gamma|\le k}m^{|\Gamma|}\\
		&=&(m-1)\mathcal{M}_k\,.
	\end{eqnarray*}
	It remains to verify $\VCdim(\mathcal{C}_{ac}^k)\le (m-1)\mathcal{M}_k$ for $k\in\{0,n-1\}$. 
	
	For $k=0$, we have $\mathcal{M}_k=n$, so let us consider any set $Y$ of size greater than $(m-1)n$ and argue why $Y$ cannot be shattered by $\mathcal{C}_{ac}^0$. Clearly, there exists a variable $v_i$ that is swapped in at least $m$ instances in $Y$. In order to shatter these $\ge m$ instances with the same swapped variable, a concept class of CP-nets would need to contain CP-nets in which some variables have non-empty parent sets, which is not the case for $\mathcal{C}_{ac}^0$. Thus $Y$ is not shattered, i.e., $\VCdim(\mathcal{C}_{ac}^k)\le (m-1)\mathcal{M}_k$.
	
	For $k=n-1$, the upper bound $\VCdim(\mathcal{C}_{ac}^k)\le m^n-1 = (m-1)\mathcal{M}_{n-1}$ follows immediately from Lemma~\ref{lem:booth}.
	\end{proof}

\subsection{Recursive Teaching Dimension}

For studying teaching complexity, it is useful to identify concepts that are ``easy to teach.'' To this end, we use the notion of subsumption~\cite{Koriche2010685}: given CP-nets $N,N'$, we say $N$ subsumes $N'$ if for all $v_i\in V$ the following holds: If $y_1\succ y_2$ is specified in $\CPT(v_i)$ in $N'$ for some context $\gamma'$, then $y_1\succ y_2$ is specified in $\CPT(v_i)$ in $N$ for some context containing $\gamma'$. If in addition $N\ne N'$, we say that $N$ strictly subsumes $N'$.

Now let $\mathcal{C}\subseteq\mathcal{C}_{ac}^k$. A concept $c\in \mathcal{C}$ is maximal in $\mathcal{C}$ if no $c'\in\mathcal{C}$ strictly subsumes $c$. 
The size of maximal concepts in $\mathcal{C}_{ac}^k$ equals $\mathcal{M}_k$, by definition.
\begin{figure}
\begin{adjustbox}{max width=\textwidth}
\begin{tikzpicture}[thick,every node/.style={scale=0.7}]
\begin{scope}
\node[state](A){A};
\node[state](B)[below right=of A]{B};
\node[state](C)[below left=of A]{C};
\node (aCPT)[above=of A,yshift=-1.3cm]{$a\succ\bar{a}$};
\node(bCPT)[above=of B,yshift=-1cm]{\makecell[l]{$a:b\succ \bar{b}$\\$\bar{a}:\bar{b}\succ b$}};
\node(cCPT)[above=of C,xshift=-1cm,yshift=-.7cm]{\makecell[l]{$\bar{a}\bar{b}\qquad :\bar{c}\succ c$\\otherwise: $c\succ \bar{c}$}};
\node(text)[below of=A,yshift=-2.5cm]{$N_1$};
\path (A) edge[->] (B) edge[->] (C)
(B) edge[->] (C);
\end{scope}

\begin{scope}[xshift=5cm]
\node[state](A){A};
\node[state](B)[below right=of A]{B};
\node[state](C)[below left=of A]{C};
\node (aCPT)[above=of A,yshift=-1.3cm]{$a\succ\bar{a}$};
\node(bCPT)[above=of B,yshift=-1cm]{\makecell[l]{$a:b\succ \bar{b}$\\$\bar{a}:\bar{b}\succ b$}};
\node(cCPT)[above=of C,yshift=-.6cm]{\makecell[l]{$a:c\succ \bar{c}$\\ $\bar{a}:\bar{c}\succ \bar{c}$}};
\path (A) edge[->] (B) edge[->] (C);
\node(text)[below of=A,yshift=-2.5cm]{$N_2$};
\end{scope}

\begin{scope}[xshift=9cm]
\node[state](A){A};
\node[state](B)[below right=of A]{B};
\node[state](C)[below left=of A,xshift=.6cm]{C};
\node (aCPT)[above=of A,yshift=-1.3cm]{$a\succ\bar{a}$};
\node(bCPT)[above=of B,yshift=-1cm]{\makecell[l]{$a:b\succ \bar{b}$\\$\bar{a}:\bar{b}\succ b$}};
\node(cCPT)[above=of C,yshift=-1.3cm]{$c\succ \bar{c}$};
\node(text)[below of=A,yshift=-2.5cm]{$N_3$};
\path (A) edge[->] (B) ;
\end{scope}
\end{tikzpicture}
\end{adjustbox}
\caption{Three networks each of which is subsumed by the ones to its left.}
\label{fig:subsumeExample}
\end{figure}
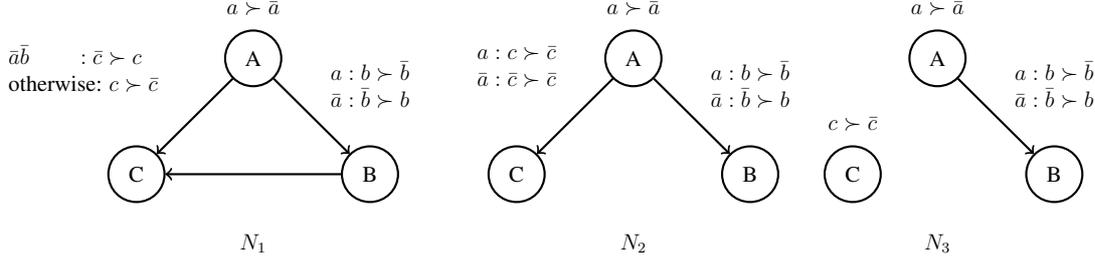

\begin{example}
Let $\mathcal{C}_{ac}^2$ be the class of all unbounded complete acyclic CP-nets over three variables $V=\{A,B,C\}$. Consider the three CP-nets in Figure \ref{fig:subsumeExample}. Clearly, $N_3$ is subsumed by $N_2$ and $N_1$, and $N_2$ is subsumed by $N_1$. Further, $N_1$ is maximal with respect to $\mathcal{C}_{ac}^2$ and thus also maximal with respect to $\overline{\mathcal{C}}_{ac}^2$. Each of the three CP-nets also strictly subsumes any CP-net resulting from removal of some of its statements. Since subsumption is transitive, the CP-net $N_1$ strictly subsumes any CP-net resulting from $N_2$ after removal of statements.
\end{example}

The following two lemmas formalize the intuition that maximal concepts are ``easy to teach.'' Lemma~\ref{lem:RTDupperBound} gives an upper bound on the size of a smallest teaching set of any maximal concept, while Lemma~\ref{lem:nonmaximal} implies that a maximal concept is never harder to teach than any concept it subsumes. 

\begin{lemma}\label{lem:RTDupperBound}
Let $n\ge 1$, $m\ge 2$, and $k\le n-1$. Let $\mathcal{C}\subseteq\mathcal{C}_{ac}^k$ be a subclass of the class of all complete acyclic $k$-bounded CP-nets over $n$ variables of domain size $m$.
For any maximal concept $c$ in $\mathcal{C}$, we have $\TD(c,\mathcal{C})\leq (m-1)size(c)$, i.e., $(m-1)size(c)$ swap examples suffice to distinguish $c$ from any other concept in $\mathcal{C}$.
\end{lemma}
\begin{proof}
Every statement in the CP-net $N$ represented by $c$ corresponds to an order of $m$ values for some variable $v_i$ under a fixed context $\gamma$. For every such order $y_1\succ_\gamma^{v_i}\ldots\succ_\gamma^{v_i}y_m$, we include $m-1$ positively labeled swap examples in a set $T$. For $1\le j\le m-1$, the $j$th such example labels a pair $x=(x.1,x.2)$ of swap outcomes with $V(x)=v_i$, the projection of $x$ onto $\{v_i\}$ is $(y_j,y_{j+1})$, and the projection of $x$ onto the remaining variables contains $\gamma$. The set $T$ then has cardinality $(m-1)size(c)$ and is obviously consistent with $N$.

It remains to show that no other CP-net in $\mathcal{C}$ is consistent with $T$. Suppose some CP-net $c'\in\mathcal{C}$ is consistent with $T$. Note that, for each $i$ and each context $\gamma$ occurring in $\CPT(v_i)$ in $c$, the information in $T$ determines a strict total order on $D_{v_i}$. The corresponding preferences are all entailed by $c'$ as well, since $c'$ is consistent with $T$. Therefore, $c'$ subsumes $c$. Since $c$ is a maximal concept, this implies that $c'$ does not strictly subsume $c$. Hence, $c'=c$, i.e., $c$ is the only concept in $\mathcal{C}$ that is consistent with $T$.
\end{proof}


Next, we determine the teaching dimension of maximal concepts in the class $\mathcal{C}^k_{ac}$.

\begin{lemma}\label{lem:maximal} Let $n\ge 1$, $m\ge 2$, and $k\le n-1$. Further, let $c$ be a maximal concept in the class $\mathcal{C}_{ac}^k$ of all complete acyclic $k$-bounded CP-nets over $n$ variables of domain size $m$. Then 
$\TD(c,\mathcal{C}_{ac}^k) = (m-1) \mathcal{M}_k=(m-1)(n-k)m^k+m^k-1$.
\end{lemma}

\begin{proof}
A teaching set for a maximal concept must contain $m-1$ examples for each of the $\mathcal{M}_k$ statements in its $\CPT$s, so as to determine the preferences for each context. This yields $\TD(c,\overline{\mathcal{C}}_{ac}^k)=\TD(c,\mathcal{C}_{ac}^k) \ge (m-1) \mathcal{M}_k$. Lemma~\ref{lem:RTDupperBound} then completes the proof.
\end{proof}

The following lemma then shows that no concept $c'$ in $\mathcal{C}^k_{ac}$ has a smaller teaching dimension than any maximal concept in $\mathcal{C}^k_{ac}$ that subsumes $c'$.

\begin{lemma}\label{lem:nonmaximal} Let $n\ge 1$, $m\ge 2$, and $k\le n-1$, and consider the class $\mathcal{C}_{ac}^k$ of all complete acyclic $k$-bounded CP-nets over $n$ variables of domain size $m$. Then each non-maximal $c'\in\mathcal{C}_{ac}^k$ is strictly subsumed by some $c\in\mathcal{C}_{ac}^k$ such that $\TD(c',\mathcal{C}_{ac}^k)\!\geq\! \TD(c,\mathcal{C}_{ac}^k)$. In particular, if $c$ is any maximal concept in $\mathcal{C}_{ac}^k$, then $\TD(c'',\mathcal{C}_{ac}^k)\!\geq\! \TD(c,\mathcal{C}_{ac}^k)$ for each $c''\in \mathcal{C}_{ac}^k$ subsumed by $c$. 
\label{lem:subsumeTD}
\end{lemma}

\begin{proof}
From the graph $G'$ for $c'$, we build a graph $G$ by adding the maximum possible number of edges to a single variable $v$. As $c'$ is not maximal, it is possible to add at least one edge without violating the indegree bound $k$. The CP-nets corresponding to $G$ and $G'$ differ only in $\CPT(v)$. Let $c$ be the concept representing $G$ and $z$ be the size of its $\CPT$ for $v$. A smallest teaching set $T'$ for $c'$ can be modified to a teaching set for $c$ by replacing only those examples that refer to the swapped variable $v$; $(m-1)z$ examples suffice. To distinguish $c'$ from $c$, $T'$ must contain at least $(m-1)z$ examples referring to the swapped variable $v$ ($m-1$ for each context in $\CPT(v)$ in $c$). Hence $\TD(c',\overline{\mathcal{C}}_{ac}^k)\geq$ $\TD(c,\overline{\mathcal{C}}_{ac}^k)$.
\end{proof}

As a consequence, the easiest to teach concepts in $\mathcal{C}_{ac}^k$ are the maximal ones.
Using these lemmas, one can determine the recursive teaching dimension of the concept classes $\mathcal{C}_{ac}^k$, for any $k$.

\begin{theorem}\label{thm:RTDac} Let $n\ge 1$, $m\ge 2$, and $k\le n-1$. For the recursive teaching dimension of the class $\mathcal{C}_{ac}^k$ of all complete acyclic $k$-bounded CP-nets over $n$ variables of domain size $m$, we obtain:
\[\RTD(\mathcal{C}_{ac}^k)=(m-1)\mathcal{M}_k=(m-1)(n-k)m^k+m^k-1\,.\]
As before, here $\mathcal{C}_{ac}^k$ is defined over the instance space $\mathcal{X}_{swap}$.
\end{theorem}

\begin{proof}
The upper bound on the recursive teaching dimension follows from Lemma~\ref{lem:maximal} and then repeated application of Lemma~\ref{lem:RTDupperBound}. The lower bound follows from Lemma~\ref{lem:nonmaximal} in combination with Lemma~\ref{lem:maximal}---these two lemmas state that the smallest teaching dimension of any concept in $\mathcal{C}_{ac}^k$ is that of any maximal concept, and that this teaching dimension value equals $(m-1)\mathcal{M}_k$.
\end{proof}

\subsubsection{An Example Illustrating Teaching Sets}
For the case of complete CP-nets, we will now illustrate our results on teaching sets through examples. Let us consider the case where $V=\{A,B,C\}$, $D_A=\{a,\bar{a}\}$, $D_B=\{b,\bar{b}\}$, and $D_C=\{c,\bar{c}\}$. We consider the class $\mathcal{C}_{ac}^{n-1}$ of all complete acyclic CP-nets defined over $V$. Figure \ref{fig:subsumeExample} shows three concepts, $N_1$, $N_2$, and $N_3$, from this class.

In order to illustrate Lemma \ref{lem:RTDupperBound}, let us first compute an upper bound on the teaching dimension of the maximal concept $N_1$. The claim is that $\TD(N_1,\mathcal{C}_{ac}^{n-1})$ is less than or equal to the size of $N_1$, which is 7. Consider a set of entailments $\mathcal{E}$ corresponding to the teaching set $T$ as described in the proof of Lemma \ref{lem:RTDupperBound}. One possibility for $\mathcal{E}$ is the set consisting of the entailments $abc\succ \bar{a}bc$, $abc\succ a\bar{b}c$, $\bar{a}\bar{b}\bar{c}\succ\bar{a}b\bar{c}$, $abc\succ ab\bar{c}$, $a\bar{b}c\succ a\bar{b}\bar{c}$, $\bar{a}bc\succ \bar{a}b\bar{c}$, and $\bar{a}\bar{b}\bar{c}\succ \bar{a}\bar{b}c$. $\mathcal{E}$ is obviously consistent with $N_1$. A closer look shows that it is also a teaching set for $N_1$ with respect to $\mathcal{C}_{ac}^{n-1}$. To see why, consider the partition $\mathcal{E}= \mathcal{E}_A\cup\mathcal{E}_B\cup\mathcal{E}_C$, where
\begin{itemize}
\item $\mathcal{E}_A=\{abc\succ \bar{a}bc\}$,
\item $\mathcal{E}_B=\{abc\succ a\bar{b}c, \bar{a}\bar{b}\bar{c}\succ \bar{a}b\bar{c}\}$, and
\item $\mathcal{E}_C=\{abc\succ ab\bar{c}, a\bar{b}c\succ a\bar{b}\bar{c}, \bar{a}bc\succ \bar{a}b\bar{c}, \bar{a}\bar{b}\bar{c}\succ\bar{a}\bar{b}c\}$.
\end{itemize} 
The set $\mathcal{E}_C$ shows that $C$ has at least two parents, which, in our case, have to be $A$ and $B$. As a result, $Pa(C)=\{A,B\}$ and $\CPT(C)$ is defined precisely: it is the $\CPT$ with the maximum possible parent set $\{A,B\}$ and $\mathcal{E}_C$ contains a statement for every context over this parent set. Similarly, $\mathcal{E}_{B}$ shows that there must be at least one parent for $B$. Given the fact that $B\in Pa(C)$ we conclude that $Pa(B)=\{A\}$ as there is no other way to explain $\mathcal{E}_C$ and $\mathcal{E}_B$ together. Therefore, we have identified $\CPT(B)$ precisely. In an acyclic CP-net, thus $\CPT(A)$ must be unconditional (which is consistent with $\mathcal{E}_A$ having one entailment only). Thus, $\mathcal{E}$ is a teaching set for $N_1$ and $\TD(N_1,\mathcal{C}_{ac}^{n-1})\leq size(N_1)=7$.

Now, let us illustrate Lemma \ref{lem:subsumeTD}. For this purpose, consider the teaching dimension of $N_3$ which is not maximal w.r.t.~$\mathcal{C}_{ac}^{n-1}$. According to Lemma \ref{lem:subsumeTD}, there exists a concept $N$ that subsumes $N_3$ and for which $\TD(N_3,\mathcal{C}_{ac}^{n-1})\geq \TD(N,\mathcal{C}_{ac}^{n-1})$. We follow the proof of Lemma~\ref{lem:subsumeTD}. If $G'$ is the graph of $N_3$, we construct a new graph $G$ as follows: $G$ results from $G'$ by selecting the variable $C$ and adding two incoming edges (i.e., the maximum possible number of edges) to the node labeled by this variable. A concept with such graph is $N_1$. Let $\mathcal{E}$ be a set of entailments corresponding to a teaching set for $N_3$ with size equal to $\TD(N_3,\mathcal{C}_{ac}^{n-1})$. We claim that the number of entailments in $\mathcal{E}$ whose swapped variable is $C$ has to be greater than or equal $4$. To see this, consider the CP-net $N$ with graph $G$ where $\CPT(A)$ and $\CPT(B)$ are identical to the corresponding conditional preference tables in $N_3$. Moreover, for $\CPT(C)$ in $N$, for every entailment in $e\in \mathcal{E}$ whose swapped variable is $C$, a statement $u:c\succ \bar{c}$ or $u:\bar{c}\succ c$ is created in agreement with $e$, where $u$ provides the values of $A$ and $B$ in $e$. For instance if the subset of $\mathcal{E}$ referring to swapped variable $C$ is $\{abc\succ ab\bar{c} , \bar{a}bc\succ \bar{a}b\bar{c}\}$, then one possible table for $\CPT(C)$ in $N$ is $\{b:c\succ \bar{c}, \bar{b}:\bar{c}\succ c\}$.\footnote{Alternatively, the $\CPT$ for $C$ in $N_1$ can also be used.}

Note that $\TD(N_3,\mathcal{C}_{ac}^{n-1})\le \TD(\mathcal{C}_{ac}^{n-1})$.  By Table~\ref{resultsSummary}, $\TD(\mathcal{C}_{ac}^{n-1})$ equals $(m-1)nm^{n-1}$ which, for $m=2$ and $n=3$ evaluates to 12. 



In the sum, we discussed why $\TD(N_1,\mathcal{C}_{ac}^{n-1})\le 7$ and $\TD(N_1,\mathcal{C}_{ac}^{n-1})\le \TD(N_3,\mathcal{C}_{ac}^{n-1})\le 12$. It is actually the case that $\TD(N_1,\mathcal{C}_{ac}^{n-1})=7$, $\TD(N_2,\mathcal{C}_{ac}^{n-1})=9$ and $\TD(N_3,\mathcal{C}_{ac}^{n-1})=10$. Figure \ref{fig:teachingSets} shows one possible example of minimum teaching sets for the three networks. In the displayed choice of teaching sets, we selected 10 instances for teaching $N_3$, use nine of them (some with flipped labels) to teach $N_2$ and seven of them (some with flipped labels) to teach $N_1$. While this is not the only choice of smallest possible teaching sets for these concepts, it illustrates that some subset of the instances that are used to teach some concept, with appropriate labeling, can also be used to teach a concept that subsumes it. In Figure~\ref{fig:teachingSets}, we show which instances are removed from the teaching set of $N_3$ by crossing them out. For instance, the entailment $ab\bar{c}\succ \bar{a}b\bar{c}$ is not needed in the teaching set of $N_2$ as the remaining examples provide enough evidence to conclude that $A\in Pa(B)$ and $A\in Pa(C)$; thus one entailment for $\CPT(A)$ suffices for teaching $\CPT(A)$.  


\begin{figure}
\centering
\begin{tikzpicture}
\node(text1)[xshift=-4cm]{$N_1$};
\node(text2)[right=of text1,xshift=1cm]{$N_2$};
\node(text3)[right=of text2,xshift=1cm]{$N_3$};
\node(N1)[above=of text1]{\makecell[c]{
$abc\succ \bar{a}bc$\\
$\xcancel{ab\bar{c}\succ \bar{a}b\bar{c}}$\\
$abc\succ a\bar{b}c$\\ 
$\xcancel{ab\bar{c}\succ a\bar{b}\bar{c}}$\\
$\bar{a}\bar{b}\bar{c}\succ \bar{a}b\bar{c}$\\
$\xcancel{\bar{a}\bar{b}c \succ \bar{a}bc}$\\
$abc\succ ab\bar{c}$\\ 
$a\bar{b}c\succ a\bar{b}\bar{c}$\\
$\bar{a}\bar{b}\bar{c}\succ \bar{a}\bar{b}c$ \\ 
$\bar{a}bc\succ \bar{a}b\bar{c}$
}};
\node(N2)[above=of text2]{\makecell[c]{
	$abc\succ \bar{a}bc$\\
	$\xcancel{ab\bar{c}\succ \bar{a}b\bar{c}}$\\
	$abc\succ a\bar{b}c$\\ 
	$ab\bar{c}\succ a\bar{b}\bar{c}$\\
	$\bar{a}\bar{b}\bar{c} \succ \bar{a}b\bar{c}$\\
	$\bar{a}\bar{b}c\succ \bar{a}bc$\\
	$abc\succ ab\bar{c}$\\ 
	$a\bar{b}c\succ a\bar{b}\bar{c}$\\
	$\bar{a}\bar{b}\bar{c}\succ \bar{a}\bar{b}c $ \\ 
	\begin{boldmath}$\bar{a}b\bar{c}\succ\bar{a}bc$\end{boldmath}
}};

\node(N3)[above=of text3]{\makecell[c]{
		$abc\succ \bar{a}bc$\\
		$ab\bar{c}\succ \bar{a}b\bar{c}$\\
		$abc\succ a\bar{b}c$\\ 
		$ab\bar{c}\succ a\bar{b}\bar{c}$\\
		$\bar{a}\bar{b}\bar{c}\succ \bar{a}b\bar{c}$\\
		$\bar{a}\bar{b}c\succ\bar{a}bc$\\
		$abc\succ ab\bar{c}$\\ 
		$a\bar{b}c\succ a\bar{b}\bar{c}$\\
		\begin{boldmath}$\bar{a}\bar{b}c\succ\bar{a}\bar{b}\bar{c}$\end{boldmath} \\ 
		$\bar{a}bc\succ \bar{a}b\bar{c}$
	}};
	
	\end{tikzpicture}
	\caption{Teaching sets for the three networks in Figure \ref{fig:subsumeExample} w.r.t.~the class $\mathcal{C}_{ac}^{n-1}$. }
	\label{fig:teachingSets}
\end{figure}

\subsubsection{Structural Properties of Interest to Learning-Theoretic Studies}

The class $\mathcal{C}_{ac}^{n-1}$ is interesting from a structural point of view, as its VC dimension equals its recursive teaching dimension, see Table~\ref{resultsSummary}. In general, the VC dimension can exceed the recursive teaching dimension by an arbitrary amount, and it can also be smaller than the recursive teaching dimension~\cite{DFSZ14}. Simon and Zilles~\cite{SimonZ15} posed the question whether the recursive teaching dimension can be upper-bounded by a function that is linear in the VC dimension. So far, the best known upper bound on the recursive teaching dimension is quadratic in the VC dimension~\cite{HuWLW17}. 

The computational learning theory literature knows of some structural properties under which the VC dimension and the recursive teaching dimension coincide. In Appendix\ref{sec:struct}, we show that none of these structural properties apply to the class of complete unbounded acyclic CP-nets.

\subsection{Teaching Dimension}

Computing the teaching dimension is somewhat more involved. We will introduce some terminology first. In a teaching set, the purpose of a \emph{conflict pair}\/ for a variable $v_i$ given a variable $v_j$ is to demonstrate that $v_j$ is a parent of $v_i$. 

\begin{definition}\label{def:conflict}
Let $n\ge 1$, $m\ge 2$. Let $N$ be a CP-net over $n$ variables of domain size $m$,  and let $v_i,v_j$ be variables in $N$. A conflict pair for $v_i$ given $v_j$ is a pair $(x,x')$ of swaps such that (i) $V(x)=V(x')=v_i$, (ii) $x.1$ and $x'.1$ differ only in $v_j$, (iii) $x.2$ and $x'.2$ differ only in $v_j$, and (iv) $N$ entails $x$ but not $x'$. 
\end{definition}

Obviously, the variable $v_j$ is a parent of the variable $v_i$ if and only if a conflict pair for $v_i$ given $v_j$ exists. For instance, in the CP-net $N_1$ in Figure~\ref{fig:subsumeExample}, the variable $C$ has a conflict pair given $A$, since $A$ is a parent of $C$. An example of a conflict pair here would be $((\bar{a}\bar{b}\bar{c},\bar{a}\bar{b}c)\,,(a\bar{b}\bar{c},a\bar{b}c))$. The CP-net $N_1$ entails $(\bar{a}\bar{b}\bar{c},\bar{a}\bar{b}c)$, but it does not entail $(a\bar{b}\bar{c},a\bar{b}c)$. In the same figure, in the CP-net $N_3$, the variable $C$ has no parents, so that each context over $A$ and $B$ results in the same preference for $C$. In other words, $C$ has no conflict pair given $A$ and no conflict pair given $B$. The variable $B$ in $N_3$ has a conflict pair given $A$ but none given $C$.

When a variable $v_i$ in a CP-net $c$ has the largest possible number of parents, conflict pairs for each parent are an efficient way of determining the parent set of $v_i$. However, to teach the learner the absence of parents, e.g., when the parent set of $v_i$ is empty, the previously introduced notion of universal set becomes relevant. The idea is to use swaps over $v_i$ in which the contexts form an $(m,n-1,k)$-universal set in order to show that $v_i$ has no parents at all. If however $v_i$ does have parents, then such sets of swaps can help determine the preference relations for those parents.

\begin{definition}
Let $n\ge 1$, $m\ge 2$, and $0\le k\le n-1$. Let $V=\{v_1,\ldots,v_n\}$ be a set of $n$ distinct variables, where each $v_i\in V$ has the domain $\{v^i_1,\ldots,v^i_m\}$ of size $m$. Fix any $(m,n-1,k)$-universal set $U$, and let $v_i\in V$. Then define the set $U_{v_i}$ of contexts over $V\setminus\{v_i\}$ as follows.
\[
U_{v_i}=\{(v^1_{j_1},\ldots,v^{i-1}_{j_{i-1}},v^{i+1}_{j_{i+1}},\ldots,v^n_{j_n})\mid (j_1,\ldots,j_{i-1},j_{i+1},\ldots,j_n)\in U\}\,.
\]
We call $U_{v_i}$ the context set imposed by $U$ and $v_i$. 

A set $\mathcal{F}\subseteq \mathcal{X}_{swap}$ of swaps is called a swap expression of $U_{v_i}$, if $|\mathcal{F}|=(m-1)|U_{v_i}|$ and, for each $(v^1_{j_1},\ldots,v^{i-1}_{j_{i-1}},v^{i+1}_{j_{i+1}},\ldots,v^n_{j_n})\in U_{v_i}$ there is an order $a_1\succ a_2\succ\ldots\succ a_m$ over the domain $D_{v_i}=\{a_1,\ldots,a_m\}$ such that 
\[((v^1_{j_1},\ldots,v^{i-1}_{j_{i-1}},a_t,v^{i+1}_{j_{i+1}},\ldots,v^n_{j_n}),\,(v^1_{j_1},\ldots,v^{i-1}_{j_{i-1}},a_{t+1},v^{i+1}_{j_{i+1}},\ldots,v^n_{j_n}))\in \mathcal{F}\]
for all $t$ with $1\le t\le m-1$.
\end{definition}

For example, consider $n=4$, $m=2$, and $k=2$, and four variables $A$, $B$, $C$, and $D$ with domains $\{a,\bar{a}\}$, $\{b,\bar{b}\}$, $\{c,\bar{c}\}$, and $\{d,\bar{d}\}$, respectively. A $(2,3,2)$-universal set of vectors is, e.g.,
\[
U=\{ (0,0,0)\,,(1,0,1)\,,(0,1,1)\,,(1,1,0)  \}\,,
\]
since any projection of these vectors onto two of their three components yields all four binary vectors of length 2. The context set $U_A$ imposed by $U$ and $A$ is then 
\[
U_A=\{bcd, \bar{b}c\bar{d}, b\bar{c}\bar{d}, \bar{b}\bar{c}d\}\,.
\]
Finally, a swap expression of $U_A$ is given as follows.
\[
\mathcal{F}=\{(abcd,\bar{a}bcd),(a\bar{b}c\bar{d},\bar{a}\bar{b}c\bar{d}),(ab\bar{c}\bar{d},\bar{a}b\bar{c}\bar{d}),(a\bar{b}\bar{c}d,\bar{a}\bar{b}\bar{c}d)\}\,.
\]
If $A$ had a third value $\bar{\bar{a}}$ in its domain, then $\mathcal{F}$ would be twice as large; for each swap $(ab^*c^*d^*,\bar{a}b^*c^*d^*)$ in $\mathcal{F}$, one would also include, for instance, $(\bar{a}b^*c^*d^*,\bar{\bar{a}}b^*c^*d^*)$.

\begin{lemma}\label{lem:TDuniversal} 
Let $n\ge 1$, $m\ge 2$, and $0\le k\le n-1$. Let $N$ be any $k$-bounded CP-net over $n$ variables of domain size $m$. Let $v_i\in V$ be any variable in $N$. 

The variable $v_i$ has at least one parent in $N$ if and only if, for every $(m,n-1,k)$-universal set $U$, the imposed context set $U_{v_i}$ contains two distinct contexts $\gamma,\gamma'\in\mathcal{O}_{V\setminus\{v_i\}}$ such that $N$ entails $\gamma:a\succ\bar{a}$ but not $\gamma':\bar{a}\succ a$, for some values $a,\bar{a}\in D_{v_i}$.
\end{lemma}

\begin{proof}
First, suppose $Pa(v_i)\ne\emptyset$ in $N$ and let $U$ be any $(m,n-1,k)$-universal set. Then there are $a,\bar{a}\in D_{v_i}$ and two distinct contexts $\alpha,\alpha'\in\mathcal{O}_{Pa(v_i)}$ such that $N$ entails $\alpha:a\succ\bar{a}$ but not $\alpha':\bar{a}\succ a$.\footnote{Recall that variables do not have dummy parents, i.e., every parent listed in the parent set of $v_i$ must affect the preference over $D_{v_i}$.} Since $U$ is an $(m,n-1,k)$-universal set and the length of $\alpha$ and $\alpha'$ is at most $k$, there are some contexts $\beta,\beta'\in\mathcal{O}_{V\setminus(Pa(v_i)\cup\{v_i\})}$, such that the contexts $\gamma,\gamma'\in\mathcal{O}_{V\setminus\{v_i\}}$ belong to $U_{v_i}$, where $\gamma,\gamma'$ result from $\alpha,\alpha'$ by extension with $\beta$ and $\beta'$, respectively. Clearly, $N$ entails $\gamma:a\succ\bar{a}$ but not $\gamma':\bar{a}\succ a$, since no variable over which $\beta$ and $\beta'$ are defined has an influence on the preference over $v_i$.

Second, suppose that there are two distinct contexts $\gamma,\gamma'\in\mathcal{O}_{V\setminus\{v_i\}}$ such that $N$ entails $\gamma:a\succ\bar{a}$ but not $\gamma':\bar{a}\succ a$. This means that the preference over $D_{v_i}$ is conditional, i.e., $v_i$ has a parent in $N$.
\end{proof}

The following lemma gives upper and lower bounds on the teaching dimension of any single concept in the class of complete acyclic $k$-bounded CP-nets.

\begin{lemma}\label{lem:TDbound} 
Let $n\ge 1$, $m\ge 2$, and $0\le k\le n-1$. Let $\mathcal{C}_{ac}^k$ be the class of all complete acyclic $k$-bounded CP-nets over $n$ variables of domain size $m$, over the instance space $\mathcal{X}_{swap}$. Fix $c\in \mathcal{C}_{ac}^k$ and let $e_c$ denote the number of edges in $c$. The teaching dimension of the concept $c$ with respect to the class $\mathcal{C}_{ac}^k$ is bounded as follows.
\[(m-1)\mathcal{M}_k\le \TD(c,\mathcal{C}_{ac}^k)\le e_c+n(m-1)\mathcal{U}_k\,.\]
\end{lemma}

\begin{proof}
The lower bound follows from Lemma~\ref{lem:nonmaximal} in combination with Lemma~\ref{lem:maximal}---these two lemmas state that the smallest teaching dimension of any concept in $\mathcal{C}_{ac}^k$ is that of any maximal concept, and that this teaching dimension value equals $(m-1)\mathcal{M}_k$.

For the upper bound, it suffices to show that $c$ has a teaching set of size $e_c+n(m-1)\mathcal{U}_k$. Consider a set $T$ of labeled examples defined in the following way. Fix any smallest $(m,n-1,k)$-universal set $U$. For each variable $v_i$ and each context $\gamma\in U_{v_i}$, the set $T$ includes $m-1$ examples determining the preference order over $D_{v_i}$ under the context $\gamma$. Note that these examples form a swap expression of $U_{v_i}$. In addition to these $n(m-1)\mathcal{U}_k$ examples, for each variable $v_i$ with $Pa(v_i)\ne\emptyset$ and for each $v\in Pa(v_i)$, the set $T$ contains two swaps $x,x'$ over the swapped variable $v_i$ that form a conflict pair for $v_i$ given $v$. By Definition~\ref{def:conflict}, the contexts in these two swaps are identical except for the variable $v$, which is assigned in a way that the two swaps display a difference in the preference over $D_{v_i}$. Any conflict pair $(x,x')$ for $v_i$ given $v$ will be sufficient for this purpose. In particular, we can choose $x$ and $x'$ in a way such that  $x$ is already contained in the swap expression of $U_{v_i}$ included in $T$. Thus, in total, we add one labeled swap for each edge in $c$ on top of the $n(m-1)\mathcal{U}_k$ examples from the swap expressions.

Clearly, the size of $T$ is $e_c+n(m-1)\mathcal{U}_k$. To show that $T$ is a teaching set for $c$, let $c'$ be any concept in $\mathcal{C}_{ac}^k$ that is consistent with $T$, and let $v_i$ be any variable. 

Let $P$ be the parent set of $v_i$ in $c$, and $P'$ the parent set of $v_i$ in $c'$. The conflict pairs in $T$ imply that $P\subseteq P'$. Now, for every context $\alpha$ over $P$, the swap expression of $U_{v_i}$ contained in $T$ shows that all context extensions of $\alpha$ to contexts of length up to $k$ give rise to the same preference order over $v_i$ in $c'$. Therefore, $P'$ cannot be a strict superset of $P$---otherwise all elements of $P'\setminus P$ would be dummy parents. So $P=P'$. 

As just stated, the swap expression of $U_{v_i}$ contained in $T$ determines the preference order over $v_i$ for each context $\alpha$ over $P$ and each extension of such $\alpha$ to contexts over $k$ variables. Since every parent set in $c'$ has size at most $k$, the $\CPT$ of $v_i$ in $c'$ is thus fully determined. Consequently, $c=c'$, i.e., $T$ is a teaching set for $c$. This concludes the proof.
\end{proof}

We illustrate the teaching set construction in the proof of Lemma~\ref{lem:TDbound} with an example.

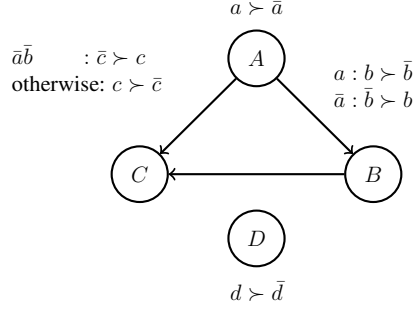
\begin{figure}
\centering{
\begin{tikzpicture}[thick,every node/.style={scale=0.7}]
\begin{scope}
\node[state](A){$A$};
\node[state](D)[below=of A,yshift=-0.9cm]{$D$};
\node[state](B)[below right=of A]{$B$};
\node[state](C)[below left=of A]{$C$};
\node (aCPT)[above=of A,yshift=-1.3cm]{$a\succ\bar{a}$};
\node (dCPT)[below=of D,yshift=1.3cm]{$d\succ\bar{d}$};
\node(bCPT)[above=of B,yshift=-1cm]{\makecell[l]{$a:b\succ \bar{b}$\\$\bar{a}:\bar{b}\succ b$}};
\node(cCPT)[above=of C,xshift=-1cm,yshift=-.7cm]{\makecell[l]{$\bar{a}\bar{b}\qquad :\bar{c}\succ c$\\otherwise: $c\succ \bar{c}$}};
\path (A) edge[->] (B) edge[->] (C)
(B) edge[->] (C);
\end{scope}
\end{tikzpicture}
}
\caption{Target CP-net for illustration of Lemma~\ref{lem:TDbound} and Algorithm 2.}\label{fig:alg2}
\end{figure}

\begin{example}
Consider the case $n=4$, $m=2$, and $k=2$, and suppose the CP-net from Figure~\ref{fig:alg2} is the target CP-net. Note that a smallest $(3,2)$-universal set has exactly four elements, which will be the cardinality of each set $\mathcal{F}$ considered for each variable. To construct a teaching set, each variable will be treated separately.

For variable $A$, one selects a set of contexts that are imposed by a smallest $(2,3,2)$-universal set. One such choice is $\{bcd, \bar{b}c\bar{d}, b\bar{c}\bar{d}, \bar{b}\bar{c}d\}$. This leads to the following swap expression $\mathcal{F}$.
\[
\mathcal{F}=\{(abcd,\bar{a}bcd),(a\bar{b}c\bar{d},\bar{a}\bar{b}c\bar{d}),(ab\bar{c}\bar{d},\bar{a}b\bar{c}\bar{d}),(a\bar{b}\bar{c}d,\bar{a}\bar{b}\bar{c}d)\}\,.
\]
The set $T$ will contain all the examples in $\mathcal{F}$, together with their correct labels, which will show that the target CP-net entails $a\gamma\succ\bar{a}\gamma$ for all $\gamma\in\mathcal{O}_{V\setminus\{A\}}$. This part of the teaching set $T$ hence determines that $A$ has no parents and that $a\succ\bar{a}$ is the only statement in $\CPT(A)$.

For variable $B$, the same construction would yield a swap expression
\[
\mathcal{F}=\{(abcd,a\bar{b}cd),(\bar{a}bc\bar{d},\bar{a}\bar{b}c\bar{d}),(ab\bar{c}\bar{d},a\bar{b}\bar{c}\bar{d}),(\bar{a}b\bar{c}d,\bar{a}\bar{b}\bar{c}d)\}\,.
\]
Including these swaps with their correct labels in $T$ reveals that $abcd\succ a\bar{b}cd$, $\bar{a}\bar{b}c\bar{d}\succ\bar{a}bc\bar{d}$, $ab\bar{c}\bar{d}\succ a\bar{b}\bar{c}\bar{d}$, $\bar{a}\bar{b}\bar{c}d\succ \bar{a}b\bar{c}d$. It is clear from these examples already that $B$ must have at least one parent, but the parent set itself is not yet determined. In order to show that $A$ is a parent of $B$, we select the swap $x=(abcd,a\bar{b}cd)$ from $T$ and pick an appropriate swap $x'$ so that $(x,x')$ forms a conflict pair for $B$ given $A$; in this case $x'=(\bar{a}bcd,\bar{a}\bar{b}cd)$. We add $x'$ with its label to $T$. Now we claim that $T$ fully determines $\CPT(B)$. To see this, partition $\mathcal{F}$ according to the values in $A$:
\[
\mathcal{F}_{A=a}=\{(abcd,a\bar{b}cd),(ab\bar{c}\bar{d},a\bar{b}\bar{c}\bar{d})\}\mbox{ and }\mathcal{F}_{A=\bar{a}}=\{ (\bar{a}bc\bar{d},\bar{a}\bar{b}c\bar{d}),(\bar{a}b\bar{c}d,\bar{a}\bar{b}\bar{c}d)\}\,.
\]
Each set displays a single preference order over $D_B$, so that $B$ has no further parents. The set $T$ thus reveals that $Pa(B)=\{A\}$ and that $\CPT(B)$ contains the statements $a:b\succ\bar{b}$ and $\bar{a}:\bar{b}\succ b$. 

For variable $C$, the construction from the proof of Lemma~\ref{lem:TDbound} works similarly as for $B$. Two conflict pairs reveal that $A$ and $B$ are parents of $C$, and partitioning the initial swap expression set into four subsets (according to the four possible assignments to the variables $A$ and $B$) will show the same preference order over $D_C$ in each of the four parts, which will fully determine $\CPT(C)$.

Finally, variable $D$ can be treated by analogy to the case of variable $A$.
\end{example}

Our main result on the teaching dimension can be stated as follows.

\begin{theorem} Let $n\ge 1$, $m=2$, and $0\le k\le n-1$. The teaching dimension of the class of all complete acyclic $k$-bounded CP-nets over $n$ variables of domain size $m$, over the instance space $\mathcal{X}_{swap}$, is bounded as follows.
\[n(m-1)m^k\le n(m-1)\mathcal{U}_k\le\TD(\mathcal{C}_{ac}^k)\le e_{max}+n(m-1)\mathcal{U}_k\le nk+n(m-1)\binom{n-1}{k}m^k\,.\]
For the values $k=0$, $k=1$, and $k=n-1$, the following holds:
\begin{enumerate}
\item $\TD(\mathcal{C}_{ac}^0)=(m-1)n$. 
\item $\TD(\mathcal{C}_{ac}^1)=(m-1)mn$. 
\item $\TD(\mathcal{C}_{ac}^{n-1})=(m-1)nm^{n-1}$.
\end{enumerate}
\label{thm:TD}
\end{theorem} 

\begin{proof}
To obtain a lower bound of $n(m-1)\mathcal{U}_k$ on the teaching dimension, let us first determine the teaching dimension of any complete separable CP-net with respect to the class $\mathcal{C}_{ac}^k$. Consider any unconditional $\CPT$, i.e., a $\CPT$ of the form $\CPT(v_i)=\{a_1\succ \dots \succ a_m\}$. To distinguish a complete separable CP-net $c$ with this $\CPT$ from all other concepts in $\mathcal{C}_{ac}^k$, examples are needed that show that $v_i$ has no parent. Since for CP-nets in $\mathcal{C}_{ac}^k$, the variable $v_i$ can have up to $k$ parents, a teaching set for $c$ has to demonstrate the following: For each set $R$ of $k$ variables in $V\setminus\{v_i\}$, and for each of the $m^k$ possible contexts over $R$, the target CP-net $c$ entails the same order $a_1\succ \dots \succ a_m$. That means, a teaching set of swap pairs for $c$ must ``cover'' each such context over $R$, and for each such context provide at least $m-1$ examples so that the total order over the $m$ domain values of $v_i$ can be determined. In other words, the swap pairs with swapped variable $v_i$ that occur in the teaching set for $c$ must be partitioned into parts of size $m-1$, where each part forms an $(m,n-1,k)$-universal set. The same applies to each of the $n$ variables in $V$, since knowledge of the $\CPT$ for any subset of $V$ still leaves open the possibility of up to $k$ parents for each of the remaining variables. Consequently, to teach $c$, at least $n(m-1)\mathcal{U}_k$ examples are needed. Thus, $\TD(\mathcal{C}_{ac}^k)\ge n(m-1)\mathcal{U}_k$.

The upper bound of $e_{max}+n(m-1)\mathcal{U}_k$ on the teaching dimension is immediate from Lemma~\ref{lem:TDbound}.

The upper and lower bounds on $\mathcal{U}_k$ are not hard to obtain. First, $m^k\le \mathcal{U}_k$ because an $(m,n-1,k)$-universal set must yield $m^k$ distinct projections onto any set of at least $k$ components. Second, $\mathcal{U}_k\le\binom{n-1}{k}m^k$ because an $(m,n-1,k)$-universal set can always be obtained as follows: for each $k$-element subset of the $n-1$-element universe, include $m^k$ vectors, namely one for each possible assignment of values in $\{1,\ldots,m\}$ to each of the $k$ elements in the subset.

Concerning the values for $k=0$, $k=1$, and $k=n-1$, note that $\mathcal{U}_0=1$, $\mathcal{U}_1=m$, and $\mathcal{U}_{n-1}=m^{n-1}$ by Lemma~\ref{lem:universal}, which yields the desired lower bounds on the teaching dimension values for $\mathcal{C}_{ac}^0$, $\mathcal{C}_{ac}^1$, and $\mathcal{C}_{ac}^{n-1}$, respectively. 

For $k=0$, we have $e_{max}=0$, so that the upper and lower bounds on the teaching dimension coincide, yielding $n(m-1)$. 

For $k=1$, we will show an upper bound of $(m-1)mn$, matching the lower bound. Consider any variable $v_i$ in the target CP-net. If it has no parents, then its $\CPT$ can be taught with $(m-1)\mathcal{U}_1=(m-1)m$ examples. If it has a parent, then $(m-1)m$ suitably chosen examples suffice to teach the target preference relation over $D_{v_i}$ for every possible value of the parent variable. Since $k=1$, it is then clear that $v_i$ has no other variable, so that $\CPT(v_i)$ is fully determined. Applying the same reasoning to every variable, we obtain a teaching set of size $n(m-1)m$.

For $k=n-1$, the target concept is uniquely determined by a teaching set that specifies an order over $m$ values via $m-1$ suitably chosen swaps over $v$, for each $v\in V$ and each of the $m^{n-1}$ contexts over the remaining variables; i.e., the lower bound of $(m-1)nm^{n-1}$ is attained.
\end{proof}

Theorem~\ref{thm:TD} implies that, for $\mathcal{C}_{ac}^{n-1}$, the ratio of $\TD$ over instance space size $|\mathcal{X}_{swap}|$ is $\frac{2}{m}$. In particular, in the case of binary CP-nets (i.e., when $m=2$), which is the focus of most of the literature on learning CP-nets, the teaching dimension equals the instance space size. However, maximal concepts have a teaching dimension far below the worst-case teaching dimension. 

\section{The Complexity of Learning Potentially Incomplete Acyclic CP-Nets}\label{sec:incomplete}

In this section, we revisit our results on the VC dimension and the teaching dimension for the case that CP-nets are potentially incomplete. This extension is useful for several reasons.
\begin{itemize}
\item In practice, a user's preferences may not always be representable by a complete CP-net. For example, the genre of a movie might be a parent variable of the variable representing the lead actor. Now a user may prefer actor $a_1$ over actor $a_2$ when the genre is comedy, while not having any preference between the same two actors for a different genre. In other words, the user generally considers the two actors equally preferable, but considers $a_1$ a better comedian than $a_2$. Such preference relation can be modeled with an incomplete CP-net.
\item We will introduce algorithms that learn both complete and incomplete CP-nets from membership queries. In order to assess the optimality of these algorithms in terms of the number of queries required, we will use the value of the teaching dimension of the corresponding concept class. Likewise, we will reassess an algorithm by Koriche and Zanuttini~\cite{Koriche2010685} that uses equivalence and membership queries for learning potentially incomplete CP-nets; for this purpose we need to calculate the VC dimension of the corresponding concept class.
\item Our extension of Theorem~\ref{thm:VCDac} to the case of potentially incomplete CP-nets substantially improves on (and corrects) a result by Koriche and Zanuttini~\cite{Koriche2010685}, who present a lower bound on $\VCdim(\mathcal{C}_{ac}^k)$; their bound is in fact incorrect unless $k\ll n$.
\end{itemize}

The main results of this section are summarized in Table~\ref{tab:resultsIncomplete}. The calculation of the $\RTD$ of classes containing complete and incomplete CP-nets is left as an open problem.

\begin{table*}
\centering
\caption{Summary of complexity results for classes of potentially incomplete CP-nets. $\mathcal{M}_k=(n-k)m^k+\frac{m^k-1}{m-1}$, as detailed in Lemma~\ref{lem:Mk}; $e_{max}=(n-k)k+\binom{k}{2}\le nk$; the value $\mathcal{U}_k$ is defined in Definition~\ref{def:Uk}.}
\begin{footnotesize}
\begin{tabular}{|c|c|c|}
\hline class & VCD & TD \\
\hline $\overline{\mathcal{C}}_{ac}^k$& $\geq (m-1)\mathcal{M}_k$&  $2n(m-1)\mathcal{U}_k\le \TD \le e_{max}+2n(m-1)\mathcal{U}_k$ \\
\hline $\overline{\mathcal{C}}_{ac}^{n-1}$& $m^n-1$& $2n(m-1)m^{n-1}$ \\
\hline $\overline{\mathcal{C}}_{ac}^{0}$&$(m-1)n$&$2(m-1)n$  \\
\hline
\end{tabular}
\end{footnotesize}
\label{tab:resultsIncomplete}
\end{table*}

\subsection{VC Dimension}

We will first establish that Theorem~\ref{thm:VCDac} remains true when we allow incomplete CP-nets and use the larger instance space $\overline{\mathcal{X}}_{swap}$. In particular, the VC dimension $\mathcal{C}_{ac}^k$ of the class of all complete acyclic $k$-bounded CP-nets (over $\mathcal{X}_{swap}$) equals the VC dimension of the class $\overline{\mathcal{C}}_{ac}^k$ of all complete and incomplete acyclic $k$-bounded CP-nets (over $\overline{\mathcal{X}}_{swap}$.) 

\begin{theorem}\label{thm:VCD}
For fixed $n\ge 1$, $m\ge 2$, and $k\le n-1$, the value $\VCdim(\mathcal{C}_{ac}^k)$, taken over the instance space $\mathcal{X}_{swap}$, is equal to $\VCdim(\overline{\mathcal{C}}_{ac}^k)$, taken over the instance space $\overline{\mathcal{X}}_{swap}$. In particular,
\begin{enumerate}
\item $\VCdim(\overline{\mathcal{C}}_{ac}^{n-1})=m^n-1$.
\item $\VCdim(\overline{\mathcal{C}}_{ac}^0)=(m-1)n$.
\item $\VCdim(\overline{\mathcal{C}}_{ac}^k)\geq(m-1)\mathcal{M}_k = (m-1)(n-k)m^k+m^k-1$. 
\end{enumerate}
\end{theorem}

\begin{proof}
By definition, $\VCdim(\overline{\mathcal{C}}_{ac}^k)\ge \VCdim(\mathcal{C}_{ac}^k)$. It remains to show that $\VCdim(\mathcal{C}_{ac}^k)\ge \VCdim(\overline{\mathcal{C}}^k_{ac})$. 

Suppose $X\subseteq \overline{\mathcal{X}}_{swap}$ is a set of swap instances that is shattered by $\overline{\mathcal{C}}^k_{ac}$. We need to show that there is a set $X' \subseteq \mathcal{X}_{swap}$ such that $|X'|=|X|$ and $X'$ is shattered by $\mathcal{C}_{ac}^k$. Since $X$ is shattered by $\overline{\mathcal{C}}^k_{ac}$, there is no $(o,o')$ such that both $(o,o')$ and $(o',o)$ belong to $X$. Otherwise there would be a $c\in \overline{\mathcal{C}}^k_{ac}$ such that $c(o,o')= c(o',o)=1$, which is an impossible inconsistency. 

The fact that there is no $(o,o')$ such that both $(o,o')$ and $(o',o)$ belong to $X$ has two implications: 

(i) Either $X\subseteq \mathcal{X}_{swap}$ or swapping the order of some pairs in $X$ results in a set $X'$ such that $X'$ is a subset of $\mathcal{X}_{swap}$.

(ii) Every incomplete CP-net $c\in \overline{\mathcal{C}}^k_{ac}$ can be turned into a complete CP-net in $\mathcal{C}^k_{ac}$ by adding statements to some of its $\CPT$s without changing the value of $c(x)$ for any $x\in X$. (This is because every incomplete CP-net results from a complete one by removing some $\CPT$ statement(s).) This holds in particular for the set of $2^{|X|}$ CP-nets in $\overline{\mathcal{C}}^k_{ac}$ that witness the shattering of $X$: each of them can be completed in a way so that the $2^{|X|}$ completions still shatter $X$. 

Combining (i) and (ii), there is a set $X' \subseteq \mathcal{X}_{swap}$ shattered by $\mathcal{C}_{ac}^k$ such that $|X'|=|X|$.

Theorem~\ref{thm:VCDac} then yields the claimed formulas.
\end{proof}

\subsubsection{Re-assessment of Koriche and Zanuttini's contribution}

Koriche and Zanuttini~\cite{Koriche2010685} present an algorithm for exact learning of potentially incomplete $k$-bounded acyclic CP-nets from membership and equivalence queries. To evaluate their algorithm, they compare its query consumption to the value $\log_2(4/3)\VCdim(\mathcal{C})$, which is a lower bound on the required number of membership and equivalence queries, known from fundamental learning-theoretic studies~\cite{AuerL99}. In lieu of an exact value for $\VCdim(\mathcal{C})$, Koriche and Zanuttini plug in a lower bound on $\VCdim(\mathcal{C})$, cf.\ their Theorem~6. Using our Theorem~\ref{thm:VCD}, we show in Appendix\ref{sec:app} that this lower bound is not quite correct; consequently, here we re-assess the query consumption of Koriche and Zanuttini's algorithm. 

For any $k$, their algorithm uses at most $s_{N^*}+e_{N^*}\log_2(n)+e_{N^*}+1$ queries in total, for a target CP-net $N^*$ with $s_{N^*}$ statements and $e_{N^*}$ edges. In the worst case $s_{N^*}=\mathcal{M}_k\leq \VCdim(\mathcal{C})$ and $e_{N^*}=\binom{k}{2}+(n-k)k$ (i.e., $N^*$ is maximal w.r.t.\ $\mathcal{C}$). This yields $\mathcal{M}_k+e_{N^*}(\log_2(n)+1)$ queries for their algorithm, which exceeds the lower bound $\log_2(4/3)\VCdim(\mathcal{C})$ by at most $\log_2(3/2)\VCdim(\mathcal{C})+e_{N^*}\log_2(n)$. This is a more refined assessment compared to the term $e_{N^*}\log_2(n)$ that they report, and it holds for \emph{any}\/ value of $k$. 


\subsection{Teaching Dimension}

It is not difficult to adapt our previous teaching dimension results to the case of learning both complete and incomplete acyclic CP-nets, for various indegree bounds $k$. 

\begin{theorem} Let $n\ge 1$, $m=2$, and $0\le k\le n-1$. The teaching dimension of the class of all complete and incomplete acyclic $k$-bounded CP-nets over $n$ variables of domain size $m$, over the instance space $\overline{\mathcal{X}}_{swap}$, is bounded as follows.
\[2n(m-1)m^k\le 2n(m-1)\mathcal{U}_k\le\TD(\overline{\mathcal{C}}_{ac}^k)\le e_{max}+2n(m-1)\mathcal{U}_k\le nk+2n(m-1)\binom{n-1}{k}m^k\,.\]
For the values $k=0$, $k=1$, and $k=n-1$, the following holds:
\begin{enumerate}
\item $\TD(\overline{\mathcal{C}}_{ac}^0)=2(m-1)n$. 
\item $\TD(\overline{\mathcal{C}}_{ac}^1)=2(m-1)mn$. 
\item $\TD(\overline{\mathcal{C}}_{ac}^{n-1})=2(m-1)nm^{n-1}$.
\end{enumerate}
\label{thm:TDincomplete}
\end{theorem} 

\begin{proof} For the lower bound, consider the empty CP-net, i.e., a separable CP-net without any statements in any of its $\CPT$s. We use the same argument as in the proof of Theorem~\ref{thm:TD} to show that one needs at least $n$ swap expressions of contexts imposed by universal sets to teach this CP-net. Note however, that every swap $(o,o')$ in the resulting teaching set will be labeled 0, as the target CP-net entails no preferences. This would leave open the option that the reverse swap $(o',o)$ could be labeled 1 by the target CP-net. To exclude this option, the teaching set must contain the reverse swap $(o',o)$ as well, with the correct label 0. Therefore, the lower bound is exactly twice as large as the lower bound in Theorem~\ref{thm:TD}.

For the upper bound, our argument is similar to that in the proof of Lemma~\ref{lem:TDbound}. The construction of the teaching sets remains essentially the same: one includes (i) a swap expression of a context set imposed by a universal set, in order to teach preferences once the parents are known, as well as to determine when a variable has no parents (on top of the already found ones), and (ii) a conflict pair $(x,x')$ for each edge in the target CP-net, where $x$ is chosen from the swap expression set. The only difference is that we make the swap expression set twice as large by including the reverse swap $(x.2,x.1)$ to every $(x.1,x.2)$ already included. This will make sure that, in the case of a swap $(o,o')$ labeled 0, one can find out whether the reverse swap $(o',o)$ is also labeled 0 (which would mean that the target CP-net entails no preference between $o$ and $o'$) or whether the reverse swap is labeled 1 (which would mean that the target CP-net entails the preference $o'\succ o$). 

The trivial upper and lower bounds on the value $\mathcal{U}_k$, namely $\binom{n-1}{k}m^k$ and $m^k$, respectively, were already established in Theorem~\ref{thm:TD}.

For the special cases of $k=0$, $k=1$, and $k=n-1$, one only needs to adapt the argument in the proof of Theorem~\ref{thm:TD} by forming the symmetric closure of the set of swaps used in the teaching sets, i.e., always adding the reverse swaps.
\end{proof}

\section{Learning from Membership Queries}
\label{sec:perfect}

In this section, we investigate the problem of learning complete CP-nets from membership queries alone, in contrast to the setting considered by Koriche and Zanuttini, where the learner asks both membership and equivalence queries \cite{Koriche2010685}.

A membership query, represented by a pair $(o,o')$ of objects, corresponds to asking the user directly whether they prefer $o$ over $o'$ or not. In an idealized model, it is expected that the user will always answer truthfully. An equivalence query, represented by a CP-net $N$, corresponds to asking the user whether $N$ correctly captures their preferences. The user is expected to answer correctly, where a negative response is accompanied by a witness, i.e., a pair $(o,o')$ of objects for which $N$ entails a preference opposite to the user's.

While in computational learning theory, beginning with Angluin's seminal work on learning regular languages~\cite{Angluin87}, the combination of equivalence and membership queries has been the most commonly investigated query scenario, we have several reasons for investigating learning from membership queries only.

Firstly, from the cognitive perspective, answering membership queries of the form \quotes{\emph{is $o$ better than $o'$?}} is more intuitive and poses less burden upon the user than comparing a proposed CP-net to the true one (which is the case when answering equivalence queries). Considering that membership queries can at times be perceived as too intrusive by the user, equivalence queries do not seem reasonable at all. To the best of our knowledge, the only ``real-world'' implementations of equivalence queries apply to situations in which the equivalence queries are answered by a program rather than by a human user. One such application is in the area of formal methods, where an equivalence query corresponds to a guess on the semantics of a program in a given class of target programs, and a formal verification procedure (comparable to a model checker) can be run in order to verify the correctness of the guess. This works well for fully automated reasoning about finite-state systems, where such verification procedures can be implemented, but it is unreasonable to assume that users of an e-commerce system (or of some similar kind of application) can comprehend and verify equivalence queries about their preferences.\footnote{Note that equivalence queries, from a learning-theoretic point of view, can be considered as a prediction system ``in action''~\cite{Angluin:1988:QCL:639961.639995}: a system corresponding to an equivalence query for $N$ would keep predicting the user's preferences according to the entailments of $N$. As long as all predictions by the system are correct, the hypothesis $N$ will be maintained. When a prediction mistake is made, the system will update its hypothesis using the mistake as a counterexample, which then results in a new equivalence query. In the end, the overall number of mistakes made by the system equals the number of equivalence queries asked. To the best of our knowledge, there is no reasonable application scenario in preference elicitation where such system would make sense. In practice, a system may want to offer products to a user, and the system's success lies in being able to identify the best product according to the user's preferences. However, the scenario of an equivalence query for a CP-net does not correspond to making guesses about the best product, but rather to making guesses about preferences between any two arbitrary products, for which no obvious counterpart (other than making membership queries) is seen in practice.}

Secondly, Koriche and Zanuttini showed that membership queries are powerful in the sense that CP-nets are not efficiently learnable from equivalence queries alone but they are from equivalence and membership queries \cite{Koriche2010685}. Thus, an immediate question is whether membership queries alone are powerful enough to efficiently learn CP-nets. In particular, we can consider learning from membership queries as an extreme case of limiting the allowable number of equivalence queries in order to investigate how much information can be obtained from the (intuitively less costly) membership queries.

Thirdly, in preference elicitation in general, the notion of membership query is of central importance, as is detailed in Section~\ref{ssec:pe}. For example, a value query asks how much an item is worth to a user~\cite{BlumJSZ04}. Clearly, a membership query asking the user to express a preference between two items can always be simulated by two value queries. However, membership queries may be easier to answer than value queries, as they do not require the user to quantify the value of an item. 


The complexity results presented in Sections \ref{sec:complete} and \ref{sec:incomplete} have interesting consequences on learning CP-nets from membership queries alone. In particular, it is known that the query complexity of the optimal membership query algorithm is lower-bounded by the teaching dimension of the class \cite{queriesRevisited}. Therefore, in this section, we propose strategies to learn CP-nets and use the $\TD$ results to assess their optimality. In what follows, we show near-optimal query strategies for tree CP-nets and generally for classes of bounded acyclic CP-nets. 

%


\subsection{Tree CP-Nets}
\label{sec:treeMQ}
Koriche and Zanuttini~\cite{Koriche2010685} present an algorithm for learning a binary tree-structured CP-net $N^*$ that may be incomplete (i.e., it learns the superclass $\overline{\mathcal{C}}_{ac}^1$ of $\mathcal{C}_{ac}^1$ for $m=2$.) Their learner uses at most $n_{N^*}+1$ equivalence queries and $4n_{N^*}+e_{N^*}\log_2(n)$ membership queries, where $n_{N^*}$ is the number of relevant variables and $e_{N^*}$ the number of edges in $N^*$. We present a method for learning any CP-net in $\mathcal{C}_{ac}^1$ (i.e., \emph{complete\/} tree CP-nets) for any $m$, using only membership queries. Later, we will extend that method to cover also the case of incomplete CP-nets.

Recall that, for a CP-net $N$, a conflict pair w.r.t.\ $v_i$ is a pair $(x,x')$ of swaps such that (i) $V(x)=V(x')=v_i$, (ii) $x.1$ and $x'.1$ agree on $v_i$, (iii) $x.2$ and $x'.2$ agree on $v_i$, and (iv) $N$ entails one of the swaps $x,x'$, but not the other, cf.\ Definition~\ref{def:conflict}. If $v_i$ has a conflict pair $(x,x')$, then $v_i$ has a parent variable $v_j$ whose values in $x$ and $x'$ are different. Such a variable $v_j$ can be found with $\log_2(n)$ membership queries by binary search (each query halves the number of candidate variables with different values in $x$ and $x'$) \cite{adaptiveLearning}. 

We use this binary search to learn tree-structured CP-nets from membership queries, by exploiting the following fact: if a variable $v_i$ in a tree CP-net has a parent, then a conflict pair w.r.t.\ $v_i$ exists and can be detected by asking membership queries to sort $m$ ``test sets'' for $v_i$. Let $(v^i_ 1,\ldots,v^i_m)$ be an arbitrary but fixed permutation of $D_{v_i}$. Then, for all $j\in\{1,\ldots,m\}$, a test set $I_{i,j}$ for $v_i$ is defined by $I_{i,j}=\{(v^1_j,\ldots,v^{i-1}_j,v^i_r,v^{i+1}_j,\ldots,v^n_j)\mid 1\le r\le m\}$. Since $v_i$ has no more than one parent, determining preference orders over $m$ such test sets of size $m$ is sufficient for revealing conflict pairs, rather than having to test all possible contexts in $\mathcal{O}_{V\setminus\{v_i\}}$.

\begin{example}
	Consider the set of variables $V=\{A,B,C\}$ where $D_{A}=\{a,a',a'',a'''\}$, $D_{B}=\{b,b',b'',b'''\}$, and $D_{C}=\{c,c',c'',c'''\}$. The following is one possible collection of test sets for the variable $A$:
	\begin{align*}
	\vspace*{-\baselineskip}
	&I_{A,1}=\{abc, a'bc, a''bc, a'''bc\}\\
	&I_{A,2}=\{ab'c', a'b'c', a''b'c', a'''b'c'\}\\
	&I_{A,3}=\{ab''c'',a'b''c'',a''b''c'',a'''b''c''\}\\
	&I_{A,4}=\{ab'''c''',a'b'''c''',a''b'''c''',a'''b'''c'''\}
	\end{align*}
	
\end{example}

Clearly, a complete target CP-net imposes a total order on every $I_{i,j}$, which can be revealed by posing enough membership queries selected from the $\binom{m}{2}$ swaps over $I_{i,j}$; a total of $O(m\log_2(m))$ comparisons suffice to determine the order over $I_{i,j}$. This yields a simple algorithm for learning tree CP-nets with membership queries: 

\medskip

\noindent\textbf{Algorithm 1.} For every variable $v_i$, determine $Pa(v_i)$ and $\CPT(v_i)$ as follows:
\begin{enumerate}
	\item For every value $j\in\{1,\ldots,m\}$, ask $O(m\log_2(m))$ membership queries from the $\binom{m}{2}$ swaps over $I_{i,j}$ to obtain an order over $I_{i,j}$. Note that, for $m=2$, a single query is enough to obtain an order over $I_{i,j}$.
	\item If for all $j_1,j_2\in\{1,\ldots,m\}$ the obtained order over $I_{i,j_1}$ imposes the same order on $D_{v_i}$ as the obtained order over $I_{i,j_2}$ does, i.e., there is no conflict pair for $v_i$, then $Pa(v_i)=\emptyset$. In this case, $\CPT(v_i)$ is fully determined by the queries in Step 1, following the order over $D_{v_i}$ that is imposed by the order over any of the $I_{i,j}$.
	\item If there are some $j_1,j_2\in\{1,\ldots,m\}$ such that the obtained order over $I_{i,j_1}$ imposes a different order on $D_{v_i}$ than the obtained order over $I_{i,j_2}$ does, i.e., there is a conflict pair $(x,x')$ for $v_i$, then find the only parent of $v_i$ by $\log_2(n)$ further queries, as described by Damaschke \cite{adaptiveLearning}. From these queries, together with the ones posed in Step 1, $\CPT(v_i)$ is fully determined.
\end{enumerate}

The procedure described by Damaschke \cite{adaptiveLearning} is a binary search on the set of candidates for the parent variable. Let $(x,x')$ be the conflict pair over variable $v_i$, as found in Step 3, where 
\begin{eqnarray*}
x&=(a_1\ldots a_{i-1}a_ia_{i+1}\ldots a_n,\ a_1\ldots a_{i-1}\overline{a_i}a_{i+1}\ldots a_n)\,,\\
x'&=(a'_1\ldots a'_{i-1}a_ia'_{i+1}\ldots a'_n,\ a'_1\ldots a'_{i-1}\overline{a_i}a'_{i+1}\ldots a'_n)\,.
\end{eqnarray*}
Initially, each variable other than $v_i$ is a potential parent. The set of potential parents is halved recursively by asking membership queries for swaps $(o,o')$ over $v_i$, with $o(v_i)=a_i$, $o'(v_i)=\bar{a_i}$, and half of the potential parent variables in $o$ and $o'$ having the same values as in $x$, while the other half of the potential parent variables has values identical to those in $x'$. (The variables that have been eliminated from the set of potential parent variables will all be assigned the same values as in $x$.)

\begin{example}
Suppose the queries on the test sets revealed a conflict pair $(x,x')$ for the variable $v_5$, where
\begin{eqnarray*}
x&=(a_1a_2a_3a_4a,\ a_1a_2a_3a_4\overline{a})\,,\\
x'&=(a'_1a'_2a'_3a'_4a,\ a'_1a'_2a'_3a'_4\overline{a})\,,
\end{eqnarray*}
and
\[
a_1a_2a_3a_4a\succ a_1a_2a_3a_4\overline{a}\,,
\]
while
\[
a'_1a'_2a'_3a'_4\overline{a}\succ a'_1a'_2a'_3a'_4a\,.
\]
To find the only parent of $v_5$, Damaschke's procedure will check whether 
\[
a_1a_2a'_3a'_4a\succ a_1a_2a'_3a'_4\overline{a}\,.
\]
If yes, then either $v_1$ or $v_2$ must be the parent of $v_5$, and one will next check whether 
\[
a_1a'_2a_3a_4a\succ a_1a'_2a_3a_4\overline{a}\,.
\]
If yes, then $v_1$ is the parent of $v_5$, else $v_2$ is the parent of $v_5$.
If, however, $a_1a_2a'_3a'_4\overline{a}\succ a_1a_2a'_3a'_4a$, then the second query would have been to test whether $a_1a_2a'_3a_4a\succ a_1a_2a'_3a_4\overline{a}$, in order to determine whether the parent of $v_5$ is $v_3$ or $v_4$.
\end{example}

We formulate a result on learning complete tree CP-nets.

\begin{theorem}\label{thm:treePerfect}
Let $n\ge 1$, $m\ge 2$. Algorithm 1 learns every complete tree CP-net $N^*\in\mathcal{C}_{ac}^1$ over $n$ variables of domain size $m$ with 
\[
O(nm^2\log_2(m)+e_{N^*}\log_2(n))
\] 
membership queries over swap examples in $\mathcal{X}_{swap}$, where $e_{N^*}$ is the number of edges in $N^*$. In particular, for $m=2$, Algorithm 1 requires at most $2n+e_{N^*}\log_2(n)$ queries.
\end{theorem}

\begin{proof}
Step 1 consumes $O(m^2\log_2(m))$ queries and is run $n$ times, and Step 3 consumes an additional $\log_2(n)$ for every variable $v_i$ that has a parent. In total, this sums up to $O(nm^2\log_2(m)+e_{N^*}\log_2(n))$ queries. When $m=2$, Step 1 requires only $m$ queries for each variable $v_i$, namely a single query per test set, for 2 test sets. This results in a total of $2n+e_{N^*}\log_2(n)$ queries.

The correctness of Algorithm 1 follows from the properties of conflict pairs and tests sets. The existence of a conflict pair is equivalent to the existence of a parent, and the orders over the test sets correspond to $\CPT$ statements for all relevant contexts. The $m$ test sets as specified above are sufficient since each variable has at most one parent. Thus, the responses to the queries asked by Algorithm 1 uniquely determine the target tree CP-net.
\end{proof}

As Theorem~\ref{thm:treePerfect} states, for the binary case, Algorithm 1 requires $2n+e_{N^*}\log_2(n)$ queries at most, i.e., compared to Koriche and Zanuttini's method, when focusing only on tree CP-nets with non-empty $\CPT$s, our method reduces the number of membership queries by a factor of 2, while at the same time dropping equivalence queries altogether.

It is a well-known fact that the teaching dimension of a concept class $\mathcal{C}$ is a lower bound on the worst-case number of membership queries required for learning concepts in $\mathcal{C}$~\cite{GK95}. We have proven above that $\TD(\mathcal{C}^1_{ac})=n(m-1)\mathcal{U}_1=nm(m-1)$. That means that our method uses no more than on the order of $\log_2(m)+e_{N^*}\log_2(n)$ queries more than an optimal one, which means, asymptotically, it uses at most an extra $e_{N^*}\log_2(n)$ queries when $m=2$. 

For $m=2$, i.e., when CP-nets are binary, it is not hard to extend our result to the case of potentially incomplete CP-nets, at the cost of just doubling the number of queries. Perhaps not coincidentally, the teaching dimension also essentially doubles when switching from complete CP-nets to both complete and incomplete CP-nets, cf.\ Theorems~\ref{thm:TD} and \ref{thm:TDincomplete}.

\begin{theorem}\label{thm:treePerfectIncomplete}
Let $n\ge 1$, $m=2$. There exists an algorithm that learns every (complete or incomplete) tree CP-net $N^*\in\overline{\mathcal{C}}_{ac}^1$ over $n$ binary variables with at most $4n+2e_{N^*}\log_2(n)$ membership queries over swap examples in $\overline{\mathcal{X}}_{swap}$, where $e_{N^*}$ is the number of edges in $N^*$. 
\end{theorem}

\begin{proof}
The desired algorithm is a simple modification of Algorithm 1; whenever Algorithm 1 asks a query for a swap $x=(x.1,x.2)$, the new algorithm will ask two queries, namely one for $x=(x.1,x.2)$ and one for $x'=(x.2,x.1)$ (since it is no longer guaranteed that $c(x')=1$ whenever $c(x)=0$.) Other than that, the algorithm proceeds the same way as Algorithm 1. If both queries $x$ and $x'$ are answered with 0, then the target CP-net does not specify a preference between $x.1$ and $x.2$. Determining orders over test sets then works with at most twice the number of queries as before, even if those orders are empty. Conflict pairs still imply the existence of a parent, and still every variable has at most one parent, which can still be found with binary search. The claim then follows from Theorem~\ref{thm:treePerfect}.
\end{proof}

Let us compare the algorithm provided by Koriche and Zanuttini~\cite{Koriche2010685} to ours. While our algorithm does not require equivalence queries, uses the same number of queries in total when learning both complete and incomplete CP-nets, and is described for not necessarily binary CP-nets, Koriche and Zanuttini's has two desirable properties that our method does not have:
\begin{itemize}
\item It learns CP-nets with nodes of arbitrary bounded indegree $k$, not just for $k=1$. 
\item It is attribute-efficient, i.e., the number of queries it poses is polynomial in the size of the target CP-net, but only logarithmic in the number $n$ of variables.
\end{itemize}

The attribute-efficiency of their algorithm is possible only because equivalence queries are allowed. In particular, consider the incomplete CP-net without edges in which every $\CPT$ is empty. This CP-net has size 0, but to learn it only with membership queries over swap examples requires to ask queries regarding every one of its $n$ $\CPT$s.

Concerning arbitrary values of $k$, Section~\ref{sec:boundedPerfect} will refine Algorithm 1 and present a new procedure that efficiently learns any $k$-bounded acyclic CP-net when $m=2$, i.e., in the binary case.


\subsection{Bounded Binary Acyclic CP-nets}\label{sec:boundedPerfect}

In this section, we show that the teaching dimension results for $\mathcal{C}_{ac}^k$ (cf.\ Theorem \ref{thm:TD}) immediately yield a general strategy for learning complete acyclic CP-nets from membership queries alone, when $m=2$. Recall that, in the binary case,  $n\mathcal{U}_k\le \TD(\mathcal{C}_{ac}^k)\le e_{max}+n\mathcal{U}_k$, where $\mathcal{U}_k$ is the size of an $(n-1,k)$-universal set $U$ of minimum size. 

For $k\geq 2$, the quantity $\mathcal{U}_k$ is known to be $\Omega(2^k \log_2(n-1))$ and $O(k2^k \log_2(n-1))$ \cite{DBLP:journals/tit/SeroussiB88}. Thus, using $e_{max}\le nk$, we obtain
\[
\Omega(n2^k \log_2(n-1))\ni \TD(\mathcal{C}_{ac}^k)\in O(nk2^k \log_2(n-1))\,.
\] 

Our proposed method for learning $\mathcal{C}_{ac}^k$ from membership queries reuses the idea of universal sets in teaching sets. For each variable $v_i$, one selects a swap expression $\mathcal{F}$ of the context set $U_{v_i}$ imposed by $U$. Then, for every variable $v_i$, one queries the elements of such set $\mathcal{F}$. The rest of the method we propose is an adaptation of Algorithm 1. For simplicity, in the description of the algorithm, we assume all variables have the same domain $\{0,1\}$.

\medskip

\noindent\textbf{Algorithm 2.} For every variable $v_i$, determine $Pa(v_i)$ and $\CPT(v_i)$ as follows:
\begin{enumerate}
	\item Ask membership queries for all the elements of $\mathcal{F}$. \emph{This step amounts to $\mathcal{U}_k$ membership queries.}
	\item Initialize $P=\emptyset$; the set $P$ will always contain all the parents of $v_i$ found so far. If $|P|=k$, $Pa(v_i)$ must equal $P$, since no variable has more than $k$ parents. In this case, go to Step 7, else go to Step 3.
	\item If all of the queried instances for $\mathcal{F}$ show the same preference over the values of $v_i$, then $v_i$ has no further parents, i.e., $Pa(v_i)=P$. (This is because there are at most $k$ parents for $v_i$ and the same statement appears in every context of all potential parent sets of size $k$, cf.\ Lemma~\ref{lem:TDuniversal}.) In this case, go to Step 7. 
	
	Otherwise, there exist two contexts $\gamma,\gamma'\in\mathcal{O}_{V\setminus (P\cup \{v_i\})}$ for which $\gamma:v^i_1\succ v^i_2$ and $\gamma':v^i_2\succ v^i_1$ are answers obtained from the queries in Step 1. Clearly, $Pa(v_i)$ contains an element of $V\setminus (P\cup \{v_i\})$. 
	\item Apply Damaschke's binary search procedure \cite{adaptiveLearning} to determine one new variable $v\in Pa(v_i)$, by analogy to the procedure used in Algorithm 1. Add $v$ to the set $P$. \emph{This step amounts to at most $\log_2(n-1)$ membership queries.} 
	\item Partition $\mathcal{F}$ into two sets $\mathcal{F}_{v=0}$ and $\mathcal{F}_{v=1}$. Here $\mathcal{F}_{v=z}$ is the set of all swaps in $\mathcal{F}$ in which the value of $v$ is $z$. Note that these swaps have all been queried in Step 1 already.
	\item If all of the queried instances for $\mathcal{F}_{v=0}$ show a single preference over the values of $v_i$, and likewise all of the queried instances for $\mathcal{F}_{v=1}$ show a single preference over the values of $v_i$, then $Pa(v_i)=P$; go to Step 7. 
	
	Otherwise, recursively add parents to $P$ by calling Step 3 once with $\mathcal{F}_{v=0}$ in place of $\mathcal{F}$ and once with $\mathcal{F}_{v=1}$ in place of $\mathcal{F}$, both times excluding the elements of $P$ from the binary search in Step 4.
	\item From the queries posed above, $\CPT(v_i)$ is fully determined.
\end{enumerate}

\begin{example}
Consider $n=4$, $m=2$, and $k=2$, and suppose the CP-net from Figure~\ref{fig:alg2} is the target. A smallest $(3,2)$-universal set has four elements, which will be the cardinality of each set $\mathcal{F}$ considered for each variable. Algorithm 2 will consider each variable in turn. 

For variable $A$, it will first ask four queries, one for each swap pair in, say, the following set
\[
\mathcal{F}=\{(abcd,\bar{a}bcd),(a\bar{b}c\bar{d},\bar{a}\bar{b}c\bar{d}),(ab\bar{c}\bar{d},\bar{a}b\bar{c}\bar{d}),(a\bar{b}\bar{c}d,\bar{a}\bar{b}\bar{c}d)\}\,.
\]
(Here the $(3,2)$-universal set realizing all possible 2-dimensional context sub-vectors within 3-dimensional context vectors would be $\{bcd, \bar{b}c\bar{d}, b\bar{c}\bar{d}, \bar{b}\bar{c}d\}$.)
The responses show that $a\gamma\succ\bar{a}\gamma$ for all $\gamma\in\mathcal{O}_{V\setminus\{A\}}$. The algorithm hence detects that $A$ has no parents and that $a\succ\bar{a}$ is the only statement in $\CPT(A)$.

For variable $B$, the same procedure applied to
\[
\mathcal{F}=\{(abcd,a\bar{b}cd),(\bar{a}bc\bar{d},\bar{a}\bar{b}c\bar{d}),(ab\bar{c}\bar{d},a\bar{b}\bar{c}\bar{d}),(\bar{a}b\bar{c}d,\bar{a}\bar{b}\bar{c}d)\}
\]
reveals that $abcd\succ a\bar{b}cd$, $\bar{a}\bar{b}c\bar{d}\succ\bar{a}bc\bar{d}$, $ab\bar{c}\bar{d}\succ a\bar{b}\bar{c}\bar{d}$, $\bar{a}\bar{b}\bar{c}d\succ \bar{a}b\bar{c}d$. Two contexts $\gamma$ and $\gamma'$ for which the preference over $D_B$ differs are, for example, $\gamma=acd$, $\gamma'=\bar{a}c\bar{d}$. Hence $B$ must have at least one parent among $A,C,D$, and applying Damaschke's binary search to the context $\gamma$ yields that the preference given $acd$ differs from the preference given $\bar{a}cd$. So $A$ must be a parent of $B$. Partitioning $\mathcal{F}$ yields 
\[
\mathcal{F}_{A=a}=\{(abcd,a\bar{b}cd),(ab\bar{c}\bar{d},a\bar{b}\bar{c}\bar{d})\}\mbox{ and }\mathcal{F}_{A=\bar{a}}=\{ (\bar{a}bc\bar{d},\bar{a}\bar{b}c\bar{d}),(\bar{a}b\bar{c}d,\bar{a}\bar{b}\bar{c}d)\}\,.
\]
Each set displays a single preference order over $D_B$ in the queries originally asked, so that $B$ has no further parents. The algorithm has determined that $Pa(B)=\{A\}$ and that $\CPT(B)$ contains the statements $a:b\succ\bar{b}$ and $\bar{a}:\bar{b}\succ b$. 

For variable $C$, queries for pairs in the corresponding set $\mathcal{F}$
reveal $abcd\succ ab\bar{c}d$, $\bar{a}bc\bar{d}\succ\bar{a}b\bar{c}\bar{d}$, $a\bar{b}c\bar{d}\succ a\bar{b}\bar{c}\bar{d}$, and $\bar{a}\bar{b}\bar{c}d\succ\bar{a}\bar{b}cd$. In the same way as for variable $B$, the algorithm finds a first parent of $C$, say $A$. This time, partitioning $\mathcal{F}$ into  $\mathcal{F}_{A=a}$ and $\mathcal{F}_{A=\bar{a}}$ will show that $C$ must have another parent, since the two contexts $\bar{a}b\bar{d}$ and $\bar{a}\bar{b}d$ in $\mathcal{F}_{A=\bar{a}}$ give rise to two different preference orders over $D_C$. Another binary search finds the parent $B$. Since $k=2$, the variable $C$ has no further parents, and the algorithm infers $Pa(C)=\{A,B\}$ alongside the CPT statements for $C$. 

Finally, Algorithm 2 would, similarly to the case of $\CPT(A)$, use four queries to learn $\CPT(D)$.
\end{example}

\begin{theorem}\label{thm:boundedPerfect}
Let $n\ge 1$, $m=2$, $0\le k\le n-1$. Algorithm 2 learns every complete binary $k$-bounded CP-net $N^*\in\mathcal{C}_{ac}^k$ over $n$ variables with 
\[
O(n\mathcal{U}_k+e_{N^*}\log_2(n))
\] 
membership queries over swap examples in $\mathcal{X}_{swap}$, where $e_{N^*}$ is the number of edges in $N^*$. 
\end{theorem}

\begin{proof} The number of queries made by Algorithm 2 can be upper-bounded as follows. The only steps in which queries are made are Steps 1 and 3. Step 1 is run once for each of the $n$ variables, each time costing $\mathcal{U}_k$ queries, for a total of $n\mathcal{U}_k$ queries. Step 3 is run once for each edge in the target CP-net, each time costing at most $\log_2(n)$ queries, for a total of $O(e_{N^*}\log_2(n))$ queries. In the sum, the query consumption of Algorithm 2 is in $O(n\mathcal{U}_k+e_{N^*}\log_2(n))$.

We will next prove the correctness of Algorithm 2. Note that the set $\mathcal{F}$ in Step 1 of the algorithm is a swap expression of a context set imposed by a $(2,n-1,k)$-universal set. By Lemma~\ref{lem:TDuniversal}, Algorithm 2 decides correctly in Step 2 whether or not $v_i$ has a parent. The same lemma can be applied to any part of the set $\mathcal{F}$ as defined in Step 5. Hence, Algorithm 2 correctly determines the parent set of each variable. Since $\mathcal{F}$ is a swap expression of a context set imposed by a $(2,n-1,k)$-universal set, and the domain is of size 2, the responses to queries in $\mathcal{F}$ determine a preference order for each of the relevant contexts. Consequently, Algorithm 2 is able to infer all $\CPT$ statements from those responses.
\end{proof}

Thus, one can identify any concept $c\in \mathcal{C}_{ac}^k$ with $n\mathcal{U}_k+e_c\log_2(n)$ membership queries, where $\mathcal{U}_k\in O(nk2^k\log_2(n-1))$.  A universal set of such size was proven to exist by a probabilistic argument with no explicit construction of the set \cite{Jukna:2010:ECA:1965203}. However, a construction of an $(n-1,k)$-universal set of size $2^k \log_2(n-1) k^{O(\log_2 (k))}$ is reported in the literature \cite{DBLP:journals/tit/SeroussiB88}. Therefore, one can effectively learn any concept in $\mathcal{C}_{ac}^k$ with a number of queries bounded by $n2^k \log_2(n-1) k^{O(\log_2 (k))}+e\log_2(n)$, which is at most a factor of $k^{O(\log_2 k)}$ away from the teaching dimension, and thus at most a factor of $k^{O(\log_2 k)}$ away from optimal. Moreover, when $k2^k<\sqrt{n}$, there is an explicit construction of a $(n,k)$-universal set of size $n$~\cite{Jukna:2010:ECA:1965203}, e.g., for $n=800$ this construction is guaranteed to work with an indegree up to $3$ and for $n=1500$ with an indegree up to $4$. This can be utilized if one is interested in learning sparse CP-nets over a large number of variables. 


Finally, a simple adaptation of Algorithm 2 efficiently learns all $k$-bounded acyclic CP-nets, whether complete or incomplete. As in the proof of Theorem~\ref{thm:treePerfectIncomplete}, this comes at a cost of a factor of 2 in the number of queries. The proof is easily adapted from that of Theorem~\ref{thm:boundedPerfect}.

\begin{theorem}\label{thm:boundedPerfectIncomplete}
Let $n\ge 1$, $m=2$, $0\le k\le n-1$. There exists an algorithm that learns every (complete or incomplete) binary $k$-bounded CP-net $N^*\in\overline{\mathcal{C}}_{ac}^k$ over $n$ variables with 
\[
O(n\mathcal{U}_k+e_{N^*}\log_2(n))
\] 
membership queries over swap examples in $\overline{\mathcal{X}}_{swap}$, where $e_{N^*}$ is the number of edges in $N^*$. 
\end{theorem}

We can now use our results on the VC dimension to compare the query consumption of our algorithms to the sample complexity of PAC learning. PAC learning~\cite{V84} is a model of learning from randomly chosen examples, in which a target concept is approximated with high confidence. By comparison, our algorithms are \emph{guaranteed} to learn the target CP-net \emph{exactly}, by selecting examples non-randomly.

The VC dimension of a concept class $\mathcal{C}$ is known to characterize the asymptotic sample complexity for PAC-learning $\mathcal{C}$; in particular, the number of examples required for learning with error at most $\epsilon$ and confidence at least $1-\delta$ is in $\Theta(\frac{1}{\epsilon}(\VCdim(\mathcal{C})+\ln(\frac{1}{\delta})))$~\cite{BEHW89,Hanneke16}. That means, for tree CP-nets, using our result $\VCdim \ge 2n-1$, the PAC sample complexity is in $\Omega(\frac{1}{\epsilon}(n+\ln(\frac{1}{\delta})))$, which is asymptotically at most a factor of $\log(n)$ below the query complexity of our algorithm, which amounts to $2n+e\log(n)$, where $e\le n-1$. For $k$-bounded acyclic CP-nets, Theorem~\ref{thm:VCD} yields a VC dimension of at least $(n-k)2^k+2^k-1$, resulting in a PAC sample complexity of $\Omega(\frac{1}{\epsilon}((n-k)2^k+\ln(\frac{1}{\delta})))$, while our query complexity is $n2^k\log(n)k^{O(\log(k))}$. That means asymptotically, for constant $k$, a difference of no more than a factor of $\log(n)$, as in the tree case.

The query complexity of our algorithm improves on the best upper bound on the membership query complexity claimed in the literature so far: Chevaleyre et al.~\cite{qqq} stated that $k$-bounded acyclic CP-nets can be learned with $O(kn^{k+1}2^{k-1})$ membership queries, which asymptotically exceeds our query complexity by a factor of $\frac{n^k}{k^{O(\log(k))}}$.

\section{Learning from Corrupted Membership Queries}
\label{sec:corrupted}

So far, we have assumed that all membership queries are answered correctly. This assumption is unrealistic in many settings, especially when it comes to dealing with human experts that may have incomplete knowledge of the target function. For instance, when eliciting CP-nets from users, it could be the case that the user actually does not know which of two given outcomes is to be preferred or does not want to respond to a query. When recommender systems explicitly request information from users, be it in the form of item ratings or in the form of preference queries, not all queries will be answered and some will be answered incorrectly. 

In this section, we consider two models of learning from membership queries in which the oracle does not always provide a correct answer. In particular, we consider the situation in which there is a fixed set $L$ of instances $x$ for which the oracle does not provide the true classification $c^*(x)$. The set $L$ is assumed to be chosen in advance by an adversary. There are two ways in which the oracle could deal with queries to elements in $L$:
\begin{itemize}
\item A \emph{limited oracle}\/ returns \quotes{I don't know} (denoted by $\perp$) when queried for any $x\in L$, and returns the true label $c^*(x)$ for any $x\not\in L$ \cite{Angluin1997,DBLP:journals/jcss/BishtBK08}. 
\item A \emph{malicious oracle}\/ returns the wrong label $1-c^*(x)$ when queried for any $x\in L$, and returns the true label $c^*(x)$ for any $x\not\in L$ \cite{Angluin1997,DBLP:journals/tcs/BennetB07}.
\end{itemize}
In either case, the oracle is persistent in the sense that it will return the same answer every time the same instance is queried.  A concept class $\mathcal{C}$ is learnable with membership queries to a limited (malicious, resp.)\ oracle if there is an algorithm that exactly identifies any target concept $c^*\in\mathcal{C}$ by asking a limited (malicious, resp.)\ oracle a number of membership queries that is polynomial in $n$, the size of $c^*$, and $|L|$ \cite{Angluin1997,DBLP:journals/jcss/BishtBK08}\footnote{This corresponds to the strict learning model discussed in \cite{Angluin1997,DBLP:journals/jcss/BishtBK08}.}. 

Consider for instance the application scenario of a recommender system that explicitly asks users for preferences over pairs of items. The limited oracle then corresponds to a user who chooses to ignore some queries or who does not know which of the two presented items to prefer (perhaps since the user doesn't know the items well enough to express a preference.) The malicious oracle is a simplified model of a user who misunderstands some queries and answers them incorrectly, or who gets annoyed with queries and answers some of them on purpose incorrectly. 

At first, the assumption that the oracle is persistent may seem unrealistic. However, in practice, this assumption translates into never asking the same query twice, since the system cannot possibly gain information from repeating queries to a persistent user. Repeating queries to the user is often undesirable, as it bears a high risk of annoying the user.

For practical applications, one might prefer models of probabilistic noise and learning algorithms that infer approximations to the true target based on randomness in the noise. Instead, we pursue the idealized goal of exactly identifying the target concept in the face of partial or incorrect information, which is useful for studying strict limitations of learning. 

It is obvious that exact learning is not possible unless some restrictions are made on the set $L$. Note though that we assume that the learner has no prior information on the size or content of $L$. Instead, we will study under which conditions on $L$ learning complete acyclic CP-nets from limited or malicious oracles is possible even if the learner has no information on these conditions. 

We will restrict our analysis to the case of binary CP-nets. Further, we will make use of the trivial observation that exact learnability with membership queries to a malicious oracle implies exact learnability with membership queries to a limited oracle.

\subsection{Limitations on the Corrupted Set}

We first establish that learning the class $\mathcal{C}_{ac}^k$ of all complete acyclic binary $k$-bounded CP-nets over $n$ variables is impossible, both from malicious and from limited oracles, when $|L|\ge 2^{n-1-k}$.

\begin{proposition}
	Let $m=2$ and $1\le k\le n-1$. If $|L|\geq 2^{n-1-k}$, then the class $\mathcal{C}_{ac}^k$ is not learnable with membership queries to a limited oracle and not learnable with membership queries to a malicious oracle.
\end{proposition}
\begin{proof}
It suffices to prove non-learnability from limited oracles.

Consider a CP-net $N$ in which all the variables are unconditional, except for one variable $v_i$, which has exactly $k$ parents. Furthermore, let $\CPT(v_i)$ impose the same order on the domain of $v_i$ for all contexts over $Pa(v_i)$, except for one context $\gamma$ over $Pa(v_i)$ for which $\CPT(v_i)$ imposes the reverse order over the domain of $v_i$. Figure \ref{fig:corruptionLimitExample} shows an example of such a CP-net. Consider the set $X=\{x=(x.1,x.2)\in \mathcal{X}_{swap}\mid V(x)=v_i$ and $x[Pa(v_i)]=\gamma\}$. Each swap instance $x$ in $X$ has the same values in the $k$ parent variables of $v_i$, and there are no choices for its values in the binary swap variable $v_i$ either. For each of the remaining $n-1-k$ variables, there are $m=2$ choices. Hence, the cardinality of $X$ is $2^{n-1-k}$. Since $|L|\geq 2^{n-1-k}$, the adversary can choose $L$ so that $X\subseteq L$.

Then, in communication with a limited membership oracle, the learner will not be able to distinguish $N$ from the CP-net $N'$ that is equivalent to $N$ except for having the orders over the domain of $v_i$ swapped in $\CPT(v_i)$.
\end{proof}
		
In essence, this negative result is due to the fact that the number of instances supporting a statement in a CPT with $k$ parents is $2^{n-1-k}$---having all these corrupted makes learning hopeless. In particular, when $k=n-1$, i.e., for unbounded CP-nets, even a single corrupted query will make learning impossible. This is because a $\CPT$ of a variable with $n-1$ parents may have statements that are witnessed by only a single context, and thus by only a single swap pair $x$. The adversary can choose $L=\{x\}$ and not respond to the query $x$, so that the learner cannot figure out the correct preference for the swap pair $x$. 

In the remainder of this section, we will study assumptions on the structure of $L$ that will allow a learning algorithm to overcome the corrupted answers from limited or malicious oracles at the expense of a bounded number of additional queries. That means, we will constrain the adversary in its options for selecting $L$, without directly limiting the size of $L$. Roughly speaking, we no longer look at the worst possible choice of the set $L$, but at a case where the corrupted responses are ``reasonably spread out.''

The goal is to be able to obtain the correct answer for any query $x$ made by our algorithms that learn from perfect oracles, simply by taking majority votes over a bounded number of additional ``verification queries.'' If we know that the corrupted oracle will affect only a small subset of these verification queries (small enough so that the majority vote over them is guaranteed to yield the correct label for $x$,) we can use the same learning procedures as in the perfect oracle case, supplemented by a bounded number of verification queries.

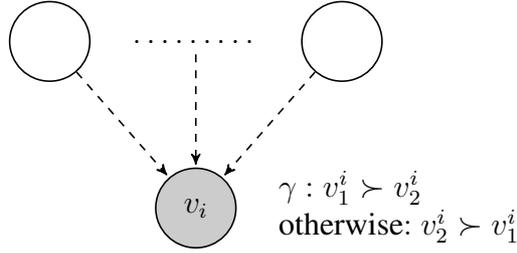
\begin{figure}
	\centering
	\begin{tikzpicture}[->,>=stealth',shorten >=1pt,auto,node distance=1.5cm and .4cm,semithick,scale=1]
	\node[state] (a) {};
	\node[right= of a] (dots) {$\dots\dots\dots$};
	\node[state,right= of dots](end) {};
	\node[state,fill=black!20,below=of dots](vi){$v_i$};
	\node[right =of vi](viCPT){\makecell[l]
		{
			$\gamma:v^i_1\succ v^i_2$\\ otherwise$: v^i_2\succ v^i_1$
		}
	};
	\path (a) edge[dashed] (vi)
	(dots) edge[dashed] (vi)
	(end) edge[dashed] (vi);
	\end{tikzpicture}
	
	\caption{A CP-net $N\in \mathcal{C}_{ac}^k$ that, for $|L|\ge2^{n-1-k}$, cannot be distinguished from any CP-net $N'$ that differs from $N$ only in $\CPT(v_i)$.}
	\label{fig:corruptionLimitExample}
\end{figure}

Suppose that, for each $x\in\mathcal{X}_{swap}$, we could effectively compute a small set $\VQ(x)\subseteq\mathcal{X}_{swap}$ such that the limited/malicious oracle would be guaranteed to return the correct label for $x$ on more than half of the queries for elements in $\VQ(x)$. In the case of membership queries to a limited oracle, we could then simulate Algorithms 1 or 2 with the following modification:
\begin{itemize}
\item[LIM] If a query for $x\in\mathcal{X}_{swap}$ made by Algorithm 1 (or 2) is answered with $\perp$, replace this response by the majority vote of the limited oracle's responses to all queries over the set $\VQ(x)$.
\end{itemize}
In the case of learning from malicious oracles, \emph{every}\/ query made by Algorithms 1 or 2 would have to be supplemented by verification queries:
\begin{itemize}
\item[MAL] If a query for $x\in\mathcal{X}_{swap}$ is made by Algorithm 1 (or 2), respond to that query by taking the majority vote over the malicious oracle's responses to all queries over the set $\VQ(x)$.
\end{itemize}
If $q$ is an upper bound on the size of the set $\VQ(x)$, for any $x\in\mathcal{X}_{swap}$, the modified algorithms then would need to ask at most $qz$ membership queries, where $z$ is the number of queries asked by Algorithms 1 or 2.

It remains to find a suitable set $\VQ(x)$ of verification queries for any swap pair $x$, so that, intuitively, 
\begin{itemize}
\item the size of $\VQ(x)$ is not too large, and 
\item it is not too unreasonable an assumption that the corrupted oracle will return the true label for $x$ on the majority of the swap pairs in $\VQ(x)$. 
\end{itemize}
To this end, we introduce some notation.

\begin{definition} Let $x=(x.1,x.2)\in\mathcal{X}_{swap}$ and $1\le t\le n-1$. Then we denote by $F^t(x)$ the set of all swap instances $x'=(x'.1,x'.2)$ with the swapped variable $V(x')=V(x)$ and with a Hamming distance\footnote{The Hamming distance \cite{Hamming50} between two vectors is the number of components in which they disagree.} of exactly $t$ between $x.i$ and $x'.i$ when restricted to $V\setminus\{V(x)\}$, i.e.
\[
F^t(x)=\{x'\in \mathcal{X}_{swap}\mid V(x)=V(x')=v_s\mbox{ and }|\{v\in V\setminus\{v_s\}\mid x[\{v\}]\ne x'[\{v\}]\}|= t\}\,.
\]
\end{definition}

For example, if $x=(abcd,a\bar{b}cd)$ where $D_{v_1}=\{a,\bar{a}\}$, $D_{v_2}=\{b,\bar{b}\}$, $D_{v_3}=\{c,\bar{c}\}$, $D_{v_4}=\{d,\bar{d}\}$, then for $t=1$ we get
\[
F^1(x)=\{(\bar{a}bcd,\bar{a}\bar{b}cd),(ab\bar{c}d,a\bar{b}\bar{c}d),(abc\bar{d},a\bar{b}c\bar{d})\}\,.
\]
Each swap in that set has the same swapped variable as $x$ (namely $v_2$) and has a context that differs from that in $x$ in exactly one variable. 

There is a relationship between the entailment of an instance $x$ and the entailments of the elements of $F^t(x)$ for any $t$: Given a preference table $\CPT(v_i)$, where $v_i$ has $k$ parents, and given $t$, it is not hard to see that 
\begin{itemize}
\item $|F^t(x)|=\binom{n-1}{t}$, and 
\item the set $F^t(x)$ contains $\binom{n-1-k}{t}$ elements with the same entailment w.r.t.\ $V(x)$ as in $x$ itself. These are the instances that share the same values in the parent variables of $V(x)$ and, hence, their entailments have to be identical. 
\end{itemize}
If most of the elements in $F^t(x)$ share the same entailment, they could be used to compensate the oracle corruption. However, queries to elements of $F^t(x)$ again might receive corrupted answers. We will therefore impose a restriction on the overlap between the set $L$ of instances with corrupted responses and any set $F^t(x)$, specifically for $t=1$.

We would like to constrain $L$ so that querying the elements of $F^1(x)$, and then picking the most frequent answer, yields the true classification of instance $x$. 
Figure \ref{fig:f1example} shows an example of an instance $x$ with $V(x)=v_5$, its entailment and the entailments of the elements of $F^1(x)$ assuming $k=1$ and $Pa(v_5)=\{v_1\}$. 

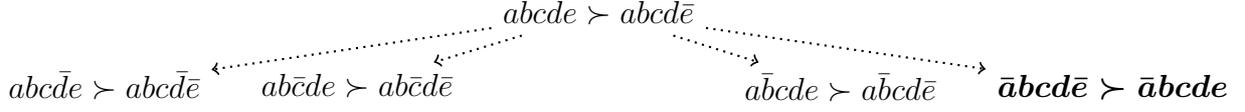
\begin{figure}
	\centering 
	\begin{adjustbox}{max width=\textwidth}
		\begin{tikzpicture}[thick,dotted, node distance=.5cm]
		\node(query){$abcde\succ abcd\bar{e}$};
		\node[below left=of query](cflipped){$ab\bar{c}de\succ ab\bar{c}d\bar{e}$};
		\node[left=of cflipped](dflipped){$abc\bar{d}e\succ abc\bar{d}\bar{e}$};
		\node[below right=of query](bflipped){$a\bar{b}cde\succ a\bar{b}cd\bar{e}$};
		\node[font=\boldmath,right=of bflipped](aflipped){$\bar{a}bcd\bar{e}\succ \bar{a}bcde$};
		\path (query) edge[->](dflipped) edge[->] (cflipped) edge[->](bflipped) edge[->](aflipped);
		\end{tikzpicture}
	\end{adjustbox}
	\caption{An example of the entailments of an instance $x$ and $F^1(x)$ for $x=(abcde,abcd\bar{e})$ with $Pa(v_5)=\{v_1\}$ and $\CPT(v_5)$ is $\{a:e\succ \bar{e},\bar{a}:\bar{e}\succ e\}$.}
	\label{fig:f1example}
\end{figure}

We then obtain the following learnability result for the case that the indegree $k$ is bounded to be sufficiently small.

\begin{theorem}\label{thm:corrupted} Suppose $n>2k+2$. For each swap pair $x\in\mathcal{X}_{swap}$, let the set $\VQ(x)$ of verification queries used by the LIM and MAL strategies be equal to $F^1(x)$.
\begin{enumerate}
\item If $|F^1(x) \cap L|\leq n-2-2k$ for every $x\in\mathcal{X}_{swap}$, then the strategy LIM will learn any complete acyclic $k$-bounded binary CP-net over $n$ variables, when interacting with a limited oracle.
\item If $|F^1(x)\cap L|\le \lfloor \frac{n-1}{2}\rfloor -k-1$ for every $x\in\mathcal{X}_{swap}$, learn any complete acyclic $k$-bounded binary CP-net over $n$ variables, when interacting with a malicious oracle.
\end{enumerate}
\end{theorem}

\begin{proof} Note that there are $n-1$ elements in $F^1(x)$, of which at least $n-1-k$ have the same label in the target concept as $x$ and at most $k$ have the opposite label.

First, suppose $|F^1(x) \cap L|\leq n-2-2k$ for every $x\in\mathcal{X}_{swap}$, in the case of a limited oracle. Then, for every $x\in\mathcal{X}_{swap}$, the limited oracle will respond with a label to at least $|F^1(x)|-n+2+2k=2k+1$ of the queries for elements in $F^1(x)$. Since at most $k$ of these elements have the opposite label as $x$, the majority of these queries will return the label for $x$.  Therefore, the strategy LIM will simulate Algorithm~2 with a perfect membership oracle and thus learn any complete acyclic $k$-bounded binary CP-net over $n$ variables.

Second, suppose $|F^1(x)\cap L|\le \lfloor \frac{n-1}{2}\rfloor$ for every $x\in\mathcal{X}_{swap}$, in the case of a malicious oracle. Then, for every $x\in\mathcal{X}_{swap}$, the limited oracle will correctly respond to at least $|F^1(x)|-\lfloor \frac{n-1}{2}\rfloor+k+1= \lceil\frac{n-1}{2}\rceil +k+1$ of the queries for elements in $F^1(x)$. In the worst case, these $\lceil\frac{n-1}{2}\rceil +k+1$ correctly answered queries contain all of the $k$ elements of $F^1(x)$ that have the opposite label of $x$. That means that at least $\lceil\frac{n-1}{2}\rceil +1$ of the $n-1$ queries over $F^1(x)$ (and thus a majority) return the correct label for $x$. Therefore, the strategy MAL will simulate Algorithm~2 with a perfect membership oracle and thus learn any complete acyclic $k$-bounded binary CP-net over $n$ variables.
\end{proof}

We obtain the following corollary.

\begin{corollary} Suppose $n>2k+2$. 
\begin{enumerate}
\item If $|F^1(x) \cap L|\leq n-2-2k$ for every $x\in\mathcal{X}_{swap}$, then the class of all complete acyclic $k$-bounded binary CP-nets over $n$ variables is learnable with membership queries to a limited oracle. 
\item If $|F^1(x)\cap L|\le \lfloor \frac{n-1}{2}\rfloor -k-1$ for every $x\in\mathcal{X}_{swap}$, then the class of all complete acyclic $k$-bounded binary CP-nets over $n$ variables is learnable with membership queries to a malicious oracle. 
\end{enumerate}
In either case, the worst-case number of queries can be upper-bounded by $O(n^2\mathcal{U}_k+e_{N^*}\log_2(n))$, which is in $O(n^22^k\log_2(n)k^{O(\log_2(k))}+e_{N^*}n\log_2(n))$.
\end{corollary}

\begin{proof}
Learnability follows from Theorem~\ref{thm:corrupted}. The bound on the number of queries is derived from the bound in Theorem~\ref{thm:boundedPerfect} by multiplying with a factor of $O(n)$ due to the verification query set of size $\binom{n-1}{1}=n-1$.
\end{proof}

In the special case of learning tree CP-nets, i.e., when $k=1$, we obtain the following improved result.

\begin{corollary} Suppose $n>4$. 
\begin{enumerate}
\item If $|F^1(x) \cap L|\leq n-4$ for every $x\in\mathcal{X}_{swap}$, then the class of all complete acyclic binary tree CP-nets over $n$ variables is learnable with membership queries to a limited oracle. 
\item If $|F^1(x)\cap L|\le \lfloor \frac{n-1}{2}\rfloor -2$ for every $x\in\mathcal{X}_{swap}$, then the class of all complete acyclic binary tree CP-nets over $n$ variables is learnable with membership queries to a malicious oracle. 
\end{enumerate}
In particular, in either case, the worst-case number of queries is in $O(n^2+e_{N^*}n\log_2(n))$, where $e_{N^*}$ is the number of edges in the target CP-net $N^*$.
\end{corollary}

\begin{proof}
Learnability follows from Theorem~\ref{thm:corrupted}. The bound on the number of queries is derived from the bound in Theorem~\ref{thm:treePerfect} by multiplying with a factor of $O(n)$ due to the verification query set of size $\binom{n-1}{1}=n-1$.
\end{proof}

\section{Conclusion}

We determined exact values or non-trivial bounds on the parameters $\VCdim$ and $\TD$ 
for the classes of all (all complete, resp.) $k$-bounded acyclic CP-nets for any $k$, and used some of the insights gained thereby for the design of algorithms for learning CP-nets from membership queries, both in a setting where membership queries are always answered correctly and in settings where some answers may be missing or incorrect. 
The $\VCdim$ values we determined 
correct a mistake in \cite{Koriche2010685}. Further, we used the calculated $\TD$ values in order to show that our proposed algorithm for learning complete tree CP-nets from membership queries alone is close to optimal, and the calculated $\VCdim$ values in order to refine an optimality assessment of Koriche and Zanuttini's algorithm for learning all $k$-bounded acyclic CP-nets from equivalence and membership queries.

All our results were obtained over the swap instance space. When using membership queries, a restriction of the instance space to swaps actually limits the options of the algorithm; one may want to investigate whether our algorithms can be made more efficient when given access to the full set of outcome pairs. At the same time, lower bounds on the query complexity (as obtained via the teaching dimension) might decrease when opening the choice to non-swaps.

For fixed $k$, the number of membership queries made by our learning algorithms grows moderately with $n$, with an asymptotic worst-case behavior of $O(n\log_2(n))$, (to be multiplied by $n$ in the corrupted setting), 
which is at most a factor of $O(\log_2(n))$ above the best achievable, and suggests practical feasibility. In practice, running our algorithms may become prohibitive when $k$ grows too large---as is also evident in the complexity parameters we calculated for the unbounded case, i.e., when $k=n-1$. Intuitively though, learning preferences that depend on a large number of parent variables is a problem that one might prefer to approach with approximate learning. When the dependency structure between the variables becomes so complex that unbounded CP-nets are needed to model them, one might in practice worry about overfitting the observed preferences---a simpler yet approximate CP-net may in such cases be a more desirable learning target. Knowledge of the number of queries required for exact identification of $\CPT$s and CP-nets will then also be useful in the design of heuristic approaches for approximating a target $\CPT$ or CP-net.


On top of the $\VCdim$ and the $\TD$, we calculated another parameter, $\RTD$, for classes of all complete $k$-bounded acyclic CP-nets. While we did not use our $\RTD$ results for optimality assessments of proposed learning algorithms, they are of interest to learning-theoretic studies, as they exhibit, with the class $\mathcal{C}_{ac}^{n-1}$ of all complete unbounded acyclic CP-nets, the first known non-maximum class as well as the first not intersection-closed class that is interesting from an application point of view and satisfies $\RTD=\VCdim$, cf.\ Appendix\ref{sec:struct}. Thus further studies on the structure of CP-nets may be helpful toward the solution of an open problem concerning the general relationship between $\RTD$ and $\VCdim$ \cite{SimonZ15}.

\section*{Acknowledgements}
The authors would like to express their gratitude to the reviewers, whose detailed comments  helped to improve the quality of this article. 

Malek Mouhoub and Sandra Zilles were supported by the Discovery Grant program of the Natural Sciences and Engineering Research Council of Canada (NSERC). Sandra Zilles was also supported by the NSERC Canada Research Chairs program.

\bibliographystyle{elsarticle-num}
\bibliography{references}

\begin{thebibliography}{10}
\expandafter\ifx\csname url\endcsname\relax
  \def\url#1{\texttt{#1}}\fi
\expandafter\ifx\csname urlprefix\endcsname\relax\def\urlprefix{URL }\fi
\expandafter\ifx\csname href\endcsname\relax
  \def\href#1#2{#2} \def\path#1{#1}\fi

\bibitem{Fuernkranz2011}
J.~F{\"u}rnkranz, E.~H{\"u}llermeier, Preference Learning: An Introduction,
  Springer Berlin Heidelberg, Berlin, Heidelberg, 2011, pp. 1--17.

\bibitem{keeney1993decisions}
R.~L. Keeney, H.~Raiffa, Decisions with multiple objectives: preferences and
  value trade-offs, Cambridge University Press, 1993.

\bibitem{DOMSHLAK20111037}
C.~Domshlak, E.~Hüllermeier, S.~Kaci, H.~Prade, Preferences in {AI}: An
  overview, Artif. Intell. 175~(7) (2011) 1037--1052.

\bibitem{AllenSG17}
T.~E. Allen, C.~Siler, J.~Goldsmith, Learning tree-structured {CP}-nets with
  local search, in: Proceedings of the Thirtieth International Florida
  Artificial Intelligence Research Society Conference, 2017, pp. 8--13.

\bibitem{Booth:2010:LCL:1860967.1861021}
R.~Booth, Y.~Chevaleyre, J.~Lang, J.~Mengin, C.~Sombattheera, Learning
  conditionally lexicographic preference relations, in: ECAI, 2010, pp.
  269--274.

\bibitem{1-DBLP:journals/jair/BoutilierBDHP04}
C.~Boutilier, R.~I. Brafman, C.~Domshlak, H.~H. Hoos, D.~Poole, {CP}-nets: A
  tool for representing and reasoning with conditional ceteris paribus
  preference statements, J.\ Artif.\ Intell.\ Res. 21 (2004) 135--191.

\bibitem{Koriche2010685}
F.~Koriche, B.~Zanuttini, Learning conditional preference networks, Artificial
  Intelligence 174~(11) (2010) 685--703.

\bibitem{DBLP:conf/ijcai/DimopoulosMA09}
Y.~Dimopoulos, L.~Michael, F.~Athienitou, Ceteris paribus preference
  elicitation with predictive guarantees, in: IJCAI, 2009, pp. 1890--1895.

\bibitem{Lang:2009:CLS:1661445.1661580}
J.~Lang, J.~Mengin, The complexity of learning separable ceteris paribus
  preferences, in: IJCAI, 2009, pp. 848--853.

\bibitem{10.1109/TKDE.2012.231}
J.~Liu, Y.~Xiong, C.~Wu, Z.~Yao, W.~Liu, Learning conditional preference
  networks from inconsistent examples, IEEE Transactions on Knowledge and Data
  Engineering 99 (2012) 1.

\bibitem{www}
J.~T. Guerin, T.~E. Allen, J.~Goldsmith, Learning {CP}-net preferences online
  from user queries, in: ADT, 2013, pp. 208--220.

\bibitem{Angluin:1988:QCL:639961.639995}
D.~Angluin, Queries and concept learning, Machine Learning 2~(4) (1988)
  319--342.

\bibitem{LaberniaZMYA17}
F.~Labernia, B.~Zanuttini, B.~Mayag, F.~Yger, J.~Atif, Online learning of
  acyclic conditional preference networks from noisy data, in: {IEEE}
  International Conference on Data Mining, {ICDM}, 2017, pp. 247--256.

\bibitem{BlumJSZ04}
A.~Blum, J.~C. Jackson, T.~Sandholm, M.~Zinkevich, Preference elicitation and
  query learning, Journal of Machine Learning Research 5 (2004) 649--667.

\bibitem{BoutilierRV09}
C.~Boutilier, K.~Regan, P.~Viappiani, Preference elicitation with subjective
  features, in: Proceedings of the 2009 {ACM} Conference on Recommender Systems
  (RecSys), 2009, pp. 341--344.

\bibitem{LahaieP04}
S.~Lahaie, D.~C. Parkes, Applying learning algorithms to preference
  elicitation, in: Proceedings of the 5th {ACM} Conference on Electronic
  Commerce (EC-2004), 2004, pp. 180--188.

\bibitem{ZinkevichBS03}
M.~Zinkevich, A.~Blum, T.~Sandholm, On polynomial-time preference elicitation
  with value queries, in: Proceedings of the 4th {ACM} Conference on Electronic
  Commerce (EC-2003), 2003, pp. 176--185.

\bibitem{VC71}
V.~N. Vapnik, A.~Y. Chervonenkis, On the uniform convergence of relative
  frequencies of events to their probabilities, Theory of Probability and its
  Applications 16~(2) (1971) 264--280.

\bibitem{AuerL99}
P.~Auer, P.~Long, Structural results about on-line learning models with and
  without queries, Machine Learning 36~(3) (1999) 147–--181.

\bibitem{GK95}
S.~A. Goldman, M.~J. Kearns, On the complexity of teaching, Journal of Computer
  and System Sciences 50 (1995) 20--31.

\bibitem{queriesRevisited}
D.~Angluin, Queries revisited, in: N.~Abe, R.~Khardon, T.~Zeugmann (Eds.),
  Algorithmic Learning Theory, Vol. 2225 of Lecture Notes in Computer Science,
  Springer Berlin Heidelberg, 2001, pp. 12--31.

\bibitem{labernia2016query}
F.~Labernia, F.~Yger, B.~Mayag, J.~Atif, Query-based learning of acyclic
  conditional preference networks from noisy data, in: DA2PL, 2016.

\bibitem{LaberniaYMA18}
F.~Labernia, F.~Yger, B.~Mayag, J.~Atif, Query-based learning of acyclic
  conditional preference networks from contradictory preferences, EURO J Decis
  Process 6 (2018) 39--59.

\bibitem{qqq}
Y.~Chevaleyre, F.~Koriche, J.~Lang, J.~Megine, B.~Zanuttini, Learning ordinal
  preferences on multiattribute domains: the case of {CP}-nets, in: Preference
  Learning, Springer-Verlag, 2010, pp. 273--296.

\bibitem{ZLHZ11}
S.~Zilles, S.~Lange, R.~Holte, M.~Zinkevich, Models of cooperative teaching and
  learning, Journal of Machine Learning Research 12 (2011) 349--384.

\bibitem{AMZ16}
E.~Alanazi, M.~Mouhoub, S.~Zilles, The complexity of learning acyclic
  {CP}-nets, in: Proceedings of the Twenty-Fifth International Joint Conference
  on Artificial Intelligence (IJCAI), 2016, pp. 1361--1367.

\bibitem{ChajewskaKO01}
U.~Chajewska, D.~Koller, D.~Ormoneit, Learning an agent's utility function by
  observing behavior, in: Proceedings of the Eighteenth International
  Conference on Machine Learning ({ICML} 2001), 2001, pp. 35--42.

\bibitem{ErculianiDTP18}
L.~Erculiani, P.~Dragone, S.~Teso, A.~Passerini, Automating layout synthesis
  with constructive preference elicitation, in: Proceedings of the European
  Conference on Machine Learning and Knowledge Discovery in Databases
  ({ECML}-{PKDD}), 2018, pp. 254--270.

\bibitem{DeryKRS16}
L.~N. Dery, M.~Kalech, L.~Rokach, B.~Shapira, Reducing preference elicitation
  in group decision making, Expert Syst. Appl. 61 (2016) 246--261.

\bibitem{AngluinHK93}
D.~Angluin, L.~Hellerstein, M.~Karpinski, Learning read-once formulas with
  queries, J.\ ACM 40 (1993) 185--210.

\bibitem{DBLP:conf/ijcai/KoricheZ09}
F.~Koriche, B.~Zanuttini, Learning conditional preference networks with
  queries, in: IJCAI, 2009, pp. 1930--1935.

\bibitem{Liu20137}
J.~Liu, Z.~Yao, Y.~Xiong, W.~Liu, C.~Wu, Learning conditional preference
  network from noisy samples using hypothesis testing, Knowledge-Based Systems
  40 (2013) 7--16.

\bibitem{DBLP:conf/ijcai/MichaelP13}
L.~Michael, E.~Papageorgiou, An empirical investigation of ceteris paribus
  learnability, in: IJCAI, 2013.

\bibitem{DBLP:conf/stairs/BigotMZ14}
D.~Bigot, J.~Mengin, B.~Zanuttini, Learning probabilistic {CP}-nets from
  observations of optimal items, in: Proc.\ 7th European Starting {AI}
  Researcher Symposium (STAIRS), 2014, pp. 81--90.

\bibitem{HaqqaniL17}
M.~Haqqani, X.~Li, An evolutionary approach for learning conditional preference
  networks from inconsistent examples, in: Proceedings of the 13th
  International Conference on Advanced Data Mining and Applications, {ADMA}
  2017, 2017, pp. 502--515.

\bibitem{Angluin1997}
D.~Angluin, M.~Krikis, R.~H. Sloan, G.~Tur{\'{a}}n, Malicious omissions and
  errors in answers to membership queries, Machine Learning 28~(2-3) (1997)
  211--255.

\bibitem{DBLP:journals/tcs/BennetB07}
R.~Bennet, N.~H. Bshouty, Learning attribute-efficiently with corrupt oracles,
  Theor. Comput. Sci. 387~(1) (2007) 32--50.

\bibitem{DBLP:journals/jcss/BishtBK08}
L.~Bisht, N.~H. Bshouty, L.~Khoury, Learning with errors in answers to
  membership queries, J. Comput. Syst. Sci. 74~(1) (2008) 2--15.

\bibitem{Hanneke16}
S.~Hanneke, The optimal sample complexity of {PAC} learning, Journal of Machine
  Learning Research 17 (2016) 38:1--38:15.

\bibitem{DFSZ14}
T.~Doliwa, G.~Fan, H.~U. Simon, S.~Zilles, Recursive teaching dimension,
  {VC}-dimension and sample compression, Journal of Machine Learning Research
  15 (2014) 3107--3131.

\bibitem{SimonZ15}
H.~U. Simon, S.~Zilles, Open problem: Recursive teaching dimension versus {VC}
  dimension, in: COLT, 2015, pp. 1770--1772.

\bibitem{Zhu15}
X.~Zhu, Machine teaching: An inverse problem to machine learning and an
  approach toward optimal education, in: Proceedings of the Twenty-Ninth {AAAI}
  Conference on Artificial Intelligence, 2015, pp. 4083--4087.

\bibitem{ZhuSZR18}
X.~Zhu, A.~Singla, S.~Zilles, A.~N. Rafferty, An overview of machine teaching,
  CoRR abs/1801.05927.

\bibitem{HuWLW17}
L.~Hu, R.~Wu, T.~Li, L.~Wang, Quadratic upper bound for recursive teaching
  dimension of finite {VC} classes, in: COLT, 2017.

\bibitem{3-DBLP:journals/jair/GoldsmithLTW08}
J.~Goldsmith, J.~Lang, M.~Truszczynski, N.~Wilson, The computational complexity
  of dominance and consistency in {CP}-nets, J. Artif. Intell. Res. 33 (2008)
  403--432.

\bibitem{DBLP:conf/uai/BigotZFM13}
D.~Bigot, B.~Zanuttini, H.~Fargier, J.~Mengin, Probabilistic conditional
  preference networks, in: Proceedings of the Twenty-Ninth Conference on
  Uncertainty in Artificial Intelligence ({UAI}), 2013.

\bibitem{Kearns:1994:ICL:200548}
M.~J. Kearns, U.~V. Vazirani, An Introduction to Computational Learning Theory,
  MIT Press, 1994.

\bibitem{Littlestone88}
N.~Littlestone, Learning quickly when irrelevant attributes abound: a new
  linear threshold algorithm, Machine Learning 2~(4) (1988) 245--–318.

\bibitem{V84}
L.~G. Valiant, A theory of the learnable, Commun. ACM 27~(11) (1984)
  1134--1142.

\bibitem{BEHW89}
A.~Blumer, A.~Ehrenfeucht, D.~Haussler, M.~K. Warmuth, Learnability and the
  {V}apnik-{C}hervonenkis dimension, Journal of the ACM 36~(4) (1989) 929--965.

\bibitem{SM91}
A.~Shinohara, S.~Miyano, Teachability in computational learning, New Generation
  Computing 8~(4) (1991) 337--347.

\bibitem{Sau72}
N.~Sauer, On the density of families of sets, Journal of Combinatorial Theory,
  Series A 13~(1) (1972) 145--147.

\bibitem{RR12}
B.~I.~P. Rubinstein, J.~H. Rubinstein, A geometric approach to sample
  compression, Journal of Machine Learning Research 13 (2012) 1221--1261.

\bibitem{adaptiveLearning}
P.~Damaschke, Adaptive versus nonadaptive attribute-efficient learning, Mach.\
  Learn. 41~(2) (2000) 197--215.

\bibitem{Jukna:2010:ECA:1965203}
S.~Jukna, Extremal Combinatorics: With Applications in Computer Science, 1st
  Edition, Springer Publishing Company, Incorporated, 2010.

\bibitem{Angluin87}
D.~Angluin, Learning regular languages from queries and counterexamples, Inf.\
  Comput. 75 (1987) 87--106.

\bibitem{DBLP:journals/tit/SeroussiB88}
G.~Seroussi, N.~H. Bshouty, Vector sets for exhaustive testing of logic
  circuits, {IEEE} Trans. Information Theory 34~(3) (1988) 513--522.

\bibitem{Hamming50}
R.~W. Hamming, Error detecting and error correcting codes, The Bell System
  Technical Journal 29 (1950) 147--160.

\bibitem{MoranW16}
S.~Moran, M.~K. Warmuth, Labeled compression schemes for extremal classes, in:
  Proceedings of the 27th International Conference on Algorithmic Learning
  Theory, {ALT}, 2016, pp. 34--49.

\end{thebibliography}
\appendix

\section{Revising a Lower Bound on $\VCdim$ from \cite{Koriche2010685}}\label{sec:app}


Let $\overline{\mathcal{C}}^{2,k,e}_{ac}$ be the class of all binary acyclic $k$-bounded CP-nets with at most $e$ edges, where $0\leq k < n$ and $k\leq e \leq \binom{n}{2}$ (note that this class includes both complete and incomplete CP-nets.) Koriche and Zanuttini~\cite[Theorem 6]{Koriche2010685} gave a lower bound on $\VCdim(\overline{\mathcal{C}}^{2,k,e}_{ac})$ over the swap instance space. In particular, setting $u=\lfloor \frac{e}{k}\rfloor$ and $r=\lfloor\log_2\frac{n-u}{k}\rfloor$, they claimed that  $\VCdim(\overline{\mathcal{C}}^{2,k,e}_{ac})$ is lower-bounded by 
\[
	\LB=\begin{cases}	1\,,&\mbox{if }k=0\,,\\
		u(r+1)\,,&\mbox{if }k=1\,,\\
		u(2^k+k(r-1)-1)\,,&\mbox{if }k>1\,.
		\end{cases}
\]
We claim that this bound is not generally correct for large values of $k$ and $e$.

For any given $k$, it is easy to see that there is a target concept in $\overline{\mathcal{C}}^{2,k,e}_{ac}$ whose graph has $e_{max}=(n-k)k+\binom{k}{2}$ edges. We can always construct such a graph $G$ as indicated in the proof of Lemma~\ref{lem:Mk}: Let $V_1$ and $V_2$ be a partition over the $n$ vertices of $G$ where $|V_1|=n-k$ and $|V_2|=k$. Add an edge from each node in $V_2$ to each node in $V_1$. This results in $(n-k)k$ edges and $G$ is clearly acyclic with indegree exactly $k$ for every element in $V_1$. For the remaining $\binom{k}{2}$ edges, let $<$ be any total order on $V_2$; now the edge $(v,w)$ is added to $G$ if and only if $v<w$. 



Next, we evaluate the lower bound $\LB$ for $k>1$ when the target concept has $e_{max}$ edges:
\begin{flalign*}
&u(2^k+k(r-1)-1)\\
&=\lfloor \frac{e}{k}\rfloor (2^k+k(\lfloor \log_2( \frac{n-\lfloor \frac{e}{k}\rfloor}{k})\rfloor-1)-1)& \text{(setting $e=e_{max}$)}\\
&=\frac{\binom{k}{2}+(n-k)k}{k}(2^k+k(\lfloor \log_2( \frac{n-\frac{\binom{k}{2}+(n-k)k}{k}}{k})\rfloor-1)-1)\\
&=\frac{1}{2}(2n-k-1)(2^k+k(\lfloor \log_2(\frac{k+1}{2k})\rfloor -1)-1) \\
&=\frac{1}{2}(2n-k-1)(2^k-2k-1) 
\end{flalign*}
When $k=n-1$, this term evaluates to $n2^{n-2}-n^2+\frac{n}{2}$ which becomes larger than $2^n-1$ for $n>6$. More generally, given $k=n-c$ for some constant $0<c<n-1$, the term equals

\begin{equation*}
(n+(c-1))2^{n-(c+1)}-n^2-(c-1)n+\frac{(2c-1)(n+(c-1))}{2}
\end{equation*}

\noindent which exceeds $2^n-1$ for small values of $c$ and $n>6$. Theorem~\ref{thm:VCD}, however, states that $2^n-1$ is an upper bound on the VC dimension of $\overline{\mathcal{C}}^{2,k,e}_{ac}$. Consequently, there must be a mistake in the bound $\LB$. It appears that the mistake was caused by Koriche and Zanuttini assuming that there are acyclic CP-nets with $\frac{e}{k}2^k$ statements, which is not true for large values of $k$ and $e$.

\section{Structural Properties of CP-net Classes}
\label{sec:struct}


The computational learning theory literature knows of a number of structural properties under which the VC dimension and the recursive teaching dimension coincide. The purpose of this section is to investigate which of these structural properties apply to certain classes of acyclic CP-nets. The main result is that the class $\mathcal{C}_{ac}^{n-1}$ of all complete unbounded acyclic CP-nets does not satisfy any of the known general structural properties sufficient for $\VCdim$ and $\RTD$ to coincide; therefore this class may serve as an interesting starting point for formulating new general properties of a concept class $\mathcal{C}$ that are sufficient for establishing $\RTD(\mathcal{C})=\VCdim(\mathcal{C})$.

We first define the structural properties to be studied in this section. These basic notions from the computational learning theory literature can, for example, also be found in~\cite{DFSZ14}.

\begin{definition}\label{def:struct} 
Let $\mathcal{C}$ be any finite concept class over a fixed instance space $\mathcal{X}$. Let $d$ denote the VC dimension of $\mathcal{C}$.
\begin{enumerate}
\item The class $\mathcal{C}$ is maximum if $|\mathcal{C}|=\sum_{i=0}^d\binom{|\mathcal{X}|}{i}$, i.e., if the cardinality of $\mathcal{C}$ meets Sauer's bound~\cite{Sau72} with equality. 
\item The class $\mathcal{C}$ is maximal if $\VCdim(\mathcal{C}\cup\{c\})>d$ for any concept $c$ over $\mathcal{X}$ that is not contained in $\mathcal{C}$, i.e., if adding any concept $c\not\in \mathcal{C}$ to the class will increase its VC dimension. 
\item The class $\mathcal{C}$ is said to be an extremal class if $\mathcal{C}$ strongly shatters every set that it shatters.~$\mathcal{C}$ strongly shatters $S\subseteq \mathcal{X}$ if there is a subclass $\mathcal{C'}$ of $\mathcal{C}$ that shatters $S$ such that all concepts in $\mathcal{C'}$ agree on the labeling of all instances in $\mathcal{X}\backslash S$.\footnote{For example, suppose a concept $c$ over an instance space $\{x_1,x_2,x_3,x_4\}$ is represented as the vector $(c(x_1),c(x_2),c(x_3),c(x_4))$. Then the concepts $(0,0,0,0)$, $(0,0,0,1)$, $(0,0,1,0)$, $(1,0,1,1)$ shatter the set $\{x_3,x_4\}$, but they do not strongly shatter it, because they do not all agree on their labelings in the remaining instances (they disagree in $x_1$.) By comparison, $\{(1,0,0,0)$, $(1,0,0,1)$, $(1,0,1,0)$, $(1,0,1,1)\}$ strongly shatters the set $\{x_3,x_4\}$.}
\item The class $\mathcal{C}$ is intersection-closed if $c\cap \bar{c}\in \mathcal{C}$ for any two concepts $c,\bar{c}\in \mathcal{C}$.
\end{enumerate}
\end{definition}

Intuitively, for a given VC dimension and a given instance space size, maximum classes are the largest possible classes in terms of \emph{cardinality}, while maximal classes are largest with respect to \emph{inclusion}. Every maximum class is also maximal but the converse does not hold \cite{Kearns:1994:ICL:200548}. Every maximum class is also an extremal class, but not vice versa \cite{MoranW16}. 

It was proven that any finite maximum class $\mathcal{C}$ that can be \emph{corner-peeled}\/ by the algorithm proposed by Rubinstein and Rubinstein~\cite{RR12} satisfies $\RTD(\mathcal{C})=\VCdim(\mathcal{C})$ \cite{DFSZ14}. The same equality holds when $\mathcal{C}$ is of VC dimension $1$ or when $\mathcal{C}$ is intersection-closed~\cite{DFSZ14}. Therefore, if $\mathcal{C}_{ac}^{n-1}$ were to fulfill any of these structural properties, it would not be surprising that the recursive teaching dimension and the VC dimension coincide for $\mathcal{C}_{ac}^{n-1}$. In this section, we will demonstrate that none of these structural properties are fulfilled for any of the classes $\mathcal{C}_{ac}^k$ studied in our manuscript, with the exception of some trivial special cases.

The main results of this section are summarized in Table \ref{tble:structProperties}. For the proofs of some of these results, the following definition and lemma will be useful.

\begin{table}
	\centering
	\caption{Structural properties of CP-net concept classes.}
	\begin{tabular}{ l || c | c | c | c }
		class	& maximum & maximal & intersection-closed & extremal\\
		\hline
		$\mathcal{C}_{ac}^0$ with $m=2$ (over $\mathcal{X}_{sep}$)& yes& yes &	yes & yes\\
		\hline 
		$\mathcal{C}_{ac}^0$ with $m>2$ (over $\mathcal{X}_{sep}$ or $\mathcal{X}_{swap}$)& no & no &no& no\\
		\hline 
		$\mathcal{C}_{ac}^k$ with $m\ge 2$, $1\le k\le n-1$& no & no & no & no \\
		\hline 
	\end{tabular}
	
	\label{tble:structProperties}
\end{table}

\begin{definition}
Let $c$ be any concept over an instance space $\mathcal{X}$. Then the complement of $c$, denoted by $c^{co}$, is the concept over $\mathcal{X}$ that contains exactly those instances not contained in $c$, i.e., 
\[c^{co}(x)=1-c(x)\mbox{ for all }x\in\mathcal{X}\,.\]
\end{definition}

\begin{lemma}\label{lem:complement}
Let $n\ge 1$, $m\ge 2$, and $k\in\{0,\ldots,n-1\}$. If $c$ belongs to the class $\mathcal{C}^k_{ac}$ of all complete acyclic $k$-bounded CP-nets with $n$ variables of domain size $m$, over the instance space $\mathcal{X}_{swap}$, then also the complement $c^{co}$ of $c$ belongs to $\mathcal{C}^k_{ac}$.
\end{lemma}

\begin{proof}
The CP-net $N^{co}$ corresponding to $c^{co}$ is obtained from the CP-net $N$ corresponding to $c$ by reversing each preference statement in each CPT of $N$. Obviously, in $N^{co}$, each variable has the same parent set as in $N$, so that $c^{co}\in\mathcal{C}^k_{ac}$.
\end{proof}

First of all, we discuss separable CP-nets. Note that there is a one-to-one correspondence between separable CP-nets and $n$-tuples of total orders of $\{1,\ldots,m\}$: since there are no dependencies between the variables, each separable CP-net simply determines an order over the $m$ domain values of a variable, and it does so for each variable independently. This way, for a separable CP-net, the swap example $(v^1_{i_1}v^2_{i_2}\ldots v^{l-1}_{i_{l-1}}\alpha v^{l+1}_{i_{l+1}}\ldots v^{n}_{i_n},v^1_{i_1}v^2_{i_2}\ldots v^{l-1}_{i_{l-1}}\beta v^{l+1}_{i_{l+1}}\ldots v^{n}_{i_n})$ will always be labeled exactly the same way as \emph{any}\/ other swap example $x=(x.1,x.2)$ whose swapped variable is $v_l$ and for which the $l$th positions of $x.1$ and $x.2$ are $\alpha$ and $\beta$, respectively. 
We may therefore consider separable CP-nets over an instance space that is a proper subset of the set $\mathcal{X}_{swap}$ of all swaps without ``redundancies'', namely one that contains exactly one swap for each two domain values of each variable. Assuming a fixed choice of such pairs, we denote this subset of $\mathcal{X}_{swap}$ by $\mathcal{X}_{sep}$ and remark that $|\mathcal{X}_{sep}|=\binom{m}{2}n$. Note that, for the class of separable CP-nets, each instance $x\in\mathcal{X}_{swap}\setminus\mathcal{X}_{sep}$ is redundant in the following sense: there exists some $x' \in\mathcal{X}_{sep}$ such that
\begin{itemize}
\item either $c(x)=c(x')$ for all $c\in\mathcal{C}^0_{ac}$,
\item or $c(x)=1-c(x')$ for all $c\in\mathcal{C}^0_{ac}$.
\end{itemize}

It is now easy to see the following for the binary case.

\begin{proposition}
Let $n\ge 1$, $m=2$. Over the instance space $\mathcal{X}_{sep}$, the concept class $\mathcal{C}^0_{ac}$ of all complete binary separable CP-nets with $n$ variables is maximum (in particular, also maximal and extremal) and intersection-closed. 
\end{proposition}

\begin{proof} Note that we are considering the case $m=2$ and $k=0$.
Given the instance space $\mathcal{X}_{sep}$, the claim is immediate from the fact that $\VCdim(\mathcal{C}^0_{ac})=(m-1)n=n=\binom{m}{2}n=|\mathcal{X}_{sep}|$, which means that $\mathcal{C}^0_{ac}$ is the class of \emph{all}\/ possible concepts over $\mathcal{X}_{sep}$.
\end{proof}

In the non-binary case, we will see below that the situation is different. We start by showing that the class of complete separable CP-nets is not intersection-closed in the non-binary case.

Up to now, a subtlety in the definition of intersection-closedness has been ignored in our discussions. This is best explained using a very simple example. Consider a concept class $\mathcal{C}$ over $\mathcal{X}=\{x_1,x_2\}$ that contains the concepts $\{x_1\}$, $\{x_2\}$, and the empty concept. Obviously, $\mathcal{C}$ is intersection-closed. From a purely learning-theoretic point of view, and certainly for the calculation of any of the information complexity parameters studied above, $\mathcal{C}$ is equivalent to the class $\mathcal{C}'=\{\{x_2\},\{x_1\},\{x_1,x_2\}\}$ that results from $\mathcal{C}$ simply when flipping all labels. This class is no longer intersection-closed, as it does not contain the intersection of $\{x_2\}$ and $\{x_1\}$. Likewise, any two concept classes $\mathcal{C}$ and $\mathcal{C'}$ over some instance space $\mathcal{X}$ are equivalent if one is obtained from the other by ``inverting'' any of its instances, i.e., by selecting any subset $X\subseteq\mathcal{X}$ and replacing $c(x)$ by $1-c(x)$ for all $c\in\mathcal{C}$ and all $x\in X$. 

When defining the instance space $\mathcal{X}_{swap}$, we did not impose any requirements, for any swap pair $(o,o')$, as to whether $(o,o')$ or $(o',o)$ should be included in $\mathcal{X}_{swap}$. So, in fact, $\mathcal{X}_{swap}$ could be any of a whole class of instance spaces, all of which are equivalent for the purposes of calculating the information complexity parameters we studied. Thus, to show that $\mathcal{C}_{ac}^0$ is not intersection-closed, we have to consider all possible combinations in which the outcome pairs in $\mathcal{X}_{swap}$ could be arranged. In the proof of Proposition~\ref{prop:intersectionsep} this requires a distinction of only two cases, while more cases need to be considered when proving that $\mathcal{C}_{ac}^k$ is not intersection-closed for $k>0$ (see Proposition~\ref{prop:intersection} below.)

\begin{proposition}\label{prop:intersectionsep}
Let $n\ge1$, $m>2$. Then the class $\mathcal{C}_{ac}^0$ of all complete separable CP-nets with $n$-variables of domain size $m$ is not intersection-closed (neither over the instance space $\mathcal{X}_{sep}$ nor over the instance space $\mathcal{X}_{swap}$).
\end{proposition}

\begin{proof}
Let $v\in V$ be any variable. Since $m>2$, the domain of $v$ contains three pairwise distinct values $a_1$, $a_2$, and $a_3$ such that the instance space ($\mathcal{X}_{sep}$ or $\mathcal{X}_{swap}$) contains swap examples $x_1$, $x_2$, and $x_3$, each with the swapped variable $v$, and one of the following two cases holds:
\begin{itemize}
\item Case 1.  The projections $x_i[v]$ of the swap pairs $x_i$ to the swapped variable $v$ are
\[
x_1[v]=(a_1,a_2),\ x_2[v]=(a_1,a_3),\ x_3[v]=(a_2,a_3)\,.
\]
\item Case 2.  The projections $x_i[v]$ of the swap pairs $x_i$ to the swapped variable $v$ are
\[
x_1[v]=(a_1,a_2),\ x_2[v]=(a_3,a_1),\ x_3[v]=(a_2,a_3)\,.
\]
\end{itemize}
It remains to show that, in either case, we can find two separable CP-nets whose intersection (as concepts over the given instance space) is not a separable CP-net. 

In Case 1, let $c_1$ be a CP-net entailing $a_1\succ a_2$, $a_1\succ a_3$, and $a_3\succ a_2$, while $c_2$ entails $a_2\succ a_1$, $a_1\succ a_3$, and $a_2\succ a_3$. Both these sets of entailments can be realized by separable CP-nets. Both $c_1$ and $c_2$ label $x_2$ with 1 (as they both prefer $a_1$ over $a_3$), but they disagree in their labels for $x_1$ and $x_3$. The intersection of $c_1$ and $c_2$ thus labels $x_2$ with 1, while it labels both $x_1$ and $x_3$ with 0. This corresponds to a preference relation in which $a_2$ is preferred over $a_1$, then $a_1$ is preferred over $a_3$, but $a_3$ is preferred over $a_2$. This cycle cannot be realized by a separable CP-net, i.e., $c_1,c_2\in \mathcal{C}_{ac}^0$ while $c_1\cap c_2\notin \mathcal{C}_{ac}^0$.

In Case 2, let $c_1$ be a CP-net entailing $a_1\succ a_2$, $a_3\succ a_1$, and $a_3\succ a_2$, while $c_2$ entails $a_2\succ a_1$, $a_1\succ a_3$, and $a_2\succ a_3$. Both these sets of entailments can be realized by separable CP-nets. The concepts $c_1$ and $c_2$ disagree in their labels for all of $x_1$, $x_2$, and $x_3$. The intersection of $c_1$ and $c_2$ thus labels all of $x_1$, $x_2$, and $x_3$ with 0. This corresponds to a preference relation in which $a_2$ is preferred over $a_1$, then $a_1$ is preferred over $a_3$, but $a_3$ is preferred over $a_2$. This cycle cannot be realized by a separable CP-net, i.e., $c_1,c_2\in \mathcal{C}_{ac}^0$ while $c_1\cap c_2\notin \mathcal{C}_{ac}^0$.
\end{proof}

Proposition~\ref{prop:intersectionsep} states that the class of non-binary complete separable CP-nets is not intersection-closed, and that this result holds true even when restricting the instance space to the set $\mathcal{X}_{sep}$. It turns out that the same class is not maximal or extremal, either, over the same choices for instance spaces. Since maximum classes are always extremal, this also implies that the class of non-binary complete separable CP-nets is not maximum. The proofs of these claims rely on Lemma~\ref{lem:complement} and establish the same claims for the class $\mathcal{C}^k_{ac}$ for any $k\in\{1,\ldots,n-1\}$ and any $m\ge 2$. Since for $k>0$, the set $\mathcal{X}_{swap}$ is the more reasonable instance space, in the remainder of this section we always assume that a concept class is given over $\mathcal{X}_{swap}$. For the class of separable CP-nets though, every proof we provide will go through without modification when $\mathcal{X}_{swap}$ is replaced by$\mathcal{X}_{sep}$.

First, we show that maximality no longer holds for the class of separable CP-nets, when $m>2$, and neither for $\mathcal{C}_{ac}^k$, when $k>0$ and $m\ge 2$.

\begin{proposition}\label{prop:nonmaximal}
Let $n\ge1$, $m\ge 2$, and $0\le k\le{n-1}$, where $(m,k)\ne(2,0)$. Then the concept class $\mathcal{C}_{ac}^k$ of all complete $k$-bounded acyclic CP-nets over $n$ variables of domain size $m$, over the instance space $\mathcal{X}_{swap}$, is not maximal, and, in particular, $\mathcal{C}_{ac}^k$ is not maximum.
\end{proposition}

\begin{proof} We need to prove that there exists some concept $c$ over $\mathcal{X}_{swap}$ (not necessarily corresponding to a consistent CP-net) such that $\VCdim(\mathcal{C}_{ac}^k\cup\{c\})=\VCdim(\mathcal{C}_{ac}^k)$. We will prove an even stronger statement, namely: for every subset $X\subseteq\mathcal{X}_{swap}$ with $|X|=\VCdim(\mathcal{C}_{ac}^k)+1$ and every set $\mathcal{C}$ of concepts such that $\mathcal{C}_{ac}^k\cup\mathcal{C}$ shatters $X$, we have $|\mathcal{C}|\ge 2$.

Let $X\subseteq\mathcal{X}_{swap}$ with $|X|=\VCdim(\mathcal{C}_{ac}^k)+1$. Such a set $X$ exists, since $\mathcal{X}_{swap}$ is not shattered by $\mathcal{C}^k_{ac}$. Moreover, let $\vec{x}=(x^1,\ldots,x^{|X|})$ be any fixed sequence of all and only the elements in $X$, without repetitions. Since $X$ is not shattered by $\mathcal{C}_{ac}^k$, there is an assignment $(l_1,\ldots,l_{|X|})\in\{0,1\}^{|X|}$ of binary values to $\vec{x}$ that is not realized by $\mathcal{C}^k_{ac}$, i.e., there is no concept $c\in\mathcal{C}^k_{ac}$ such that $c(x^i)=l_i$ for all $i$. Since no concept in $\mathcal{C}_{ac}^k$ realizes the assignment $(l_1,\ldots,l_{|X|})$ on $\vec{x}$, by Lemma~\ref{lem:complement}, no concept in $\mathcal{C}_{ac}^k$ realizes the assignment $(1-l_1,\ldots,1-l_{|X|})$ on $\vec{x}$. Thus, to shatter $X$, one would need to add at least two concepts to $\mathcal{C}_{ac}^k$.
\end{proof}

Second, the following proposition establishes that, under the same conditions as in Proposition~\ref{prop:nonmaximal}, the class $\mathcal{C}_{ac}^k$ is not extremal.

\begin{proposition}
Let $n\ge 1$, $m\ge 2$, and $0\le k\le{n-1}$, where $(m,k)\ne(2,0)$. Then the concept class $\mathcal{C}_{ac}^k$ of all complete $k$-bounded acyclic CP-nets over $n$ variables of domain size $m$, over the instance space $\mathcal{X}_{swap}$, is not extremal.
\end{proposition}

\begin{proof}
Let $X\subseteq\mathcal{X}_{swap}$ be a set of instances that is shattered by $\mathcal{C}^k_{ac}$, such that $|X|=\VCdim(\mathcal{C}^k_{ac})$. Since $\mathcal{X}_{swap}$ is not shattered by $\mathcal{C}^k_{ac}$, we can fix some $\hat{x}\in\mathcal{X}_{swap}\setminus X$. Moreover, let $\vec{x}=(x^1,\ldots,x^{|X|})$ be any fixed sequence of all and only the elements in $X$, without repetitions. 

Suppose $\mathcal{C}_{ac}^k$ were extremal. Then $X$ is strongly shattered, so that, in particular, there exists some $\hat{l}\in\{0,1\}$ such that, for each choice of $(l_1,\ldots,l_{|X|})\in\{0,1\}^{|X|}$, the labeling $(l_1,\ldots,l_{|X|},\hat{l})$ of the instance vector $(x^1,\ldots,x^{|X|},\hat{x})$ is realized by $\mathcal{C}^k_{ac}$. Lemma~\ref{lem:complement} then implies that, for each choice of $(l_1,\ldots,l_{|X|})\in\{0,1\}^{|X|}$, the labeling $(1-l_1,\ldots,1-l_{|X|},1-\hat{l})$ of the instance vector $(x^1,\ldots,x^{|X|},\hat{x})$ is realized by $\mathcal{C}^k_{ac}$. This is equivalent to saying that, for each choice of $(l_1,\ldots,l_{|X|})\in\{0,1\}^{|X|}$, the labeling $(l_1,\ldots,l_{|X|},1-\hat{l})$ of the instance vector $(x^1,\ldots,x^{|X|},\hat{x})$ is realized by $\mathcal{C}^k_{ac}$. To sum up, for each choice of $(l_1,\ldots,l_{|X|})\in\{0,1\}^{|X|}$, both $(l_1,\ldots,l_{|X|},\hat{l})$  and $(l_1,\ldots,l_{|X|},1-\hat{l})$ as labelings of the instance vector $(x^1,\ldots,x^{|X|},\hat{x})$ are realized by $\mathcal{C}^k_{ac}$. This means that $X\cup\{\hat{x}\}$ is shattered by $\mathcal{C}^k_{ac}$, in contradiction to $|X|=\VCdim(\mathcal{C}^k_{ac})$.
\end{proof}

As a last result of our study of structural properties of CP-net classes, we show that $\mathcal{C}_{ac}^k$ is not intersection-closed, when $k\ge 1$ or when $m\ge 3$. 

\begin{proposition}\label{prop:intersection}
Let $n\ge 1$, $m\ge 2$, and $0\le k\le{n-1}$, where $(m,k)\ne(2,0)$. Then the concept class $\mathcal{C}_{ac}^k$ of all complete $k$-bounded acyclic CP-nets over $n$ variables of domain size $m$, over the instance space $\mathcal{X}_{swap}$, is not intersection-closed.
\end{proposition}

\begin{proof}
Let $A$ and $B$ be two distinct variables. Let $a_1$ and $a_2$ be two distinct values in $D_A$, and $b_1$ and $b_2$ two distinct values in $D_B$. Further, let $\gamma$ be any fixed context over the remaining variables, i.e., those in $V\setminus\{A,B\}$. We will argue over the possible preferences over outcomes of the form $ab\gamma$, where $a\in\{a_1,a_2\}$ is an assignment to $A$ and $b\in\{b_1,b_2\}$ is an assignment to $B$. 

Without loss of generality, suppose that $\mathcal{X}_{swap}$ contains the instance $(a_1b_1\gamma,a_1b_2\gamma)$ (instead of $(a_1b_2\gamma,a_1b_1\gamma)$.) If that were not the case, one could rename variables and values accordingly to make $\mathcal{X}_{swap}$ contain the instance $(a_1b_1\gamma,a_1b_2\gamma)$.

There are then various cases to consider for the swap pairs representing the comparisons between outcomes of the form $ab\gamma$, where $a\in\{a_1,a_2\}$ and $b\in\{b_1,b_2\}$. For each case, we provide two concepts $c_1,c_2\in\mathcal{C}^1_{ac}$ for which $c_1\cap c_2$ is not acyclic and thus is not in $\mathcal{C}^k_{ac}$ for any $k$.

\begin{itemize}
\item[]
Case 1. $\mathcal{X}_{swap}$ contains $(a_1b_1\gamma,a_2b_1\gamma)$, $(a_2b_1\gamma,a_2b_2\gamma)$, and $(a_1b_2\gamma,a_2b_2\gamma)$.\\
Case 2. $\mathcal{X}_{swap}$ contains $(a_1b_1\gamma,a_2b_1\gamma)$, $(a_2b_2\gamma,a_2b_1\gamma)$, and $(a_1b_2\gamma,a_2b_2\gamma)$.\\
Case 3. $\mathcal{X}_{swap}$ contains $(a_1b_1\gamma,a_2b_1\gamma)$, $(a_2b_1\gamma,a_2b_2\gamma)$, and $(a_2b_2\gamma,a_1b_2\gamma)$.\\
Case 4. $\mathcal{X}_{swap}$ contains $(a_1b_1\gamma,a_2b_1\gamma)$, $(a_2b_2\gamma,a_2b_1\gamma)$, and $(a_2b_2\gamma,a_1b_2\gamma)$.\\
Case 5. $\mathcal{X}_{swap}$ contains $(a_2b_1\gamma,a_1b_1\gamma)$, $(a_2b_1\gamma,a_2b_2\gamma)$, and $(a_1b_2\gamma,a_2b_2\gamma)$.\\
Case 6. $\mathcal{X}_{swap}$ contains $(a_2b_1\gamma,a_1b_1\gamma)$, $(a_2b_2\gamma,a_2b_1\gamma)$, and $(a_1b_2\gamma,a_2b_2\gamma)$.\\
Case 7. $\mathcal{X}_{swap}$ contains $(a_2b_1\gamma,a_1b_1\gamma)$, $(a_2b_1\gamma,a_2b_2\gamma)$, and $(a_2b_2\gamma,a_1b_2\gamma)$.\\
Case 8. $\mathcal{X}_{swap}$ contains $(a_2b_1\gamma,a_1b_1\gamma)$, $(a_2b_2\gamma,a_2b_1\gamma)$, and $(a_2b_2\gamma,a_1b_2\gamma)$.\\
\end{itemize}

Cases 1 and 7 are discussed in Table~\ref{tab:intersection}. A violation of the property of intersection-closedness in Cases 2, 3, and 4 can then be immediately deduced from Case 1 by inverting the binary values in the table for column 3, column 4, both columns 3 and 4, respectively. In the same way, Cases 8, 5, and 6 can be handled following Case 7.
\end{proof}

\begin{table}\centering
\footnotesize{
\begin{tabular}{|c||c|c|c|c||c|c|}
\hline
Case 1&1&2&3&4&5&6\\
&$(a_1b_1\gamma,a_1b_2\gamma)$&$(a_1b_1\gamma,a_2b_1\gamma)$&$(a_2b_1\gamma,a_2b_2\gamma)$&$(a_1b_2\gamma,a_2b_2\gamma)$&$\CPT(A)$&$\CPT(B)$\\\hline\hline
$c_1$&0&1&1&1&$a_1\succ a_2$&$a_1:b_2\succ b_1$\\
&&&&&&$a_2:b_1\succ b_2$\\\hline
$c_2$&1&0&1&1&$b_1:a_2\succ a_1$&$b_1\succ b_2$\\
&&&&&$b_2:a_1\succ a_2$&\\\hline
$c_1\cap c_2$&0&0&1&1&$b_1:a_2\succ a_1$&$a_1:b_2\succ b_1$\\
&&&&&$b_2:a_1\succ a_2$&$a_2:b_1\succ b_2$\\\hline\hline
Case 7&1&2&3&4&5&6\\
&$(a_1b_1\gamma,a_1b_2\gamma)$&$(a_2b_1\gamma,a_1b_1\gamma)$&$(a_2b_1\gamma,a_2b_2\gamma)$&$(a_2b_2\gamma,a_1b_2\gamma)$&$\CPT(A)$&$\CPT(B)$\\\hline\hline
$c_1$&0&1&1&1&$a_2\succ a_1$&$a_1:b_2\succ b_1$\\
&&&&&&$a_2:b_1\succ b_2$\\\hline
$c_2$&1&0&1&1&$b_1:a_1\succ a_2$&$b_1\succ b_2$\\
&&&&&$b_2:a_2\succ a_1$&\\\hline
$c_1\cap c_2$&0&0&1&1&$b_1:a_1\succ a_2$&$a_1:b_2\succ b_1$\\
&&&&&$b_2:a_2\succ a_1$&$a_2:b_1\succ b_2$\\\hline
\end{tabular}
}
\caption{Cases 1 and 7 in the proof of Proposition~\ref{prop:intersection}. Columns 1 through 4 provide binary labels stating which of the four swap instances considered are contained in a concept. Column 5 provides the statements in the corresponding CPT for $A$, while column 6 does the same for $B$. Concepts $c_1$ and $c_2$ belong to $\mathcal{C}^1_{ac}$, but $c_1\cap c_2$ has a cycle, in which $A$ is a parent of $B$ and vice versa.}\label{tab:intersection}
\end{table}

To conclude, there are no known structure-related theorems in the literature that would imply $\VCdim(\mathcal{C}^{n-1}_{ac})=\RTD(\mathcal{C}^{n-1}_{ac})$. Hence, the latter equation, which we have proven in Section~\ref{sec:complete}, is of interest, as it makes the class of complete unbounded acyclic CP-nets the first ``natural'' class known in the literature for which $\VCdim$ and $\RTD$ coincide. A deeper study of its structural properties might lead to new insights into the relationship between $\VCdim$ and $\RTD$ and might thus address open problems in the field of computational learning theory~\cite{SimonZ15}.

\end{document}